\documentclass[12pt]{jml}

\usepackage{mlmath}
\usepackage[utf8]{inputenc} % allow utf-8 input
\usepackage[T1]{fontenc}    % use 8-bit T1 fonts
\usepackage{hyperref}       % hyperlinks
\usepackage{url}            % simple URL typesetting
\usepackage{booktabs}       % professional-quality tables
\usepackage{amsfonts}       % blackboard math symbols
\usepackage{nicefrac}       % compact symbols for 1/2, etc.
\usepackage{microtype}      % microtypography
\usepackage{xcolor}         % colors

\usepackage{times}
\usepackage{graphicx}
\usepackage{caption}
\newtheorem{thm}{Theorem}

\newtheorem*{prop*}{Proposition}
\newtheorem{prop}{Proposition}
\newtheorem{rmk}{Remark}

\newtheorem*{sketch}{Heuristic argument}
\newtheorem*{just}{Justification}
\newtheorem{statement}{Condition}
\newtheorem*{theorem*}{Theorem}
\newtheorem*{lemma*}{Lemma}

\newtheorem*{property*}{Property}

\newtheorem*{assumption*}{Assumption}
\newcommand{\rnarr}{\textnormal{narr}}
\newcommand{\rwide}{\textnormal{wide}}
\newcommand{\rinter}{\textnormal{inter}}

\title[Embedding Principle]{Embedding Principle: a hierarchical structure of loss landscape of deep neural networks}
\usepackage{times}
\jmlauthor{%
 \Name{Yaoyu Zhang}\footnote{Corresponding author.} \Email{zhyy.sjtu@sjtu.edu.cn}\\
 \addr School of Mathematical Sciences, Institute of Natural Sciences, MOE-LSC and Qing Yuan Research Institute, Shanghai Jiao Tong University, Shanghai Center for Brain Science and Brain-Inspired Technology, Shanghai, 200240, China 
 \AND
 \Name{Yuqing Li} \Email{liyuqing\underline{~}551@sjtu.edu.cn}\\
 \addr School of Mathematical Sciences, Shanghai Jiao Tong University, Shanghai, 200240, China 
 \AND
 \Name{Zhongwang Zhang} \Email{0123zzw666@sjtu.edu.cn}\\
 \addr School of Mathematical Sciences, Institute of Natural Sciences, Shanghai Jiao Tong University, Shanghai, 200240, China 
 \AND
 \Name{Tao\ Luo} \Email{luotao41@sjtu.edu.cn}\\
 \addr School of Mathematical Sciences, Institute of Natural Sciences, MOE-LSC and Qing Yuan Research Institute, Shanghai Jiao Tong University, Shanghai, 200240, China 
 \AND
 \Name{Zhi-Qin John Xu}\footnote{Corresponding author.} \Email{xuzhiqin@sjtu.edu.cn}\\
 \addr School of Mathematical Sciences, Institute of Natural Sciences, MOE-LSC and Qing Yuan Research Institute, Shanghai Jiao Tong University, Shanghai, 200240, China 
}

\makeatletter
 \let\Ginclude@graphics\@org@Ginclude@graphics
\makeatother

\begin{document}
% \footnotetext[1]{Corresponding author.}
%\footnotetext[2]{Corresponding author.}
\maketitle

\begin{abstract}%
   We prove a general Embedding Principle of loss landscape of deep neural networks (NNs) that unravels a hierarchical structure of the loss landscape of NNs, i.e., loss landscape of an NN \emph{contains} all critical points of all the narrower NNs. This result is obtained by constructing a class of critical embeddings which map any critical point of a narrower NN to a critical point of the target NN with the same output function. By discovering a wide class of general compatible critical embeddings, we provide a gross estimate of the dimension of critical submanifolds embedded from critical points of narrower NNs.  We further prove an irreversiblility property of any critical embedding that the number of negative/zero/positive eigenvalues of the Hessian matrix of a critical point may increase but never decrease as an NN becomes wider through the embedding. Using a special realization of general compatible critical embedding, we prove a stringent necessary condition for being a ``truly-bad'' critical point that never becomes a strict-saddle point through any critical embedding. This result implies the commonplace of strict-saddle points in wide NNs, which may be an important reason underlying the easy optimization of wide NNs widely observed in practice.
\end{abstract}


\section{Introduction}
The loss landscape of a deep neural network (NN) is important to both its optimization dynamics and generalization performance, hence is a key issue in deep learning theory. 
It has been realized for a long time that it is important to quantify exactly how the loss landscape looks like \citep{weinan2020towards,sun2020global}.
This problem is difficult since various visualization methods show that the NN loss landscape is very complicated \citep{li2018visualizing,skorokhodov2019loss}. Moreover, its non-convexity, high dimensionality and the dependence on data, model and the specific form of loss function  make it very difficult to obtain a general understanding through empirical study. 
Therefore, though it has been extensively studied over the years, it remains an open problem to provide a clear picture about the general structure of a DNN loss landscape, e.g., critical points/submanifolds, their output functions and other properties.

Our work is inspired by the following empirical observations. From the aspect of optimization, it is often observed that wider NNs are easier for training. This phenomenon not only holds in a neural tangent kernel (NTK) regime \citep{jacot2018neural}, where the gradient descent training can find the global minimum with a linear convergence rate \citep{du2019gradient,chizat_global_2018,arora2019exact,zou2018stochastic}, but also happens in highly nonlinear regimes far beyond NTK \citep{luo2021phase}.  From the aspect of generalization, the puzzle that over-parameterized NNs often generalize well seems to contradict the conventional learning theory \citep{breiman1995reflections,zhang2016understanding}. The frequency principle \citep{xu_training_2018,xu2019frequency,rahaman2018spectral,zhang2021linear,luo2019theory} shows that NNs, over-parameterized or not, tend to fit the  training data by a low-frequency function, which suggests that the learned function by an NN is often of much lower complexity than the NN's capacity. Specifically, with small initialization, e.g., in a condensed regime, weights of an NN are empirically found to condense on isolated directions resulting in an output function mimicking that of a narrower NN \citep{luo2021phase,maennel2018gradient}. These observations raise a question that in which sense learning of a wide NN is not drastically different from a narrower NN despite potentially huge difference in their numbers of parameters. From the aspect of pruning, empirical works propose a ``lottery ticket hypothesis'' that a substantially smaller sub-network can achieve the same accuracy as the original large network \citep{frankle2018lottery}. However, it is not yet clear about the mechanism of redundancy in a learned wide NN, which makes a drastic pruning possible in practice.

% motivating problems:\\
% (i) Optimization: Why wider NNs are easier to train.\\
% (ii) Training: in which sense a wider NN should not be regarded as a drastically different model comparing to a narrower NN?\\
% (iii) Pruning: what is the mechanism of redundancy that allows drastic pruning potential of a wide NN?\\

All above empirical observations, though relevant to different aspects of NN, are in essence pointing towards an intrinsic similarity between narrow and wide NNs. 
% However,
% other than their different approximation/representation capabilities, where the function space of a wider NN contains the function space of a narrower NN, we do not know much about other relations.
In this work, focusing on the  loss landscape, we address the following problem: What are the relations of critical points and the corresponding output functions of loss landscape among NNs with different widths. The significance of studying the critical points and their output functions of the NN loss landscape is as follows. From the optimization perspective, NN parameters trained by gradient descent provably converge to a critical point, which in general is not necessarily a global minimum or local minimum. Moreover, even for these saddle points which can be escaped, e.g., strict-saddle points \citep{lee2019first}, they may still attract the training trajectory, contributing to the implicit regularization of NNs. We are specifically interested in output functions corresponding to critical points, named as \emph{critical functions} for convenience. Studying these critical functions that potentially attract the learning of NN is clearly important for a deeper understanding of the learning process of an NN.

Our key finding in this work is the following general principle about critical points/functions of NN loss landscape intuitively stated as follows:
\begin{center}
\emph{\textbf{Embedding principle}: the loss landscape of an NN \emph{contains} all critical points/functions of all the narrower NNs.}
\end{center}
The Embedding Principle shows that any NN loss landscape contains a hierarchical structure of critical points/functions with different complexities from NNs of different widths. Specifically, it ensures existence of ``simple'' critical functions that can be represented by narrow NNs. Therefore, combining with the phenomenon of Frequency Principle and condensation, we conjecture that nonlinear training of NNs may be implicitly biased towards these ``simple'' critical functions. We will carefully look into this conjecture in our future works. 

To prove the Embedding Principle, we first construct one-step critical embeddings which map any parameters of a narrow network to that of an one-neuron wider NN preserving the output function and criticality. With these embeddings,  critical points of a narrow network loss landscape is mapped to $1$-d critical affine subspaces of an one-neuron wider network loss landscape with the same output function. These one-step critical embeddings are constructed by adding a null neuron or splitting an existing neuron. 
By composition of one-step embeddings, any critical point of a narrow network loss landscape can be mapped to a critical point of any wider network loss landscape preserving the output function. Importantly, we further propose a wide class of general compatible critical embeddings, where one-step embeddings or their composition are its special cases. Note that, all critical points of a wide NN embedded from a critical point of a narrower NN by all possible general compatible critical embeddings, form high-dimensional critical submanifolds which in general are not affine subspaces for three-layer or deeper NNs.

% Establishing a thorough theoretical understanding of loss landscape of deep neural networks (NNs) is essential for developing the deep learning theory. From the aspect of optimization, loss landscape, especially whether bad local-min exists, is the key to convergence as well as its rate to a global-min. Moreover, for an over-parameterized NN widely used in application, training with different hyperparameters may converge to different global-min, yielding different generalization performance. For example, in an NTK regime of large initialization, NN converges linearly at the vicinity of its initialization.

The critical embeddings naturally link critical points of NNs of different widths, thus providing a means to track how properties of these critical points may change as the width of the NN increases. Using these critical embeddings as a tool, we obtain rich information about the general structure of an NN loss landscape. 

We show that the degeneracy of a critical point substantially increases when it is embedded to a wider network, due to the fact that a critical point can be mapped to a high-dimensional critical submanifold through a class of critical embeddings. This degeneracy of critical points arises from the neuron redundancy of the wide NN in representing certain simple critical functions from narrower NNs, which is different from over-parameterization induced degeneracy studied in \cite{cooper2018loss}. We also study the property of Hessian of critical points through critical embedding, e.g., the number of its negative eigenvalues, which determines whether the corresponding critical point is a strict-saddle that enables easy optimization \citep{lee2019first}. We prove an irreversiblility  of critical embedding that the number of negative eigenvalues of Hessian matrix may increase but never decrease as an NN becomes wider through critical embedding. Moreover, we introduce a notion of ``truly-bad'' critical point which never becomes a strict-saddle point through any critical embedding. We prove a stringent necessary condition for being a ``truly-bad'' critical point that requires an important ingredient of its Hessian matrix being a zero matrix.  This result implies the commonplace of strict-saddle points in the high-dimensional critical submanifolds of wide NNs, which may be an important reason underlying the easy optimization of wide NNs widely observed in practice.

In summary, the following general understanding of an NN loss landscape is obtained by the embedding principle in this work:\\
(i) It contains a hierarchical structure of critical points/functions with different complexities from that of all narrower NNs; \\
(ii) Critical functions from narrower NNs in general forms a high-dimensional critical submanifold with a gross estimate of the dimension: $K+\sum_{k\in[L]} K_l K_{l-1}$, where $K_l$ is the difference in neuron number in layer $l$ between the target NN and the narrower NN.\\
(iii) If it has critical points other than the global minima that are not strict-saddle points, they mostly can become strict-saddle points in wider NNs through critical embedding, which means the optimization difficulty results from these critical points can be alleviated by simply using a wider NN.

\section{Related works}
This work is an comprehensive extension of our previous conference paper \cite{zhang2021embedding}, in which we study  one-step embeddings and their multi-step composition. Similar results on composition embedding are studied in other works \citep{fukumizu2000local,fukumizu2019semi,csimcsek2021geometry}.

Simple gradient-descent-based optimization on the complex loss landscape of NN \citep{skorokhodov2019loss,weinan2020towards} often finds solutions that generalize well. Many works study the geometry of the NN loss landscape at critical points in relation to its generalization ability. For example, empirical works show that SGD \citep{keskar2017large} and dropout \citep{zhang2021variance} training can find a flat minimizer, which may explain why such stochastic training can find solution that generalize better. 
\cite{wu2017towards} further suggest that the volume of basin of attraction of good (flat) minima may dominate over that of poor (sharp) minima in practical problems. \cite{he2019asymmetric} show that at a local minimum there exist many asymmetric directions such that the loss increases abruptly along one side, and slowly along the opposite side. \cite{ding2019sub} prove that for any multi-layer network with generic input data and non-linear activation functions, sub-optimal local minima can exist, no matter how wide the network is.  When the network width increases towards infinity, the loss landscape may become simpler and the training can avoid spurious valleys with high probability in an over-parameterized regime \citep{venturi2019spurious}.  In an extremely over-parameterized regime with a large initialization, i.e., the linear regime identified in \cite{luo2021phase} with the NTK regime as its special case \citep{jacot2018neural}, the gradient descent training can find the global minimum with a linear convergence rate \citep{luo2021phase,du2019gradient,chizat_global_2018,arora2019exact,zou2018stochastic}. 
 
The starting point of this work originates from our work in \cite{luo2021phase}, where we identify a highly nonlinear condensed regime far beyond the NTK regime that weights condense in isolated directions during the training. Moreover, neural networks of different width often exhibit similar condensed behavior, e.g., stagnating at similar loss with almost the same output function, which is illustrated in experiments in our conference paper \cite{zhang2021embedding}. The condensation is a highly nonlinear feature learning process important to implicit regularization and generalization of NNs. The condensation transforms a large network to a network of only a few effective neurons, leading to an output function with low complexity. Such learning process is consistent with another line of research, that is, the complexity of NN output gradually increases during the training \citep{arpit2017closer,xu_training_2018,xu2019frequency,rahaman2018spectral,zhang2021linear,luo2019theory,kalimeris2019sgd,goldt2020modeling,he2020assessing,mingard2019neural,jin2020quantifying}. For example, the Frequency Principle \citep{xu_training_2018,xu2019frequency}  states that NNs often fit target functions from low to high frequencies during the training. A series of works study the mechanism of condensation at an initial training stage, such as for ReLU network \citep{maennel2018gradient,pellegrini2020analytic} and network with continuously differentiable activation functions \citep{xu2021towards}.

This work in some sense serves as our attempt to uncover the theoretical structure underlying the condensation phenomenon from the perspective of loss function by proving a general Embedding Principle. In another aspect, the condensation phenomenon also confirms the value of Embedding Principle in understanding the highly nonlinear training behavior in practice.  

The Embedding Principle provides a structural mechanism underlying the degeneracy as a very common property for critical points  \citep{choromanska2015loss,sagun2016singularity}. Thus it complements the understanding that global minima of NNs typically form a high dimensional manifold due to over-parameterization \citep{cooper2018loss}.

\section{Preliminary}
\subsection{Deep neural networks}
Consider $L$-layer ($L\geq 2$) fully-connected NNs with a general differentiable activation function. We regard the input as the $0$-th layer and the output as the $L$-th layer. Let $m_l$ be the number of neurons in the $l$-th layer. In particular, we also set $m_0=d$ and $m_L=d'$. For any $i,k\in \sN$ and $i<k$, we denote $[i:k]=\{i,i+1,\ldots,k\}$. In particular, we denote $[k]:=\{1,2,\ldots,k\}$. For a matrix $\mA$, we use $(\mA)_{i,j}$  to denote its $(i, j)$-th entry. We will also  define $(\mA)_{i,[j:k]} :=
\left((\mA)_{i,j} , (\mA)_{i,j+1}, \cdots , (\mA)_{i,k}\right)$ as part of the $i$-th row vector. Similarly, $(\mA)_{[j:k],i}$ is a part of the $i$-th
column vector. For a vector $\va$, we use $(\va)_{i}$  to denote its $i$-th entry, we   also  define $(\va)_{[j:k]} :=
\left((\va)_{j} , (\va)_{j+1}, \cdots , (\va)_{k}\right)$ as part of the  vector.
Given weights $\mW^{[l]}\in \sR^{m_l\times m_{l-1}}$ and bias $\vb^{[l]}\in\sR^{m_{l}}$ for $l\in[L]$, we define the collection of parameters $\vtheta$ as a $2L$-tuple (an ordered list of $2L$ elements) whose elements are matrices or vectors
\begin{equation}
	\vtheta:=\Big(\vtheta|_1,\cdots,\vtheta|_L\Big):=\Big(\mW^{[1]},\vb^{[1]},\ldots,\mW^{[L]},\vb^{[L]}\Big).
\end{equation}
where the $l$-th layer parameters of $\vtheta$ is the ordered pair $\vtheta|_{l}:=\Big(\mW^{[l]},\vb^{[l]}\Big),\quad l\in[L]$. 
We may misuse our notations and do not distinguish $\vtheta$ from its vectorization $\mathrm{vec}(\vtheta)\in \sR^M$ with $M:=\sum_{l=0}^{L-1}(m_l+1) m_{l+1}$. Moreover, we call the collection of tuples of length $2L$ {\textbf{the tuple class}}, whose elements are matrices $\{\mW^{[l]}\}_{l=1}^L$ with $\mW^{[l]}\in \sR^{m_l\times m_{l-1}}$,  or vectors $\{\vb^{[l]}\}_{l=1}^L$ with $\vb^{[l]}\in\sR^{m_{l}}$, and denoted by $\mathrm{Tuple}_{\{m_0,\cdots,m_L\}}$, i.e.,
\[\mathrm{Tuple}_{\{m_0,\cdots,m_L\}}:=\left\{ \vtheta | \vtheta=\Big(\mW^{[1]},\vb^{[1]},\ldots,\mW^{[L]},\vb^{[L]}\Big),~~\mW^{[l]}\in \sR^{m_l\times m_{l-1}},~\vb^{[l]}\in\sR^{m_{l}},~l\in[L]\right\}.\]
Since the tuple class inherits the structure of Euclidean spaces, obviously it is  a linear space.
We set $\vzero$ as the zero element in the tuple class, i.e.
\[ \vzero=\Big(\vzero_{m_1\times m_{0}},\vzero_{m_1\times 1},\ldots,\vzero_{m_L\times m_{L-1}},\vzero_{m_L\times 1}\Big)\in \mathrm{Tuple}_{\{m_0,\cdots,m_L\}},\]
and we abuse the notation $\vzero$ from time to time to denote zero elements belonging to different tuple classes.  

We further define  the upper bracket $[L-1]$ by limiting ourselves to the first $2L-2$ element of the tuple, i.e., 
\begin{equation}
	\vtheta^{[L-1]}:=(\mW^{[1]},\vb^{[1]},\cdots,\mW^{[L-2]},\vb^{[L-2]},\mW^{[L-1]},\vb^{[L-1]})\in\mathrm{Tuple}_{\{m_0,\cdots,m_{L-2},m_{L-1}\}}.
\end{equation}

Given parameters $\vtheta$, the neural network function $\vf_{\vtheta}(\cdot)$ can be defined in a recursive way. First, we write $\vf^{[0]}_{\vtheta}(\vx):=\vx$ for the input $\vx\in\sR^d$, then for $l\in[L-1]$, $\vf^{[l]}_{\vtheta}$ is defined recursively as
\[\vf^{[l]}_{\vtheta}(\vx):=\sigma (\mW^{[l]} \vf^{[l-1]}_{\vtheta}(\vx)+\vb^{[l]}),\]
where $\sigma(\cdot):\sR\to\sR$ is the activation function applied coordinate-wisely, with slight abuse of notation.
Finally, we denote
\begin{equation}
	\vf_{\vtheta}(\vx):=\vf(\vx,\vtheta):=\vf^{[L]}_{\vtheta}(\vx):=\mW^{[L]} \vf^{[L-1]}_{\vtheta}(\vx)+\vb^{[L]},
\end{equation}
and for simplicity, sometimes we may drop the subscript $\vtheta$ in $\vf^{[l]}_{\vtheta}$ for $l\in[0:L]$.

\subsection{Loss function}
The set of training data  is denoted by  $S:=\left\{(\vx_i,\vy_i)\right\}_{i=1}^n$, where $\vx_i\in\sR^d$, $\vy_i\in \sR^{d'}$. Here we assume that there exists an unknown function $\vf^*(\cdot):\sR^d\to\sR^{d'}$ satisfying  $\vf^*(\vx_i)=\vy_i$ for $i\in[n]$. 
The empirical risk reads as
\begin{equation}\label{eq...text...empricalloss}
	\RS(\vtheta):=\frac{1}{n}\sum_{i=1}^n\ell(\vf(\vx_i,\vtheta),\vy_i):=\Exp_S\ell(\vf(\vx,\vtheta),\vf^*(\vx)).
\end{equation}
where the expectation $\Exp_S$ is defined  for any function $h(\cdot):\sR^d\to \sR$ as \[\Exp_S h(\vx):=\frac{1}{n}\sum_{i=1}^n h(\vx_i),\]  and the loss function    $\ell(\cdot,\cdot):\sR^{d'}\times \sR^{d'}\to \sR$ in \eqref{eq...text...empricalloss} is differentiable  in both variables, and the derivative of $\ell(\cdot,\cdot)$ with respect to its first argument is denoted by $\nabla\ell(\vy,\vy^*)$.  In this paper, we always take derivatives/gradients of $\ell(\cdot,\cdot)$ in its first argument with respect to any possible parameter.
% We consider the gradient flow of $\RS(\vtheta)$ as the training dynamics, i.e., \[\frac{\D \vtheta}{\D t} = -\nabla_{\vtheta} \RS(\vtheta),~~\vtheta(0) = \vtheta_0,\] where $\vtheta_0$ is given. 

For each $l\in[L]$, we define the error vectors as \[\vz_{\vtheta}^{[l]}:=\nabla_{\vf^{[l]}}\ell,\] and the feature gradients as
\[
\vg_{\vtheta}^{[L]}:=\mathbf{1},~\text{and}~~\vg^{[l]}_{\vtheta} :=\sigma^{(1)}\left(\mW^{[l]} \vf^{[l-1]}_{\vtheta}+\vb^{[l]}\right)~\text{for}~l\in[L-1],
\] 
where we use $\sigma^{(1)}(\cdot)$ for the first derivative of $\sigma(\cdot)$. Moreover, we call $\vf^{[l]}_{\vtheta}$ the feature vectors, and we denote the collections of feature vectors, feature gradients, and error vectors $\{\vz^{[l]}_{\vtheta}\}_{l=1}^L$ respectively by 
\[
\vF_{\vtheta}:= \{\vf^{[l]}_{\vtheta}\}_{l=1}^L,
~\vG_{\vtheta}
:= \{\vg^{[l]}_{\vtheta}\}_{l=1}^L,
~\vZ_{\vtheta}
:= \{\vz^{[l]}_{\vtheta}\}_{l=1}^L.
\] 
Moreover, using backpropagation, we can derive the following relations concerning the above quantities  
\begin{equation}\label{eq..text...BackwardProp}
	\begin{aligned}
		\vz_{\vtheta}^{[L]}
		&=\nabla\ell,\\
		\vz_{\vtheta}^{[l]}
		&= (\mW^{[l+1]})^\T\left(\vz_{\vtheta}^{[l+1]}\circ\vg_{\vtheta}^{[l+1]}\right),\quad l\in[L-1],\\
		\nabla_{\mW^{[l]}}\ell
		&= \left(\vz_{\vtheta}^{[l]}\circ\vg_{\vtheta}^{[l]}\right)(\vf_{\vtheta}^{[l-1]})^\T,\quad l\in[L],\\
		\nabla_{\vb^{[l]}}\ell
		&= \vz_{\vtheta}^{[l]}\circ\vg_{\vtheta}^{[l]},\quad l\in[L],
	\end{aligned}
\end{equation}

where we use $\circ$ for the Hadamard product~\citep{horn2012matrix} of two matrices of the same dimension.  

Specifically, for simplicity of gradient computation, we define another group of error vectors for $l\in[L]$   
\[
\ve^{[l]}_{\vtheta}:=\vz_{\vtheta}^{[l]}\circ\vg_{\vtheta}^{[l]},
\]
and we denote $\{\ve^{[l]}_{\vtheta}\}_{l=1}^L$   by   $\vE_{\vtheta}
:= \{\ve^{[l]}_{\vtheta}\}_{l=1}^L$. 

Directly from relation \eqref{eq..text...BackwardProp}, we obtain that 
\begin{equation}\label{eq..text...BackwardPropOnlyforel}
	\ve^{[l]}_{\vtheta}=\vz_{\vtheta}^{[l]}\circ\vg_{\vtheta}^{[l]}=\left((\mW^{[l+1]})^\T\left(\vz_{\vtheta}^{[l+1]}\circ\vg_{\vtheta}^{[l+1]}\right)\right)\circ\vg_{\vtheta}^{[l]}=\left((\mW^{[l+1]})^\T\ve^{[l+1]}_{\vtheta}\right)\circ\vg_{\vtheta}^{[l]}.
\end{equation}
%%%%%%%%%%%%%%%%%%%%%%%%%%%%%%%%%%%%%
\subsection{Hessian}
%\textcolor{blue}{For the results regarding the Hessian of loss, we consider NNs with $1$-d scalar output for simplicity in computation.}
Given a  scalar loss function $\ell(\cdot,\cdot):\sR^{d'}\times \sR^{d'}\to \sR$, twice differentiable in both variables,  and an activation function $\sigma(\cdot):\sR \to \sR$, also  twice differentiable, we denote that
\begin{align}
	\RS(\vtheta)&=\Exp_S \ell(f(\vx,\vtheta),f^*(\vx)). \\
	% \RS(\vtheta)&\approx\RS(\vtheta_0)+\vv_S(\vtheta_0)^{\T}\delta\vtheta+
	% \frac{1}{2}\delta\vtheta^{\T}\mH_S(\vtheta_0)\delta\vtheta. \\
	\vv_S(\vtheta)&:=\nabla_{\vtheta}\RS(\vtheta) = \Exp_{S}\nabla\ell (\vf(\vx,\vtheta),\vf^*(\vx))^\T\nabla_{\vtheta}\vf_{\vtheta}(\vx)=\sum_{i=1}^{m_L} \Exp_{S}\partial_i\ell (\vf_{\vtheta},\vf^*)\nabla_{\vtheta}(\vf_{\vtheta})_i, 
\end{align}
where $\partial_i\ell (\vf_{\vtheta},\vf^*)$ is the $i$-th element of  $\nabla\ell (\vf(\vx,\vtheta),\vf^*(\vx))$, and $(\vf_{\vtheta})_i$ is the $i$-th element of vector $\vf_{\vtheta}$.
Then for the Hessian matrix $\mH_S(\vtheta)$, we have
\begin{align*}
	\mH_S(\vtheta)&:=\nabla_{\vtheta}\nabla_{\vtheta}\RS(\vtheta)=\sum_{i=1}^{m_L}\Exp_{S}\nabla_{\vtheta}\left(\partial_i\ell (\vf_{\vtheta},\vf^*)\right)\nabla_{\vtheta}(\vf_{\vtheta})_i+\sum_{i=1}^{m_L}\Exp_{S}\partial_i\ell (\vf_{\vtheta},\vf^*)\nabla_{\vtheta}\nabla_{\vtheta}\left((\vf_{\vtheta})_i\right)\\
	&=\sum_{i,j=1}^{m_L}\Exp_{S}\partial_{ij}\ell (\vf_{\vtheta},\vf^*)\nabla_{\vtheta}(\vf_{\vtheta})_i\left(\nabla_{\vtheta}(\vf_{\vtheta})_j\right)^\T
	+
	\sum_{i=1}^{m_L}\Exp_{S}\partial_i\ell (\vf_{\vtheta},\vf^*)\nabla_{\vtheta}\nabla_{\vtheta}\left((\vf_{\vtheta})_i\right),
\end{align*}
where $\partial_{ij}\ell (\vf_{\vtheta},\vf^*)$ is the $(i,j)$-th element of  $\nabla\nabla\ell (\vf(\vx,\vtheta),\vf^*(\vx))$.

We define matrices $\mH^{(1)}_S(\vtheta)$ and $\mH^{(2)}_S(\vtheta)$  as follows:
\begin{align}
	\mH^{(1)}_S(\vtheta)&:=\sum_{i,j=1}^{m_L}\Exp_{S}\partial_{ij}\ell (\vf_{\vtheta},\vf^*)\nabla_{\vtheta}(\vf_{\vtheta})_i\left(\nabla_{\vtheta}(\vf_{\vtheta})_j\right)^\T,\\
	\mH^{(2)}_S(\vtheta)&:=\sum_{i=1}^{m_L}\Exp_{S}\partial_i\ell (\vf_{\vtheta},\vf^*)\nabla_{\vtheta}\nabla_{\vtheta}\left((\vf_{\vtheta})_i\right),
\end{align}
and
\[
\mH_S(\vtheta)=\mH^{(1)}_S(\vtheta)+ \mH^{(2)}_S(\vtheta).
\]
% \textcolor{purple}{Note that, for a common convex loss function $\ell(\cdot,\cdot)$, $\mH^{(1)}_S(\vtheta)$ is always positve semidefinite}
%%%%%%%%%%%%%%%%%%%%%%%%%%%%%%%%%%%%%%%%%%%%%%%%%%%%%%%
\subsection{Assumptions and conventions of notations}
We begin this part by introducing several assumptions  that will be used throughout this paper: 
\begin{assumption*}
	\textcolor{black}{\\
		(i) We choose the $L$-layer~($L\geq 2$) fully-connected deep neural networks~(NNs) as our model.\\
		(ii) Our training data  is $S=\{(\vx_i,\vy_i)\}_{i=1}^n$, $n\in\sZ^+\cup\{+\infty\}$.\\
		(iii) We use the empirical loss $\RS(\vtheta)=\Exp_S\ell(\vf_{\vtheta}(\vx),\vy)$.\\
		(iv) Loss function  $\ell(\cdot,\cdot)$ and activation function $\sigma(\cdot)$ are (weakly) differentiable. (Remark: twice differentiable is required for the computation of Hessian)}
\end{assumption*}
After stating out the assumptions, we would also like to introduce some conventions of notations that are frequently  used in this paper in the following.

% We denote the outputs of different NNs by $\vf_{\vtheta}(\vx)$, distinguished by $\vtheta$ of different tuple classes. Similarly, given data $S$, loss and activation,  $\RS({\vtheta})$ may correspond to
% loss landscape of different NNs distinguished by $\vtheta$.

We write $\mathrm{NN}(\{m_l\}_{l=0}^{L})$ for a fully-connected $L$-layer network with width $(m_0,\ldots,m_L)$, by which the tuple class of its parameters $\vtheta \in \mathrm{Tuple}_{\{m_0,\cdots,m_L\}}$ is determined whereas its activation $\sigma(\cdot)$ is not provided. When $\sigma(\cdot)$ is given, the output of $\mathrm{NN}(\{m_l\}_{l=0}^{L})$ is denoted by $\vf_{\vtheta}(\vx)$ with $\vtheta \in \mathrm{Tuple}_{\{m_0,\cdots,m_L\}}$. 

Note that, if given two NNs, $\mathrm{NN}(\{m_l\}_{l=0}^{L})$ and $\mathrm{NN}(\{m'_l\}_{l=0}^{L})$, their corresponding parameters belong to different tuple classes except when $m'_l= m_l$ for all $l\in[0:L]$. Therefore, in this work, $\vf_{\vtheta}(\vx)$ and $\RS({\vtheta})$  may correspond to output and
loss landscape of different NNs distinguished by $\vtheta$ of different tuple classes.

Given two NNs, $\mathrm{NN}(\{m_l\}_{l=0}^{L})$ and $\mathrm{NN}(\{m'_l\}_{l=0}^{L})$ with $m'_0=m_0$, $m'_L=m_L$, and  $m'_l\geq m_l$ for any $l\in[L-1]$, then for $K=\sum_{l=1}^{L-1}(m'_l-m_l)\in\sZ^+$,  we say that $\mathrm{NN}(\{m'_l\}_{l=0}^{L})$ is $K$-neuron \textbf{wider} than  $\mathrm{NN}(\{m_l\}_{l=0}^{L})$, and conversely,  $\mathrm{NN}(\{m_l\}_{l=0}^{L})$ is $K$-neuron \textbf{narrower} than $\mathrm{NN}(\{m'_l\}_{l=0}^{L})$.

As long as we have two NNs, $\mathrm{NN}(\{m_l\}_{l=0}^{L})$ and $\mathrm{NN}(\{m'_l\}_{l=0}^{L})$, given in the context of Definitions, Theorems, Propositions, Lemmas etc, we always assume that  $\mathrm{NN}(\{m'_l\}_{l=0}^{L})$ is wider than  $\mathrm{NN}(\{m_l\}_{l=0}^{L})$, i.e., $m'_0=m_0$, $m'_L=m_L$, and  $m'_l\geq m_l$ for any $l\in[L-1]$. We also denote $M=\sum_{l=0}^{L-1}(m_l+1) m_{l+1}$ and $M'=\sum_{l=0}^{L-1}(m'_l+1) m'_{l+1}$, and consequently, $M'\geq M$.  We   denote the parameters of a  narrower network by $\vtheta_{\rnarr}$, and the counterpart of a  wider network by $\vtheta_{\rwide}$. Then, given the data $S$, loss $\ell(\cdot,\cdot)$ and activation $\sigma(\cdot)$,  the collection of critical points of narrower NN and wider NN can be found respectively and denoted by     $\vTheta_{\rnarr}^{\rc}:=\{\vtheta|\nabla_{\vtheta}\RS(\vtheta_{\rnarr})=\vzero\}$ and $\vTheta_{\rwide}^{\rc}:=\{\vtheta|\nabla_{\vtheta}\RS(\vtheta_{\rwide})=\vzero\}$. Furthermore,  $\fF_{\rnarr}^{\rc}:=\{\vf_{\vtheta}|\vtheta\in\vTheta_{\rnarr}^{\rc}\}$ and $\fF_{\rwide}^{\rc}:=\{\vf_{\vtheta}|\vtheta\in\vTheta_{\rwide}^{\rc}\}$ are denoted for  the function spaces induced by critical points accordingly.

% \textcolor{blue}{Suppose we have an NN, then if   the data $S$, loss $\ell(\cdot,\cdot)$ and activation $\sigma(\cdot)$ are given,   its   collection of critical points can be determined and  denoted    by   $\vTheta^{\rc}:=\{\vtheta|\nabla_{\vtheta}\RS(\vtheta_{})=\vzero\}$, which induces a function space $\fF^{\rc}:=\{\vf_{\vtheta}|\vtheta\in\vTheta^{\rc}\}$. We claim that throughout this paper,  $\vTheta^{\rc}$ refers to the  collection of critical points of NN once the data $S$, loss $\ell(\cdot,\cdot)$ and activation $\sigma(\cdot)$ are given, as are the cases for $\vTheta_{\rnarr}^{\rc}$ and $\vTheta_{\rwide}^{\rc}$. 
	% Moreover, we define an inverse  critical map such that, given any NN, data, loss and activation, for any $\vf$, $\vartheta(\vf):=\{\vtheta|\vtheta\in\vTheta^{\rc},\vf_{\vtheta}=\vf\}$ maps any function to the set of critical points with output function $\vf$. $\vartheta(\vf)\neq \emptyset$ means $\vf\in\fF^{\rc}_{}$. We note that, for any $\vf^{c}_{\rnarr}\in\fF^{\rc}_{\rnarr}$, $\vartheta(\vf^{\rc}_{\rnarr})$ is a high dimensional submanifold in general by the critical embeddings we proposed. Its dimension is estimated in Section \ref{section...PropertyofCP}. }

Finally, a neuron, say the $i$-th neuron in layer $l$, is termed a {\textbf{null neuron}} if its output is a constant independent of input $\vx$ for any activation, i.e., $(\vf^{[l]})_i(\cdot)\equiv\mathrm{Const}$ for any $\sigma(\cdot)$. Otherwise, we call this neuron an {\textbf{effective neuron}}. 

%%%%%%%%%%%%%%%%%%%%%%%%%%%%%%%%%%%%%%%%%%%%%%%%%%%%%%%
\section{Embedding Principle}
In this section, we prove the Embedding Principle by constructing critical embeddings. First, we define a critical embedding operation. Then, by constructing one-step critical embeddings and their composition, we prove the Embedding Principle that critical points of the loss landscape of a narrow network can be embedded to critical affine subspaces of the loss landscape of any wider network while preserving the output function. Finally, we emphasize the importance of Embedding Principle in understanding the implicit regularization and generalization of NNs.

% Finally, we discover a wide class of general compatible critical embeddings with one-step embeddings and their composition as its special cases. Remark that, discovering as many critical embeddings as possible is important for obtaining a more precise understanding of the geometry of the critical submanifolds in a NN loss landscape as illustrated in the next section. 

We begin with the concepts of embedding, affine embedding and critical embedding.
\begin{definition}[\bf{Embedding and affine embedding}]% and critical embedding}]
Given  an   $\mathrm{NN}(\{m_l\}_{l=0}^{L})$ and  $\mathrm{NN}(\{m'_l\}_{l=0}^{L})$, an  
{\textbf {embedding}} is an injective operator $\fT:\mathrm{Tuple}_{\{m_0,\cdots,m_L\}}\to\mathrm{Tuple}_{\{m'_0,\cdots,m'_L\}}$, i.e., $\fT(\vtheta_1)\neq \fT(\vtheta_2)$ for $\vtheta_1, \vtheta_2\in\mathrm{Tuple}_{\{m_0,\cdots,m_L\}}$ and $\vtheta_1\neq \vtheta_2$. In addition, $\fT$ is an \textbf{affine embedding} if $\tilde{\fT}(\vtheta):=\fT(\vtheta)-\fT(\vzero)$ is a linear operator, i.e., $\tilde{\fT}(\vtheta_1)+\tilde{\fT}(\vtheta_2)=\tilde{\fT}(\vtheta_1+\vtheta_2)$ and $\tilde{\fT}(\beta\vtheta)=\beta\tilde{\fT}(\vtheta)$ for any $\vtheta,\vtheta_1,\vtheta_2\in\mathrm{Tuple}_{\{m_0,\cdots,m_L\}}$ and  $\beta\in\sR$.
\end{definition}
% {\yl{We remark that since the  parameter $\vtheta$ of  $\mathrm{NN}(\{m_l\}_{l=0}^{L})$ is a tuple, however in Definition \ref{def..afine}, we misuse our notations and identify $\vtheta$ with its vectorization $\mathrm{vec}(\vtheta)\in \sR^M$ with $M=\sum_{l=0}^{L-1}(m_l+1) m_{l+1}$.}}
\begin{rmk}
For any  given affine embedding $\fT$, it is associated with a matrix $\mA\in\sR^{M'\times M}$ and a vector $\vc\in  \sR^{M'}$ such that $\mathrm{vec}(\fT(\vtheta))=\mA\mathrm{vec}(\vtheta)+\vc$, where $M=\sum_{l=0}^{L-1}(m_l+1) m_{l+1}$ and $M'=\sum_{l=0}^{L-1}(m'_l+1) m'_{l+1}$. As noted before, we do not distinguish tuple $\vtheta$ from its vectorization $\mathrm{vec}(\vtheta)$ in the following. Hence, $\fT(\vtheta)=\mA\vtheta+\vc$.
\end{rmk}
% \begin{definition}[\bf{Affine embedding}]\label{def..afine}
%  Given  a   $\mathrm{NN}(\{m_l\}_{l=0}^{L})$ and  $\mathrm{NN}(\{m'_l\}_{l=0}^{L})$, an  
% {\textbf {affine embedding}} is an injective operator $\fT:\sR^M\to\sR^{M'}$, with $M=\sum_{l=0}^{L-1}(m_l+1) m_{l+1}$ and $M'=\sum_{l=0}^{L-1}(m'_l+1) m'_{l+1}$, which maps  $\vtheta_{\rnarr}\in\sR^M$ to   $\vtheta_{\rwide}=\fT\vtheta_{\rnarr}\in\sR^{M'}$ satisfying that:
% There exists $\mA\in\sR^{M'\times M}$, $\va\in  \sR^{M'}$, such that
% \[
% \fT(\vtheta):=\mA\vtheta+\va.
% \]
% \end{definition}
\begin{definition}[\bf{Critical embedding}]
Given  an   $\mathrm{NN}(\{m_l\}_{l=0}^{L})$ and  $\mathrm{NN}(\{m'_l\}_{l=0}^{L})$, a {\textbf{critical embedding}} is an affine embedding $\fT:\mathrm{Tuple}_{\{m_0,\cdots,m_L\}}\to\mathrm{Tuple}_{\{m'_0,\cdots,m'_L\}}$, which maps any set of its network parameters $\vtheta_{\rnarr}\in\mathrm{Tuple}_{\{m_0,\cdots,m_L\}}$ to that of a wider NN $\vtheta_{\rwide}=\fT(\vtheta_{\rnarr})\in\mathrm{Tuple}_{\{m'_0,\cdots,m'_L\}}$ satisfying that: For any given data $S$, loss function $\ell(\cdot,\cdot)$, activation function $\sigma(\cdot)$, \\
(i) {\bf output preserving:} $\vf_{\vtheta_{\rnarr}}(\vx) = \vf_{\vtheta_{\rwide}}(\vx)$ for any $\vx\in\sR^d$; \\
(ii) {\bf representation preserving:} \[\mathrm{span}\left\{\left\{\left(\vf^{[l]}_{\vtheta_{\rnarr}}(\cdot)\right)_j\right\}_{j\in[m_l]}\cup\{1\}\right\}=\mathrm{span}\left\{\left\{\left(\vf^{[l]}_{\vtheta_{\rwide}}(\cdot)\right)_{j'}\right\}_{j'\in[m'_l]}\cup\{1\}\right\},~~\text{for any}~l\in[L],\]
where $\{\vf^{[l]}_{\vtheta_{\rnarr}}\}_{l=1}^L$ and $\{\vf^{[l]}_{\vtheta_{\rwide}}\}_{l=1}^L$ are feature vectors of  $\mathrm{NN}(\{m_l\}_{l=0}^{L})$ and  $\mathrm{NN}(\{m'_l\}_{l=0}^{L})$, and $1:\sR^d\to\sR$ is the constant function, i.e., $1(\cdot)\equiv 1$; \\
(iii) {\bf criticality preserving:} If $\vtheta_{\rnarr}$ is a critical point of $\RS({\vtheta})$, i.e., $\nabla_{\vtheta}\RS(\vtheta_{\rnarr})=\mzero$, then $\vtheta_{\rwide}$ is also a critical point of $\RS({\vtheta})$, i.e., $\nabla_{\vtheta}\RS(\vtheta_{\rwide})=\mzero$.\\
Specifically, if an embedding is a critical embedding with $\mathrm{NN}(\{m'_l\}_{l=0}^{L})$ one-neuron wider than $\mathrm{NN}(\{m_l\}_{l=0}^{L})$,
% with $m'_{l_0}=m_{l_0}+1$ for some $l_0 \in [L-1]$, and $m'_l=m_l$ for $l\in [L-1]\backslash\{l_0\}$, 
we call it {\textbf{one-step critical embedding}}.
\end{definition}

% \begin{rmk}
% The requirement of critical embedding being an affine operator ensures the consistency of $\fT$ over the the whole parameter space, i.e., $\fT$ over a neighbourhood of any point in the parameter space can be uniquely extended to the $\fT$  over the whole parameter space.
% \end{rmk}

% \begin{definition}[\textbf{linear embedding}]
% Embedding operator $\fT:\mathrm{Tuple}_{\{m_0,\cdots,m_L\}}\to\mathrm{Tuple}_{\{m'_0,\cdots,m'_L\}}$ is linear if $\fT(\vtheta_1)+\fT(\vtheta_2)=\fT(\vtheta_1+\vtheta_2)$ and $\fT(\beta\vtheta)=\beta\fT(\vtheta)$ for any $\vtheta,\vtheta_1,\vtheta_2\in\mathrm{Tuple}_{\{m_0,\cdots,m_L\}},\ \beta\in\sR$.
% \end{definition} 

\subsection{One-step critical embedding}

% \begin{definition}[\textbf{one-step embedding}]
% Embedding realized by an embedding operator $\fT:\mathrm{Tuple}_{\{m_0,\cdots,m_L\}}\to\mathrm{Tuple}_{\{m'_0,\cdots,m'_L\}}$ from width-$\{m'_0,\cdots,m'_L\}$ NN to a one-neuron wider NN of width-$\{m'_0,\cdots,m'_L\}$, i.e., $K=\sum_{l\in[L-1]}=1$, is a one-step embedding.
% % $\exists j\in[L-1]$, s.t., $m'_j=m_j+1$ and $m'_k=m_k$ for $k\neq j$.
% \end{definition} 
In this subsection, we introduce two types of one-step critical embeddings. we  start with the definition of one-step null embedding. Intuitively, a one-step null embedding adds  a null neuron to the NN, whose output is a constant independent of input $\vx$ for any activation.

\begin{definition}[\textbf{One-step null embedding}]
Given  an   $\mathrm{NN}(\{m_l\}_{l=0}^{L})$ and its parameter\\
$\vtheta=(\mW^{[1]},\vb^{[1]},\cdots,\mW^{[L]},\vb^{[L]})\in\mathrm{Tuple}_{\{m_0,\cdots,m_L\}}$,
then    for any $l\in[L-1]$, we define the operators $\fT_{l,0}$ and $\fV_{l,0}$ applying on $\vtheta$ as follows
\begin{align*}
	\fT_{l,0}(\vtheta)|_k
	&=\vtheta|_k,\quad k\neq l,~l+1,\\
	\fT_{l,0}(\vtheta)|_l
	&= \left(\left[ {\begin{array}{cc}
			\mW^{[l]} \\
			\vzero_{1\times m_{l-1}} \\
	\end{array} } \right],
	\left[ {\begin{array}{cc}
			\vb^{[l]} \\
			0 \\
	\end{array} } \right]\right),\\
	\fT_{l,0}(\vtheta)|_{l+1} &= \left(
	\left[\mW^{[l+1]},\mzero_{m_{l+1}\times 1}\right],
	\vb^{[l+1]}\right),\\
	\fV_{l,0}(\vtheta)|_k
	&=\left(\mzero_{m_{k}\times m_{k-1}},\mzero_{m_{k}\times 1}\right),\quad k\neq l,~l+1\\
	\fV_{l,0}(\vtheta)|_l
	&=\left(\mzero_{(m_{l}+1)\times m_{l-1}},\left[ {\begin{array}{cc}
			\vzero_{m_{l}\times 1} \\
			1 \\
	\end{array} } \right]\right),\\
	\fV_{l,0}(\vtheta)|_{l+1}
	&= \left(
	\mzero_{m_{l+1}\times {(m_l+1)}},
	\mzero_{m_{l+1}\times 1}\right).
\end{align*}
We define {\textbf{one-step null embedding}}   $\fT_{l,0}^{\alpha}$  as: For any $\vtheta\in\mathrm{Tuple}_{\{m_0,\cdots,m_L\}}$, 
\begin{equation*}
	\fT_{l,0}^{\alpha}(\vtheta)=(\fT_{l,0}+\alpha\fV_{l,0})(\vtheta).
\end{equation*}
\end{definition}
Note that the neuron added by the above one-step null embedding has zero output weights, zero input weights and an arbitrary bias. An illustration is shown in Fig. \ref{fig:onestepnull}. 

\begin{figure}[h]
\centering
\includegraphics[width=0.9\textwidth]{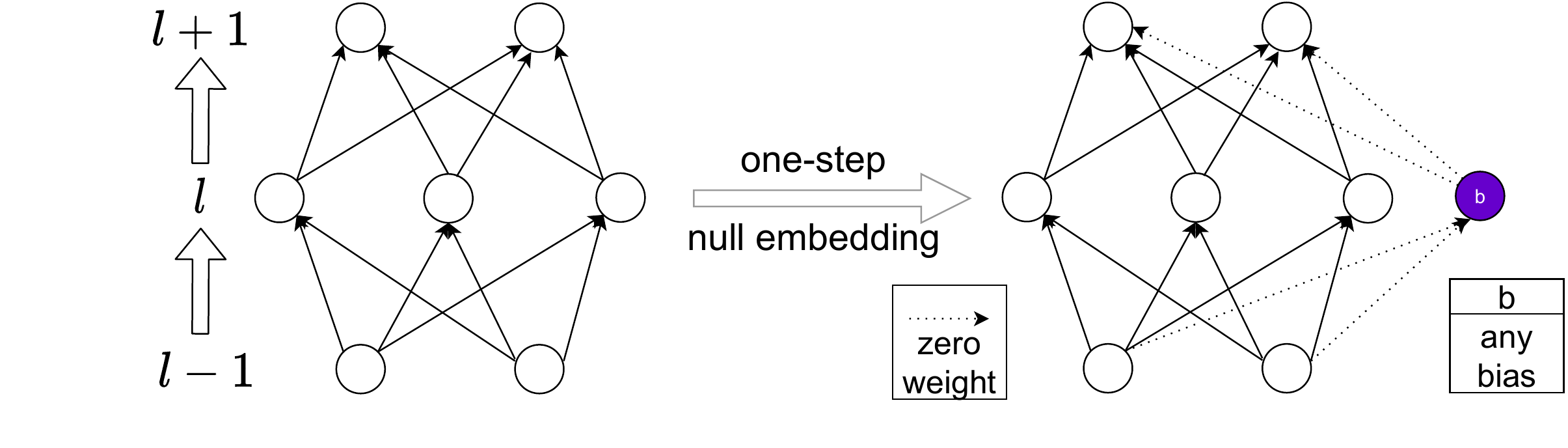}
\caption{Illustration of one-step null embedding. The purple neuron is added with both input and output weights as zero and an arbitrary bias. \label{fig:onestepnull}}
\end{figure} 

We proceed to the definition of one-step splitting embedding.
\begin{definition}[\textbf{One-step splitting embedding}]\label{def4}
Given  an   $\mathrm{NN}(\{m_l\}_{l=0}^{L})$ and its parameter
$\vtheta=(\mW^{[1]},\vb^{[1]},\cdots,\mW^{[L]},\vb^{[L]})\in\mathrm{Tuple}_{\{m_0,\cdots,m_L\}}$,  then for any $l\in[L-1]$ and  $s\in[m_l]$, we define the  operators $\fT_{l,s}$ and $\fV_{l,s}$ applying on $\vtheta$ as follows
\begin{align*}
	\fT_{l,s}(\vtheta)|_k
	&=\vtheta|_k,\quad k\neq l,l+1,\\
	\fT_{l,s}(\vtheta)|_l
	&= \left(\left[ {\begin{array}{cc}
			\mW^{[l]} \\
			\mW^{[l]}_{s,[1:m_{l-1}]} \\
	\end{array} } \right],
	\left[ {\begin{array}{cc}
			\vb^{[l]} \\
			\vb^{[l]}_s \\
	\end{array} } \right]\right),\\
	\fT_{l,s}(\vtheta)|_{l+1}
	&= \left(
	\left[\mW^{[l+1]},\mzero_{m_{l+1}\times 1}\right],
	\vb^{[l+1]}\right),\\
	\fV_{l,s}(\vtheta)|_k
	&=\left(\mzero_{m_{k}\times m_{k-1}},\mzero_{m_{k}\times 1}\right), k\neq l,l+1,\\
	\fV_{l,s}(\vtheta)|_l
	&=\left(\mzero_{(m_{l}+1)\times m_{l-1}},\mzero_{(m_{l}+1)\times 1}\right),\\
	\fV_{l,s}(\vtheta)|_{l+1}
	&= \left(
	\left[\mzero_{m_{l+1}\times (s-1)},-\mW^{[l+1]}_{[1:m_{l+1}],s},\mzero_{m_{l+1}\times (m_{l}-s)},\mW^{[l+1]}_{[1:m_{l+1}],s}\right],\mzero_{m_{l+1}\times 1}\right).
\end{align*}
We define {\textbf{one-step splitting embedding}}   $\fT_{l,s}^{\alpha}$  as: For any $\vtheta\in\mathrm{Tuple}_{\{m_0,\cdots,m_L\}}$, 
\begin{equation*}
	\fT_{l,s}^{\alpha}(\vtheta)=(\fT_{l,s}+\alpha\fV_{l,s})(\vtheta).
\end{equation*}
\end{definition}
An illustration of one-step splitting method is shown in Fig. \ref{fig:onestepsplitting}.
\begin{figure}[h]
\centering
\includegraphics[width=0.9\textwidth]{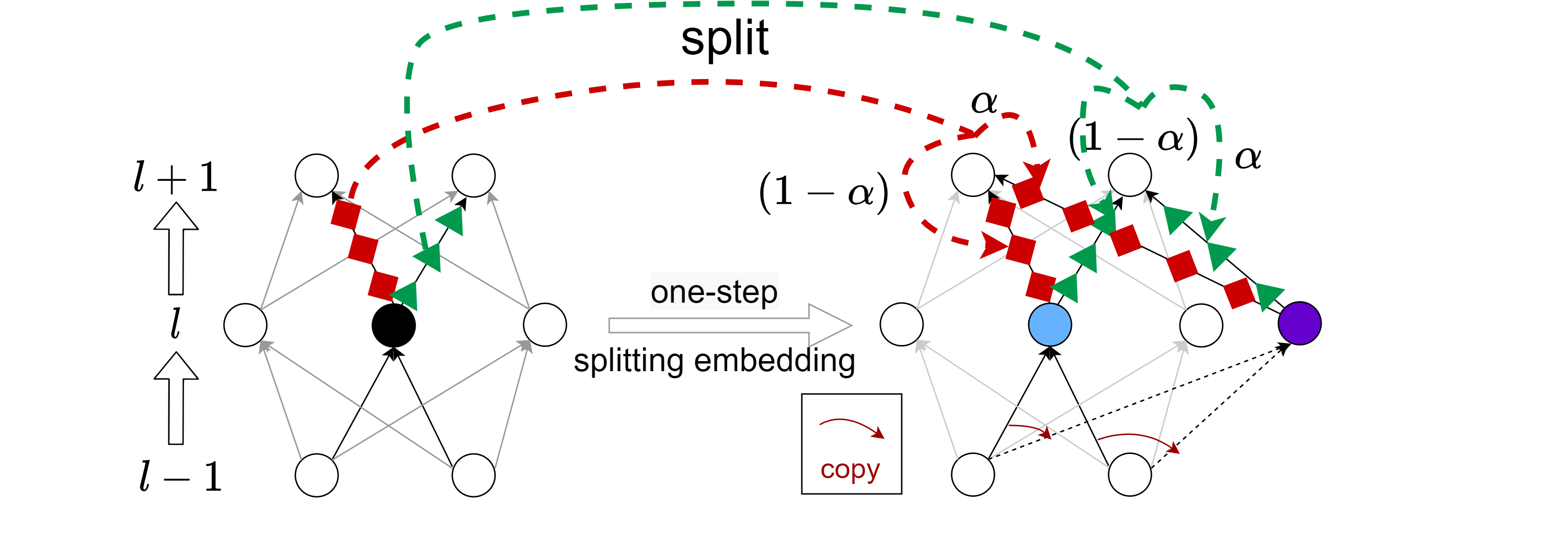}
\caption{Illustration of one-step splitting embedding. The black neuron in the left network is split into the blue and purple neurons in the right network. The red (green) output weight of the black neuron in the left net is split into two red (green) weights in the right net with ratio $(1-\alpha)$ and $\alpha$, respectively. This is also illustrated in \cite{zhang2021embedding}. \label{fig:onestepsplitting}}
\end{figure} 
\begin{rmk}
The parameters $\fT_{l,0}^{\alpha}(\vtheta)$ and $\fT_{l,s}^{\alpha}(\vtheta)$ correspond to a $L$-layer NN with width $(m_0,\ldots,m_{l-1},m_l+1,m_{l+1},\ldots,m_{L})$ since   $\fT_{l,0}^{\alpha}$ and $\fT_{l,s}^{\alpha}$ are one-step embeddings.
\end{rmk}

\begin{rmk}
We observe that   $\fT_{l,0}^{\alpha}$ and $\fT_{l,s}^{\alpha}$ can be applied on the neural network parameter $\vtheta$ of any given  $\mathrm{NN}(\{m_l\}_{l=0}^{L})$ of proper depth and width, hence the domain of $\fT_{l,0}^{\alpha}$ and $\fT_{l,s}^{\alpha}$ is not limited to a specific $\mathrm{Tuple}_{\{m_0,\cdots,m_L\}}$. Instead, the extended domain  of $\fT_{l,0}^{\alpha}$ and $\fT_{l,s}^{\alpha}$    assembles all  possible tuple class, i.e., if we denote the extended  domain of $\fT_{l,0}^{\alpha}$ by $\fD_{l,0}$, then
\begin{equation}
	\fD_{l,0}:=\bigsqcup_{ l< L}\mathrm{Tuple}_{\{n_0,\cdots,n_L\}},
\end{equation}
and if we denote the extended  domain of  $\fT_{l,s}^{\alpha}$ by $\fD_{l,s}$, then
\begin{equation}
	\fD_{l,s}:=\bigsqcup_{ l< L, s\leq n_l}\mathrm{Tuple}_{\{n_0,\cdots,n_L\}},
\end{equation}
where $\bigsqcup$ refers to the disjoint union of different tuple classes. 
\end{rmk}
By the above remark, we may extend $\fT_{l,0}^{\alpha}$ and $\fT_{l,s}^{\alpha}$ to their extended domains, and we identify
\[
\fT_{l,0}^{\alpha}:\fD_{l,0}\to\fD_{l,0},~~\fT_{l,s}^{\alpha}:\fD_{l,s}\to\fD_{l,s},
\]
with their restrictions on $\mathrm{Tuple}_{\{m_0,\cdots,m_L\}}$  for some given  $\mathrm{NN}(\{m_l\}_{l=0}^{L})$.

\begin{thm}\label{thm...Null+SplitisCritical}
One-step null embedding and one-step splitting embedding are critical embeddings.
\end{thm}
%%%%%%%%%%%%%%%%%%%%%%%%%%%%%%%%%%%%%%%%%%%%%%%%%%%%%%%%%%%%%
In order to prove Theorem \ref{thm...Null+SplitisCritical}, we need several lemmas.
%%%%%%%%%%%%%%%%%%%%%%%%%%%%%%%%%%%%%%%%
\begin{lemma}\label{lem...appen...section1...null} 
For any one-step null embedding $\fT_{l,0}^{\alpha}$, given any  $\mathrm{NN}(\{m_l\}_{l=0}^{L})$ and its parameters
$\vtheta_{\rnarr}\in\mathrm{Tuple}_{\{m_0,\cdots,m_L\}}$ with $\mathrm{Tuple}_{\{m_0,\cdots,m_L\}}\in\fD_{l,0}$, we have $\vtheta_{\rwide}:=\fT_{l,0}^{\alpha}(\vtheta_{\rnarr})$ satisfies the following conditions: given any data $S$, loss $\ell(\cdot,\cdot)$ and activation $\sigma(\cdot)$, for any $l\in[L-1]$,\\
(i) feature vectors in $\vF_{\vtheta_{\rwide}}$: $\vf^{[l']}_{\vtheta_{\rwide}}=\vf^{[l']}_{\vtheta_{\rnarr}}$, for $l'\in[L]$ and $l'\neq l$, $\vf^{[l]}_{\vtheta_{\rwide}}=\left[(\vf_{\vtheta_{\rnarr}}^{[l]})^\T,\sigma(\alpha)\right]^\T$;\\
(ii) feature gradients in $\vG_{\vtheta_{\rwide}}$: $\vg^{[l']}_{\vtheta_{\rwide}}=\vg^{[l']}_{\vtheta_{\rnarr}}$, for $l'\in[L]$ and $l'\neq l$, $\vg^{[l]}_{\vtheta_{\rwide}}=
\left[(\vg^{[l]}_{\vtheta_{\rnarr}})^\T,\sigma^{(1)}(\alpha)\right]^\T$;\\
(iii) error vectors in $\vZ_{\vtheta_{\rwide}}$: 
$\vz^{[l']}_{\vtheta_{\rwide}}=\vz^{[l']}_{\vtheta_{\rnarr}}$, for $l'\in[L]$ and $l'\neq l$, $\vz^{[l]}_{\vtheta_{\rwide}}=
\left[ \vz_{\vtheta_{\rnarr}}^{[l]},0\right]^\T$;\\
(iv) $\fT_{l,0}^{\alpha}$ is injective for all $\alpha$;\\
(v) $\fT_{l,0}^{\alpha}$ is an affine embedding for all $\alpha$.
\end{lemma}
%%%%%%%%%%%%%%%%%%%%%%%%%%%%%%%%%%%%%%%%%%%%%%%%%%%%%%%%%%%%%%%%%%%%%%%%%%%%%%%%
%%%%%%%%%%%%%%%%%%%%%%%%%%%%%%%%%%%%%%%%

\begin{lemma} \label{lem...appen...section1...split} 
For any one-step splitting embedding $\fT_{l,s}^{\alpha}$, given any  $\mathrm{NN}(\{m_l\}_{l=0}^{L})$ and its parameters
$\vtheta_{\rnarr}\in\mathrm{Tuple}_{\{m_0,\cdots,m_L\}}$ with $\mathrm{Tuple}_{\{m_0,\cdots,m_L\}}\in\fD_{l,s}$, we have $\vtheta_{\rwide}:=\fT_{l,s}^{\alpha}(\vtheta_{\rnarr})$ satisfies the following conditions: given any data $S$, loss $\ell(\cdot,\cdot)$ and activation $\sigma(\cdot)$, for any $l\in[L-1]$,\\
(i) feature vectors in $\vF_{\vtheta_{\rwide}}$: $\vf^{[l']}_{\vtheta_{\rwide}}=\vf^{[l']}_{\vtheta_{\rnarr}}$, for $l'\in[L]$ and $l'\neq l$, $\vf^{[l]}_{\vtheta_{\rwide}}=\left[(\vf_{\vtheta_{\rnarr}}^{[l]})^\T,(\vf_{\vtheta_{\rnarr}}^{[l]})_s\right]^\T$;\\
(ii) feature gradients in $\vG_{\vtheta_{\rwide}}$: $\vg^{[l']}_{\vtheta_{\rwide}}=\vg^{[l']}_{\vtheta_{\rnarr}}$, for $l'\in[L]$ and $l'\neq l$, $\vg^{[l]}_{\vtheta_{\rwide}}=
\left[(\vg^{[l]}_{\vtheta_{\rnarr}})^\T,(\vg_{\vtheta_{\rnarr}}^{[l]})_s\right]^\T$;\\
(iii) error vectors in $\vZ_{\vtheta_{\rwide}}$: 
$\vz^{[l']}_{\vtheta_{\rwide}}=\vz^{[l']}_{\vtheta_{\rnarr}}$, \\
for $l'\in[L]$ and $l'\neq l$, $\vz^{[l]}_{\vtheta_{\rwide}}=
\left[ \left(\vz_{\vtheta_{\rnarr}}^{[l]}\right)^\T_{[1:s-1]},(1-\alpha)(\vz_{\vtheta_{\rnarr}}^{[l]})_s,\left(\vz_{\vtheta_{\rnarr}}^{[l]}\right)^\T_{[s+1:m_l]}, \alpha(\vz_{\vtheta_{\rnarr}}^{[l]})_s\right]^\T$;\\
(iv) $\fT_{l,s}^{\alpha}$ is injective for all $\alpha$.\\
(v) $\fT_{l,s}^{\alpha}$ is an affine embedding for all $\alpha$.
\end{lemma}
%%%%%%%%%%%%%%%%%%%%%%%%%%%%%%%%%%%%%%%%%%%%%%%%%%%%%

%%%%%%%%%%%%%%%%%%%%%%%%%%%%%%%%%%%%%%%%
Directly from Lemma \ref{lem...appen...section1...null} and Lemma \ref{lem...appen...section1...split}, we obtain that both one-step null embedding and one-step splitting embedding satisfy the property of output preserving and representation preserving, and all we need is to check the property of criticality preserving.
%%%%%%%%%%%%%%%%%%%%%%%%%%%%%%%%%%%%%%%%
%%%%%%%%%%%%%%%%%%%%%%%%%%%%%%%%%%%%%%%%
\begin{prop}\label{prop...null}
For any one-step null embedding $\fT_{l,0}^{\alpha}$, given any  $\mathrm{NN}(\{m_l\}_{l=0}^{L})$ and its parameters
$\vtheta_{\rnarr}\in\mathrm{Tuple}_{\{m_0,\cdots,m_L\}}$ with $\mathrm{Tuple}_{\{m_0,\cdots,m_L\}}\in\fD_{l,0}$, we have $\vtheta_{\rwide}:=\fT_{l,0}^{\alpha}(\vtheta_{\rnarr})$ satisfies the following conditions: given any data $S$, loss $\ell(\cdot,\cdot)$ and activation $\sigma(\cdot)$,
if $\nabla_{\vtheta}\RS(\vtheta_{\rnarr})=\mzero$, then $\nabla_{\vtheta}\RS(\vtheta_{\rwide})=\mzero$.
\end{prop}
%%%%%%%%%%%%%%%%%%%%%%%%%%%%%%%%%%%%%%%%%%%%%%%%%%%%%%%%

%%%%%%%%%%%%%%%%%%%%%%%%%%%%%%%%%%%%%%%%%%%%%%%%%%%%%%%%
%%%%%%%%%%%%%%%%%%%%%%%%%%%%%%%%%%%%%%%%
\begin{prop}\label{prop...split}
For any one-step splitting embedding $\fT_{l,s}^{\alpha}$, given any  $\mathrm{NN}(\{m_l\}_{l=0}^{L})$ and its parameters
$\vtheta_{\rnarr}\in\mathrm{Tuple}_{\{m_0,\cdots,m_L\}}$ with $\mathrm{Tuple}_{\{m_0,\cdots,m_L\}}\in\fD_{l,s}$, we have $\vtheta_{\rwide}:=\fT_{l,s}^{\alpha}(\vtheta_{\rnarr})$ satisfies the following conditions: given any data $S$, loss $\ell(\cdot,\cdot)$ and activation $\sigma(\cdot)$, if $\nabla_{\vtheta}\RS(\vtheta_{\rnarr})=\mzero$, then $\nabla_{\vtheta}\RS(\vtheta_{\rwide})=\mzero$.
\end{prop}
%%%%%%%%%%%%%%%%%%%%%%%%%%%%%%%%%%%%%%%%%%%%%%%%%%%%%%%%
%%%%%%%%%%%%%%%%%%%%%%%%%%%%%%%%%%%%%%%%%%%%%%%%%%%%%%%%
Combining altogether Lemma \ref{lem...appen...section1...null}, Lemma \ref{lem...appen...section1...split}, Proposition \ref{prop...null} and Proposition \ref{prop...split}, we finish our proof for Theorem \ref{thm...Null+SplitisCritical}.
%%%%%%%%%%%%%%%%%%%%%%%%%%%%%%%%%%%%%%%%%%%%%%%%%%%%%%%%
\subsection{Composition of one-step embeddings and the Embedding Principle}
We first define the composition of two embeddings, which would naturally leads to the composition of arbitrary embeddings.

\begin{definition}[Composition of  two embeddings]
Suppose we have  an $\mathrm{NN}(\{m_l\}_{l=0}^{L})$ and its  parameters $\vtheta\in \mathrm{Tuple}_{\{m_0,\cdots,m_L\}}$, and we have two embeddings $\fT$ and $\fT'$ satisfying \\ $\fT:\mathrm{Tuple}_{\{m_0,\cdots,m_L\}}\to\mathrm{Tuple}_{\{m'_0,\cdots,m'_L\}}$, $\fT':\mathrm{Tuple}_{\{m'_0,\cdots,m'_L\}}\to\mathrm{Tuple}_{\{m''_0,\cdots,m''_L\}}$, with $\fT'$ maps the range of $\fT$ into  $\mathrm{Tuple}_{\{m''_0,\cdots,m''_L\}}$, where    $m''_0=m'_0=m_0$, $m''_L=m'_L=m_L$, and for any $l\in[L-1]$, $m''_l\geq m'_l\geq m_l$. Since $\fT'\fT$ is obviously an injective operator, then $\fT'\fT$ is an embedding $\fT'\fT:\mathrm{Tuple}_{\{m_0,\cdots,m_L\}}\to\mathrm{Tuple}_{\{m''_0,\cdots,m''_L\}}$, and we term $\fT'\fT$ the composition of $\fT'$ and $\fT$,i.e., for any  $\vtheta\in \mathrm{Tuple}_{\{m_0,\cdots,m_L\}}$,
\[
\fT'\fT(\vtheta):=\fT'(\fT(\vtheta)).
\]
\end{definition}

For simplicity, for any $K\in\sZ^+, K\geq 2$, we denote hereafter    $\prod_{l=1}^{K}\fT_l:=\fT_K\cdots\fT_1$ as the composition of $K$ individual embeddings $\left\{\fT_l\right\}_{l=1}^K$. For $K=1$, $\prod_{l=1}^{1}\fT_l:=\fT_1$.
\begin{definition}[\textbf{$K$-step~(Multi-step) composition embedding}]\label{def...Kstep}
Suppose we have two vectors $\vl=(l_k)_{k=1}^K,~l_k\in[L-1]$,   $\valpha=(\alpha_k)_{k=1}^K\subset\sR^K$, and a sequence $\{m^{(0)}_l\}_{l=1}^L$ with $m^{(0)}_l:=m_l$ for $l\in[L]$.
Then given  an   $\mathrm{NN}(\{m_l\}_{l=0}^{L})$ and its parameters
$\vtheta$, a {\textbf{$K$-step composition embedding}},  $\fT:\mathrm{Tuple}_{\{m_0,\cdots,m_L\}}\to\mathrm{Tuple}_{\{m'_0,\cdots,m'_L\}}$  with $K=\sum_{l=1}^{L-1}m'_j-\sum_{l=1}^{L-1}m_l$,  is defined recursively by the composition of $K$ one-step null embeddings or one-step splitting embeddings.\\
Formally speaking, a $K$-step composition embedding $\fT^{\valpha}_{\vl,\vs}$ is defined recursively as follows:\\
For $n=1$, choose $s_1\in[m^{(0)}_{l_1}]\cup\{0\}$, then
\[    
\fT^{\valpha,(1)}_{\vl,\vs} :=\fT_{l_1,s_1}^{\alpha_1}=\fT_{l_1,s_1}+\alpha_1 \fV_{l_1,s_1}.
\]
Update the sequence from $\{m^{(0)}_l\}_{l=1}^L$ to $\{m^{(1)}_l\}_{l=1}^L$ following: $m^{(1)}_{l}=m^{(0)}_{l}$ for $l\in[L-1]\backslash \{l_1\}$, and  $m^{(1)}_{l_1}=m^{(0)}_{l_1}+1$;\\
Then inductively, for $n=k$,     choose $s_k\in[m^{(k-1)}_{l_k}]\cup\{0\}$, then
\[    
\fT^{\valpha,(k)}_{\vl,\vs} :=\fT_{l_k,s_k}^{\alpha_k}\fT^{\valpha,(k-1)}_{\vl,\vs}.
\]
Update the sequence from $\{m^{(k-1)}_l\}_{l=1}^L$ to $\{m^{(k)}_l\}_{l=1}^L$ following: $m^{(k)}_{l}=m^{(k-1)}_{l}$ for $l\in[L-1]\backslash \{l_k\}$, and  $m^{(k)}_{l_k}=m^{(k-1)}_{l_k}+1$.\\
Finally, $\fT:=\fT^{\valpha}_{\vl,\vs}:=\fT^{\valpha,(K)}_{\vl,\vs}$. 
\end{definition}

\begin{rmk}
For each $i\in[K]$, $\fT_{l_i,s_i}^{\alpha_i}$ is regarded as its restriction on the tuple class $\mathrm{Tuple}_{\{m_0^{(i-1)},\cdots,m_L^{(i-1)}\}}\in\fD_{l_i,s_i}$, hence 
\[
\fT_{l_i,s_i}^{\alpha_i}:\mathrm{Tuple}_{\{m_0^{(i-1)},\cdots,m_L^{(i-1)}\}}\to\mathrm{Tuple}_{\{m_0^{(i)},\cdots,m_L^{(i)}\}}.
\]
\end{rmk}

%%%%%%%%%%%%%%%%%%%%%%%%%%%%%%%%%%%%%%%%%%%%%%%%%
\begin{theorem}\label{prop...KStep}
A $K$-step composition embedding is a critical embedding. 
\end{theorem}
%%%%%%%%%%%%%%%%%%%%%%%%%%%%%%%%%%%%%%%%%%

The composition of one-step embeddings renders a feasible method to embedding any critical point of the loss landscape of a narrow NN  to a critical point with the same output function of any wider NN, therefore, we have the following Embedding Principle.

\begin{theorem}[Embedding Principle]\label{thm..embeddingPrinciple}
Given any NN and any  $K$-neuron wider NN, there exists a $K$-step composition embedding $\fT$ satisfying  that:
For any given data $S$, loss function $\ell(\cdot,\cdot)$, activation function $\sigma(\cdot)$,  given any   critical point $\vtheta^{\rc}_{\rnarr}$ of the narrower NN,   $\vtheta^{\rc}_{\rwide}:=\fT(\vtheta^{\rc}_{\rnarr})$ is still a critical point of the  $K$-neuron wider NN with the same output function, i.e., $\vf_{\vtheta^{\rc}_{\rnarr}}=\vf_{\vtheta^{\rc}_{\rwide}}$.
\end{theorem}

\begin{proof}
Existence of a $K$-step composition embedding $\fT$ can be seen from the $K$-step construction given in Definition \ref{def...Kstep}, and we finish the proof.
\end{proof}
Moreover, we obtain a corollary from Theorem \ref{thm..embeddingPrinciple} stating the Embedding Principle for the critical functions.
\begin{corollary}[Embedding Principle of critical functions] 
Given any NN and any wider NN, for any given data $S$, loss function $\ell(\cdot,\cdot)$ and  activation function $\sigma(\cdot)$,  
\begin{equation}
	\fF^{\rc}_{\rnarr}\subset\fF^{\rc}_{\rwide} 
\end{equation} 
where $\fF_{\rnarr}^{\rc}=\{\vf_{\vtheta}|\vtheta\in\vTheta_{\rnarr}^{\rc}\}$ and $\fF_{\rwide}^{\rc}=\{\vf_{\vtheta}|\vtheta\in\vTheta_{\rwide}^{\rc}\}$ are the sets of critical functions.   
\end{corollary}

\subsection{Importance of Embedding Principle}
Mathematically speaking, Embedding Principle is a natural result of an embedding operation that preserves output function and criticality. However, we must emphasize the importance of stating and proving it explicitly in this work. For a long time, researchers study the loss landscape of NN from an optimization perspective focusing specifically on its property in the parameter space. Lots of works make effort in tackling problems like whether bad local minima exist, whether local minima are also global minima and whether all saddle points are strict-saddle points, etc. However, because the loss landscape of NN also has profound impact on its implicit regularization and generalization performance,  it is important to look into the loss landscape from the perspective of function spaces. 

Motivated by the phenomena of Frequency Principle \citep{xu_training_2018,xu2019frequency,rahaman2018spectral,zhang2021linear,luo2019theory} and condensation \citep{luo2021phase}, we are very interested in the question of what are the critical functions of an NN loss landscape that may attract the training trajectory in the function space. Specifically, we care about whether there are ``simple'' critical functions in wide NNs that may implicitly regularize the training to help avoid overfitting.  For example, in our experiments shown in Fig. \ref{fig:EPrinciple_example}, we clearly observe that the training of a width-$500$ two-layer tanh-NN in fitting $50$ data points experiences two stages. At the first stage, it learns an output function close to the best fitting of the width-$1$ tanh-NN and stays for a while, seemingly that it encounters a saddle point.  At the second stage, it converges to an output function close to the best fitting of the width-$3$ tanh-NN, which interpolates all the data points. Clearly, the complexity of the output function of the width-$500$ NN gradually increases during the training, leading to a non-overfitting interpolation of data despite of possessing overfitting capability. However, before studying the universality of such training behavior clearly relevant to generalization, it is important to have a theoretical answer to whether such ``simple'' critical functions always exist even in very wide NNs.  

\begin{figure}[h]
\centering
\subfigure[]{\includegraphics[width=0.33\textwidth]{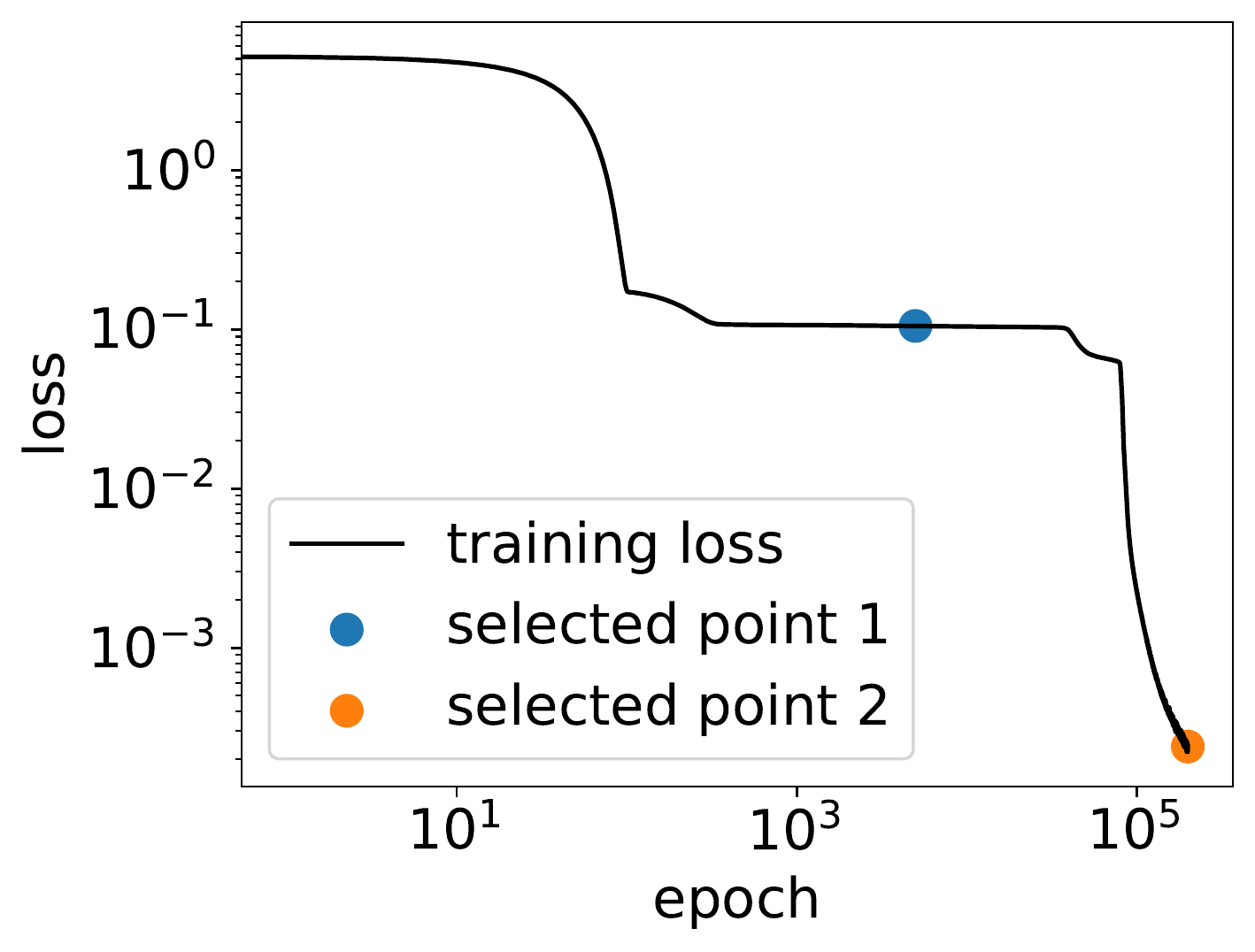}}
\subfigure[]{\includegraphics[width=0.31\textwidth]{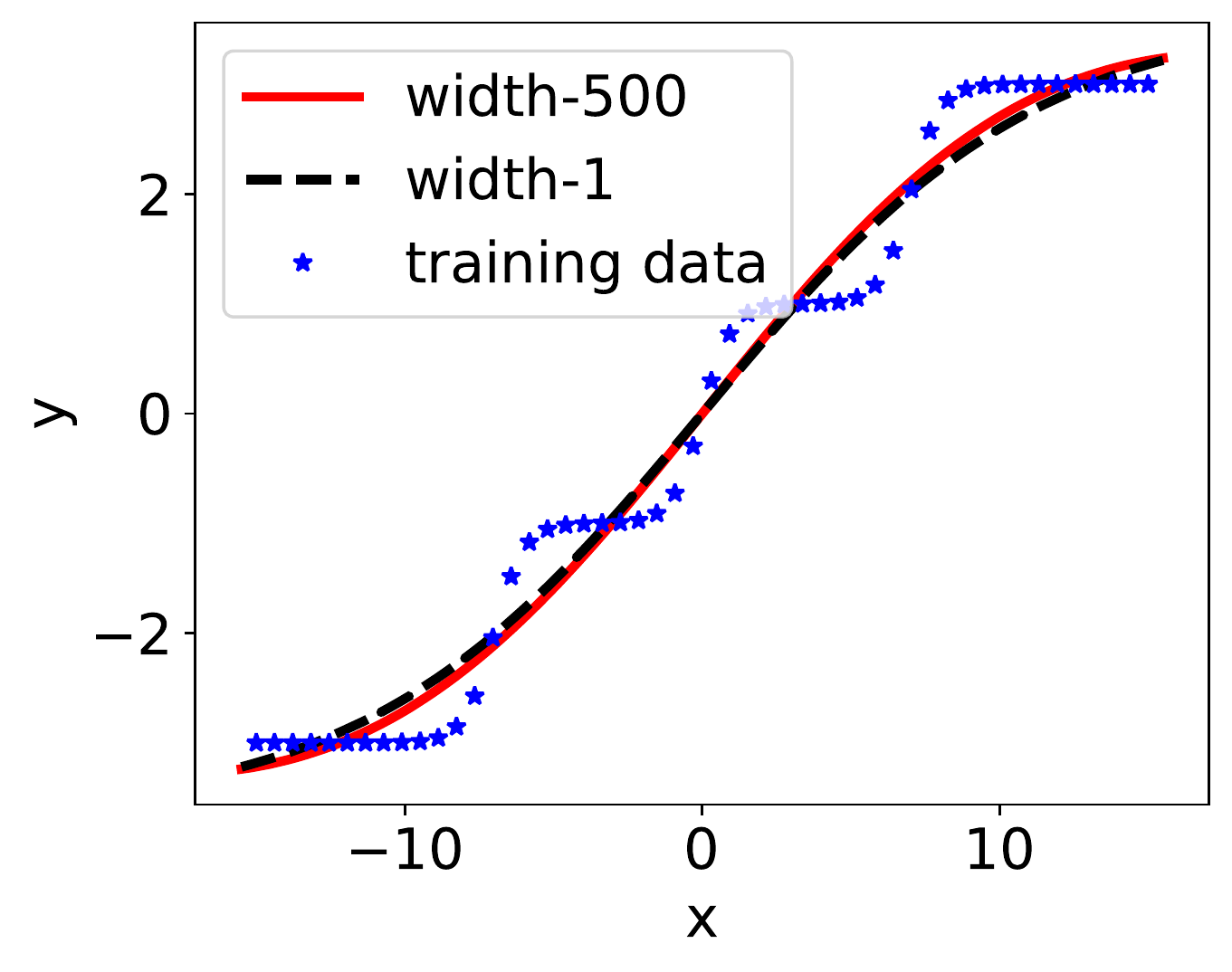}}
\subfigure[]{\includegraphics[width=0.31\textwidth]{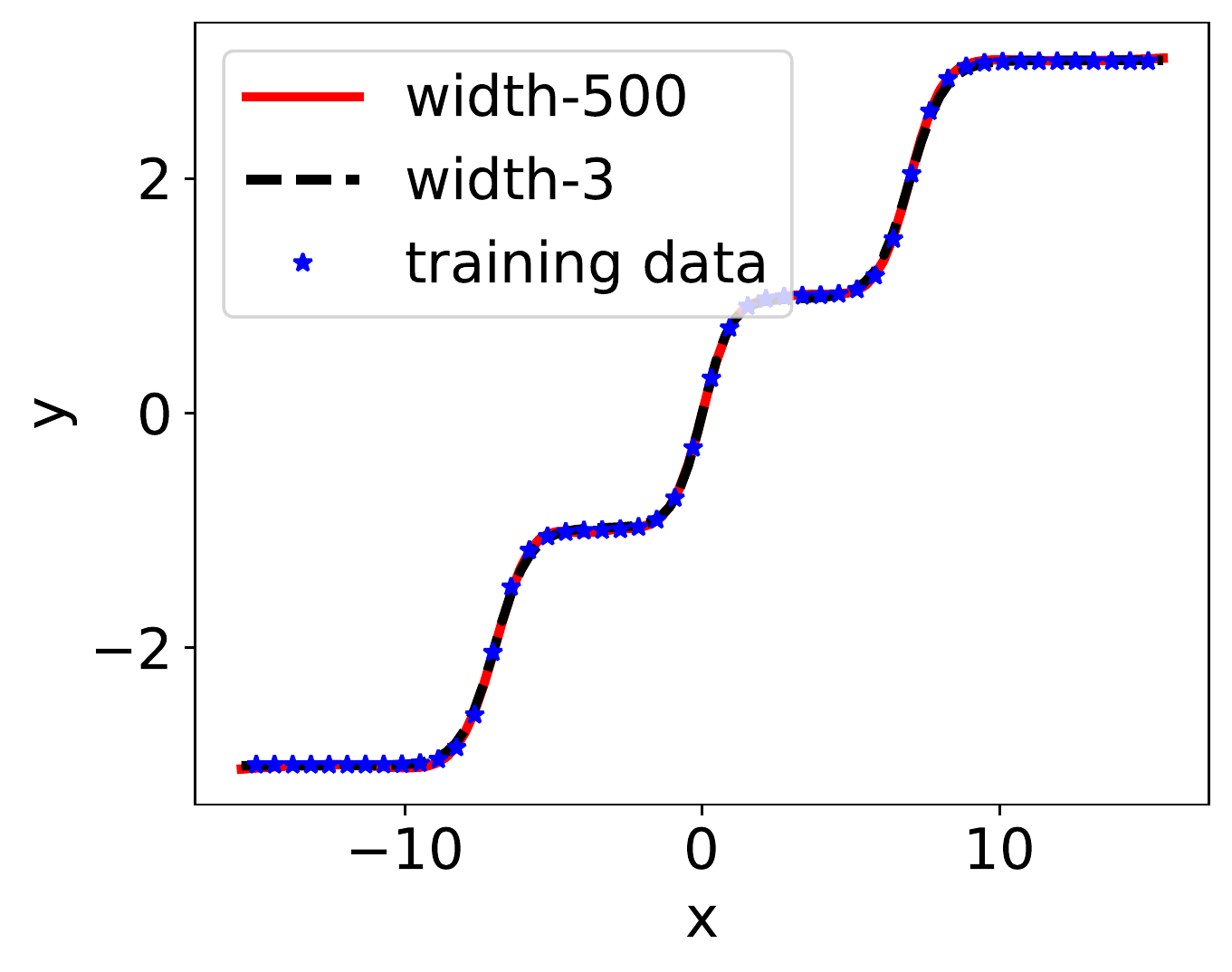}}
% 	\subfigure[]{\includegraphics[width=0.75\textwidth]{pic/tanh_example.png}}
\caption{(a) The training loss of a two-layer tanh neural network with $500$ hidden neurons. (b, c) Red solid: the DNN output at a training step, where the blue dot and the orange dot in (a) corresponds to (b) and (c), respectively; Black dashed: the output of the global minimum of the width-$1$ NN in (b) and the width-$3$ NN in (c), respectively; Blue dots: training data. This is also illustrated in \cite{zhang2021embedding}
	% 	(d) Some critical functions of width-$500$ tanh-NN obtained by the Embedding Principle.
	\label{fig:EPrinciple_example}}
\end{figure} 

By stating and proving the Embedding Principle explicitly, we provide a clear answer of ``YES'' to above question. Moreover, we unravel the exact meaning of ``simple'' critical functions---critical functions of narrower networks. An illustration of critical functions of a two-layer width-$m$ NN and a three-layer width-$\{m_1,m_2\}$ NN predicted by the Embedding Principle are shown in Fig. \ref{fig:EPrinciple_illus}(a), (b), respectively. Critical functions of the width-$500$ tanh-NN for experiments in Fig. \ref{fig:EPrinciple_example} are illustrated in Fig. \ref{fig:EPrinciple_illus}(c). Note that, by Frequency Principle \citep{xu_training_2018,xu2019frequency,rahaman2018spectral,zhang2021linear,luo2019theory}, we consider functions dominated by low frequency components as ``simple'' functions that training of an NN is implicitly biased to. Here, by Embedding Principle,  ``simple'' functions are those critical functions of narrow NNs, which signify a hierarchical structure of fittings of training data with different complexities indicated by the narrowerest NN a critical function belongs to.
Combining with empirical observations of Frequency Principle and condensation, we conjecture that nonlinear training of an NN is implicitly biased towards these ``simple'' critical functions. It is clearly important to look further into this conjecture in the future works.

\begin{figure}[h]
\centering
\subfigure[]{\includegraphics[width=0.5\textwidth]{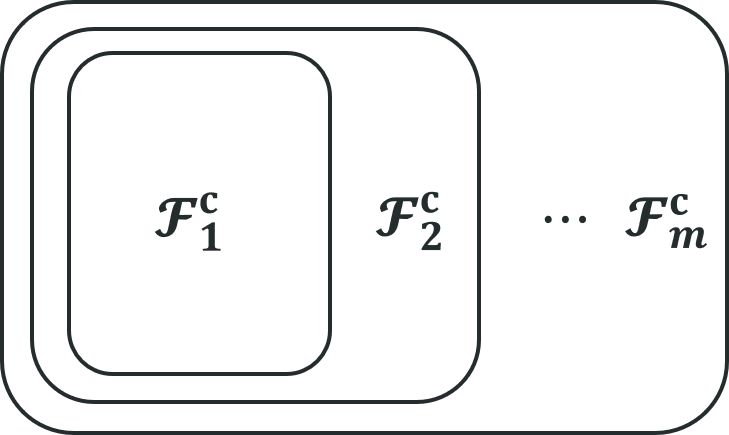}}\hspace{10pt}
\subfigure[]{\includegraphics[width=0.35\textwidth]{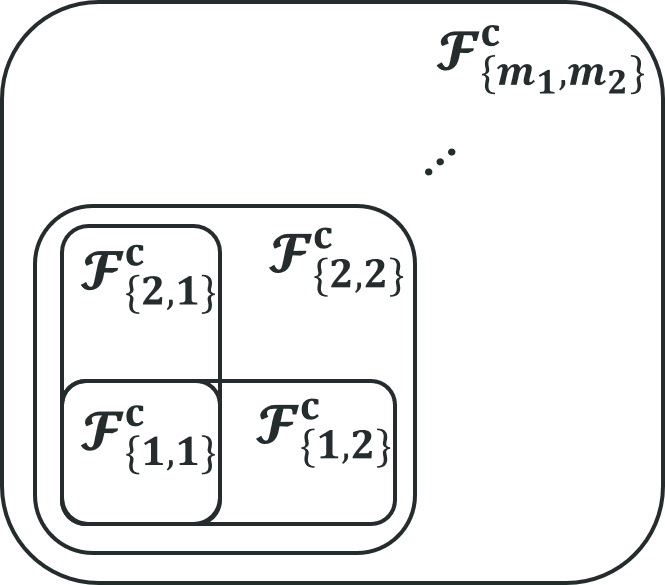}}\\
\subfigure[]{\includegraphics[width=0.75\textwidth]{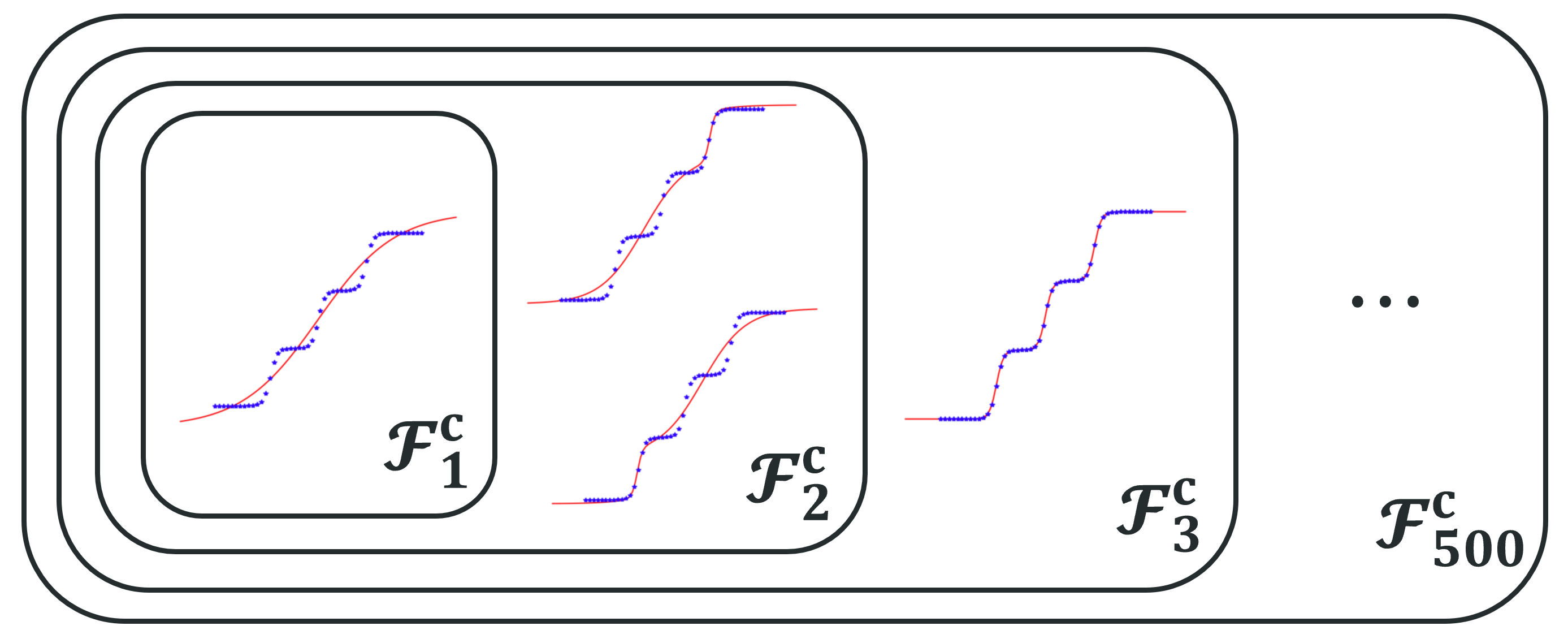}}
\caption{Illustration of the Embedding Principle for the critical functions of (a) two-layer width-$m$ NN and (b) three-layer width-$\{m_1,m_2\}$ NN. (c) Some critical functions of width-$500$ tanh-NN predicted by the Embedding Principle for experiments in Fig. \ref{fig:EPrinciple_example}. \label{fig:EPrinciple_illus}}
\end{figure}

\section{General compatible critical embedding}

Although the Embedding Principle has been proven above, discovering as many critical embeddings as possible remains an important task because the more critical embeddings we discover, the more we understand about the critical submanifolds of an NN, e.g., their dimension and geometry, corresponding to these critical functions of narrower NNs. In the following, we propose a much wider class of general compatible embeddings with above class of one-step and multi-step composition embeddings as its special cases.

To define a general compatible embedding from a narrow NN to a wide NN, we first define mappings that establish the relation between neuron indices of the narrow NN and neuron indices of the wide NN.

\begin{definition}[\textbf{Pull-back index mapping and total index mapping }]
Given  an   $\mathrm{NN}(\{m_l\}_{l=0}^{L})$ and a wider $\mathrm{NN}(\{m'_l\}_{l=0}^{L})$, a   \textbf{pull-back index mapping } $\fI:=\{\fI_l\}_{l=0}^{L}$  from the  wider NN to  the narrower NN is defined as follows:
For any fixed $l\in[L-1]$, $\fI_l$ maps a  neuron index $s'\in[m'_l]$ of the wider $\mathrm{NN}(\{m'_l\}_{l=0}^{L})$ in layer $l$ to a neuron index $s\in[m_l]\sqcup\{0\}$ in the same layer of the narrower $\mathrm{NN}(\{m_l\}_{l=0}^{L})$. As for the case of $l=0$ and $l=L$, $\fI_0$ and $\fI_L$ are always the identity maps since their indices are fixed once data is given with $m'_0=m_0=d$ and $m'_L=m_L=d'$. To sum up
\begin{equation}
	\fI_0: [m'_0]\to [m_0],~\fI_l: [m'_l]\to [m_l]\sqcup\{0\}~\text{for}~l\in[L-1],~\fI_L: [m'_L]\to [m_L].
\end{equation}
Moreover, for  any nonzero index $s\neq 0$, if the push-forward $\fI^{-1}_l(s)\neq \emptyset$ for all $l\in[0:L]$, we say that $\fI=\{\fI_l\}_{l=0}^{L}$ is a {\textbf{total (pull-back) index  mapping}}.
\end{definition}

\begin{lemma}\label{lem...Output.to.critPreser}
For any affine embedding $\fT:\mathrm{Tuple}_{\{m_0,\cdots,m_L\}}\to\mathrm{Tuple}_{\{m'_0,\cdots,m'_L\}}$ satisfying the output preserving property, 
% and we have two NNs, $\mathrm{NN}(\{m_l\}_{l=0}^{L})$ and  $\mathrm{NN}(\{m'_l\}_{l=0}^{L})$,   %then for any given    $\fT$    maps  its network parameters $\vtheta_{\rnarr}\in\mathrm{Tuple}_{\{m_0,\cdots,m_L\}}$ to that of a wider NN $\vtheta_{\rwide}=\fT(\vtheta_{\rnarr})\in\mathrm{Tuple}_{\{m'_0,\cdots,m'_L\}}$. \\
if there exist a total index mapping $\fI=\{\fI_l\}_{l=0}^{L}$ from $\mathrm{NN}(\{m'_l\}_{l=0}^{L})$ to $\mathrm{NN}(\{m_l\}_{l=0}^{L})$
and  auxiliary  variables $\vbeta=\left\{\vbeta^{[l]}_j\in\sR|~l\in[L],j\in[m'_l]\backslash\fI_l^{-1}(0)\right\}$, 
such that for any given neuron   belonging to $\mathrm{NN}(\{m'_l\}_{l=0}^{L})$, located in layer $l$ with index $j$, the following two statements hold:  \\
(i) If $\fI_l(j)\neq 0$, $(\vf_{\vtheta_{\rwide}}^{[l]})_j=(\vf_{\vtheta_{\rnarr}}^{[l]})_{\fI_l(j)}$ and $(\ve_{\vtheta_{\rwide}}^{[l]})_j=\vbeta_j^{[l]}(\ve_{\vtheta_{\rnarr}}^{[l]})_{\fI_l(j)}$,\\
(ii) If $\fI_l(j)= 0$,  $(\vf_{\vtheta_{\rwide}}^{[l]})_j=\mathrm{Const}$ and $(\ve_{\vtheta_{\rwide}}^{[l]})_j=0$,\\
then $\fT$ is a critical embedding.
\end{lemma}

Next, we propose a general compatible embedding method, where above embedding methods including one-step embedding and K-step composition embedding are its special cases. 
\begin{definition}[\textbf{General compatible   embedding}]\label{def...GeneralEmbedding}
Given  an   $\mathrm{NN}(\{m_l\}_{l=0}^{L})$ and a wider $\mathrm{NN}(\{m'_l\}_{l=0}^{L})$, then for any total index mapping $\fI=\{\fI_l\}_{l=0}^{L}$ from $\mathrm{NN}(\{m'_l\}_{l=0}^{L})$ to $\mathrm{NN}(\{m_l\}_{l=0}^{L})$, and for any tuple $\valpha:=\{\malpha^{[1]},\valpha^{[1]}_{\rb},\cdots, \malpha^{[L]},\valpha^{[L]}_{\rb}\}\in\mathrm{Tuple}_{\{m'_0,\cdots,m'_L\}}$ satisfying some  {\textbf{compatibility conditions}}~(see Condition \ref{state1} and Condition \ref{state2}),   we define a {\textbf{general   embedding}} $\fT^{\valpha}_{\fI}:\mathrm{Tuple}_{\{m_0,\cdots,m_L\}}\to\mathrm{Tuple}_{\{m'_0,\cdots,m'_L\}}$ as: For any parameters
$
{\vtheta_{\rnarr}}=(\mW_{{\rnarr}}^{[1]},\vb_{{\rnarr}}^{[1]},\cdots,\mW_{{\rnarr}}^{[L]},\vb_{{\rnarr}}^{[L]})\in \mathrm{Tuple}_{\{m_0,\cdots,m_L\}}
$,
\begin{align*}
	\fT^{\valpha}_{\fI}(\vtheta_{\rnarr}):=\Big(& \malpha^{[1]}\circ \mW^{[1]}_{\rinter}, \valpha^{[1]}_{\rb}+\vb^{[1]}_{\rinter},\cdots, \\
	&    \malpha^{[l]}\circ \mW^{[l]}_{\rinter}, \valpha^{[l]}_{\rb}+\vb^{[l]}_{\rinter},\cdots,     
	\\
	& \malpha^{[L]}\circ \mW^{[L]}_{\rinter}, \valpha^{[L]}_{\rb}+\vb^{[L]}_{\rinter}\Big),
\end{align*}
where $\mW^{[l]}_{\rinter}:=\left[\left(\mW^{[l]}_{\rnarr}\right)_{\fI_l(i),\fI_{l-1}(j)}\right]$, $\vb^{[l]}_{\rinter}:=\left(\left(\vb^{[l]}_{\rnarr}\right)_{\fI_l(k)}\right)$ for $l\in[L]$ with $i,~k\in[m_l']$, $j\in[m'_{l-1}]$, and $\circ$ is the  the Hadamard product.
\end{definition}
%%%%%%%%%%%%%%%%%%%%%%%%%%%%%%%%%%%%%%%%%%%%%
\begin{rmk}
Since $\fI_l: [m'_l]\to [m_l]\sqcup\{0\}$ for~$l\in[L-1]$, and  the components in   $\mW^{[l]}_{\rnarr}$ and $\vb^{[l]}_{\rnarr}$ are not defined for zero indices, i.e., no definitions can be found for $(\mW^{[l]}_{\rnarr})_{0j}$,  $(\mW^{[l]}_{\rnarr})_{i0}$, $(\vb^{[l]}_{\rnarr})_0$,  for any   $l\in[L-1]$ with $i\in[m_l]$ and $j\in[m_{l-1}]\sqcup\{0\}$, for convenience of expression,  we set  $(\mW^{[l]}_{\rnarr})_{0j}=1$, $(\mW^{[l]}_{\rnarr})_{i0}=1$, and $(\vb^{[1]}_{\rnarr})_0=0$,   with $i\in[m_l]$ and $j\in[m_{l-1}]\sqcup\{0\}$.
\end{rmk}
Now we proceed to state out the {\textbf{certain conditions}} for the tuple 
\begin{equation}\label{eq...text...tupleAlpha}
\valpha=\{\malpha^{[1]},\valpha^{[1]}_{\rb},\cdots, \malpha^{[L]},\valpha^{[L]}_{\rb}\}\in\mathrm{Tuple}_{\{m'_0,\cdots,m'_L\}}
\end{equation}  in Definition \ref{def...GeneralEmbedding}.
\begin{statement}[\textbf{Compatibility conditions I} (see Fig. \ref{fig:alpha_constraints} for illustration)]\label{state1}
The elements $\{\malpha^{[l]}\}_{l=1}^L$ in  \eqref{eq...text...tupleAlpha}   satisfy that: \\
There exist a collection  of  auxiliary  variables $\vbeta:=\left\{\vbeta^{[l]}_j\in\sR|~l\in[0:L],~j\in[m'_l]\backslash\fI_l^{-1}(0)\right\}$  such that
\begin{itemize}
	\item $\vbeta^{[L]}_k=1$ for $k\in[m'_L]$.~(Since $\fI_L$ is the identity map, $\fI_L^{-1}(0)=\emptyset$). 
	\item  \textbf{Forward conditions}:
	\begin{itemize}
		\item {\textbf{Effective neurons forward to an effective neurons}}: For $i\notin \fI^{-1}_l(0)$, $s\in[m_{l-1}]$, we have $\sum_{j\in \fI^{-1}_{l-1}(s)}\malpha^{[l]}_{ij}=1$; 
		\item {\textbf{Effective neurons forward to a null neuron}}: For $i\in \fI^{-1}_l(0)$, $s\in[m_{l-1}]$, we have $\sum_{j\in \fI^{-1}_{l-1}(s)}\malpha^{[l]}_{ij}=0$.
	\end{itemize}
	\item \textbf{Backward conditions}:
	\begin{itemize}
		\item {\textbf{Effective neurons backpropagate to an effective neuron}}: For $j\notin \fI^{-1}_{l-1}(0)$, $k\in[m_{l}]$, we have $\sum_{i\in \fI^{-1}_{l}(k)} \vbeta^{[l]}_i\malpha^{[l]}_{ij}=\vbeta^{[l-1]}_j$;
		\item  {\textbf{Effective neurons backpropagate to a null neuron}}:  For $j\in \fI^{-1}_{l-1}(0)$, $k\in[m_{l}]$, we have $\sum_{i\in \fI^{-1}_{l}(k)} \vbeta^{[l]}_i\malpha^{[l]}_{ij}=0$.
		%   \item {\textbf{Null neurons backpropagate to other neurons}}: For $k=0$, there is no constraint.
	\end{itemize}
\end{itemize}  
\end{statement}
\begin{figure}[h]
\centering
\includegraphics[width=\textwidth]{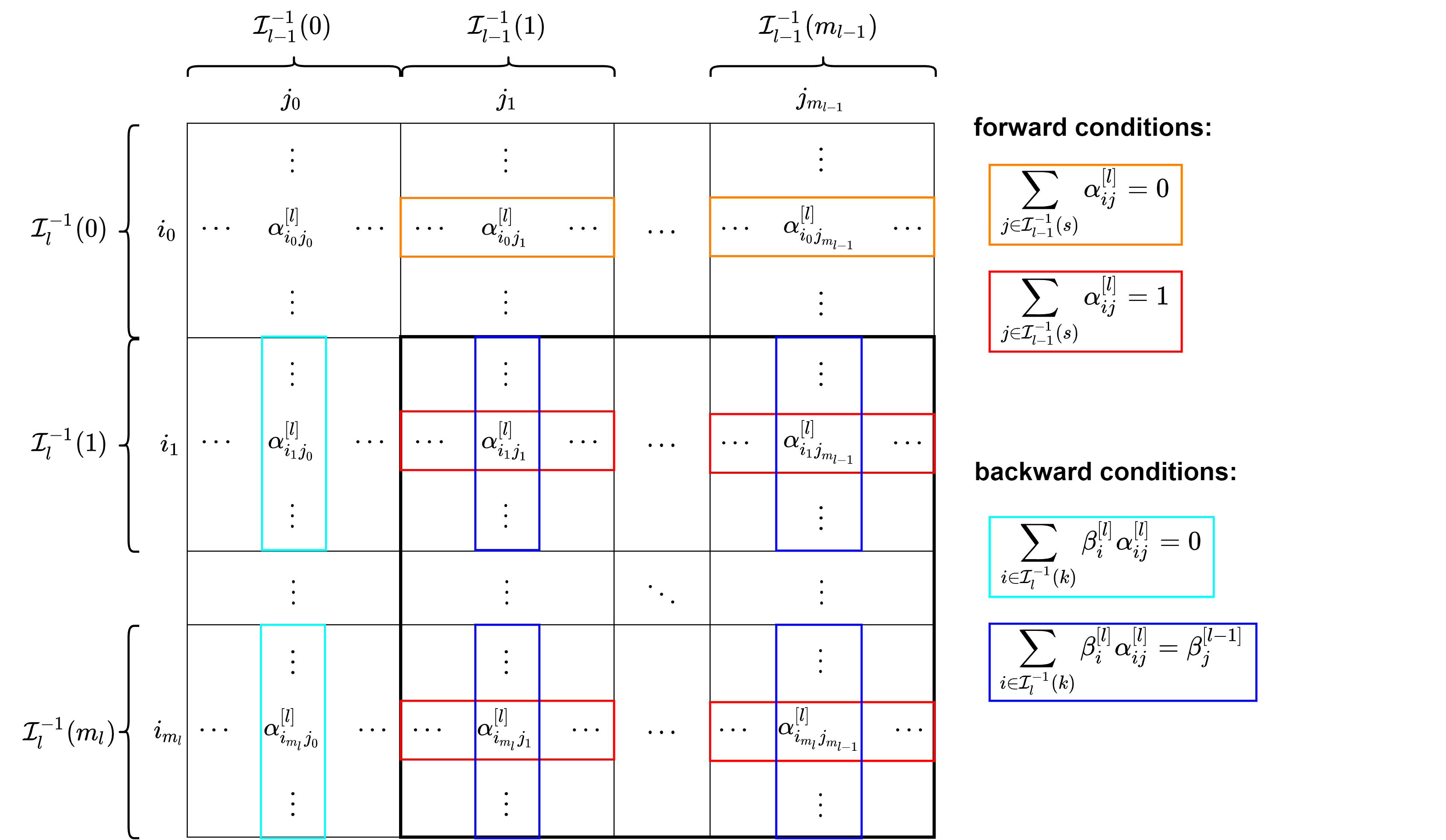}
\caption{Illustration of forward and backward compatibility conditions I for $\malpha^{[l]}$. \label{fig:alpha_constraints}}
\end{figure} 
%%%%%%%%%%%%%%%%%%%%%%%%%%%%%%%%%
The compatibility conditions for $\{\valpha^{[l]}_{\rb}\}_{l=1}^L$ is stated in the following. In order for that, we need another collection of  
% effective bias 
auxiliary variables  $\mB_*:=\left\{(\vb^{[l]}_*)_i\in\sR|~l\in[0:L],~i\in\fI_l^{-1}(0)\right\}$. We term $\mB_*$ {\emph{the effective biases for null neurons.}}
\begin{statement}[\textbf{Compatiblity conditions II} (see Fig. \ref{fig:alpha_b} for illustration)]\label{state2}
The rest of the  elements in $\valpha$, i.e., $\{\valpha^{[l]}_{\rb}\}_{l=1}^L$ satisfy that:\\
For any  effective biases of null neurons $\mB_*:=\left\{(\vb^{[l]}_*)_i\in\sR|~l\in[0:L],~i\in\fI_l^{-1}(0)\right\}$, we have
\begin{itemize}
	\item  \textbf{Effective neurons:} $(\valpha_{\rb}^{[l]})_i=0$ for any $l\in[L]$ with $i\notin\fI_l^{-1}(0)$;
	\item \textbf{Null neurons to a null neuron:} for $i\in \fI^{-1}_l(0)$, $s=0$, we have $\sum_{j\in \fI^{-1}_{l-1}(0)}\malpha^{[l]}_{ij}\sigma((\vb^{[l-1]}_*)_j)+(\valpha_{\rb}^{[l]})_i=(\vb^{[l]}_*)_i$; 
	\item \textbf{Null neurons to an effective neuron:} for $i\notin \fI^{-1}_l(0)$, $s=0$, we have $\sum_{j\in \fI^{-1}_{l-1}(0)}\malpha^{[l]}_{ij}\sigma((\vb^{[l-1]}_*)_j)+(\valpha_{\rb}^{[l]})_i=0$; \\
\end{itemize}
\end{statement}
\begin{figure}[h]
\centering
\includegraphics[width=0.6\textwidth]{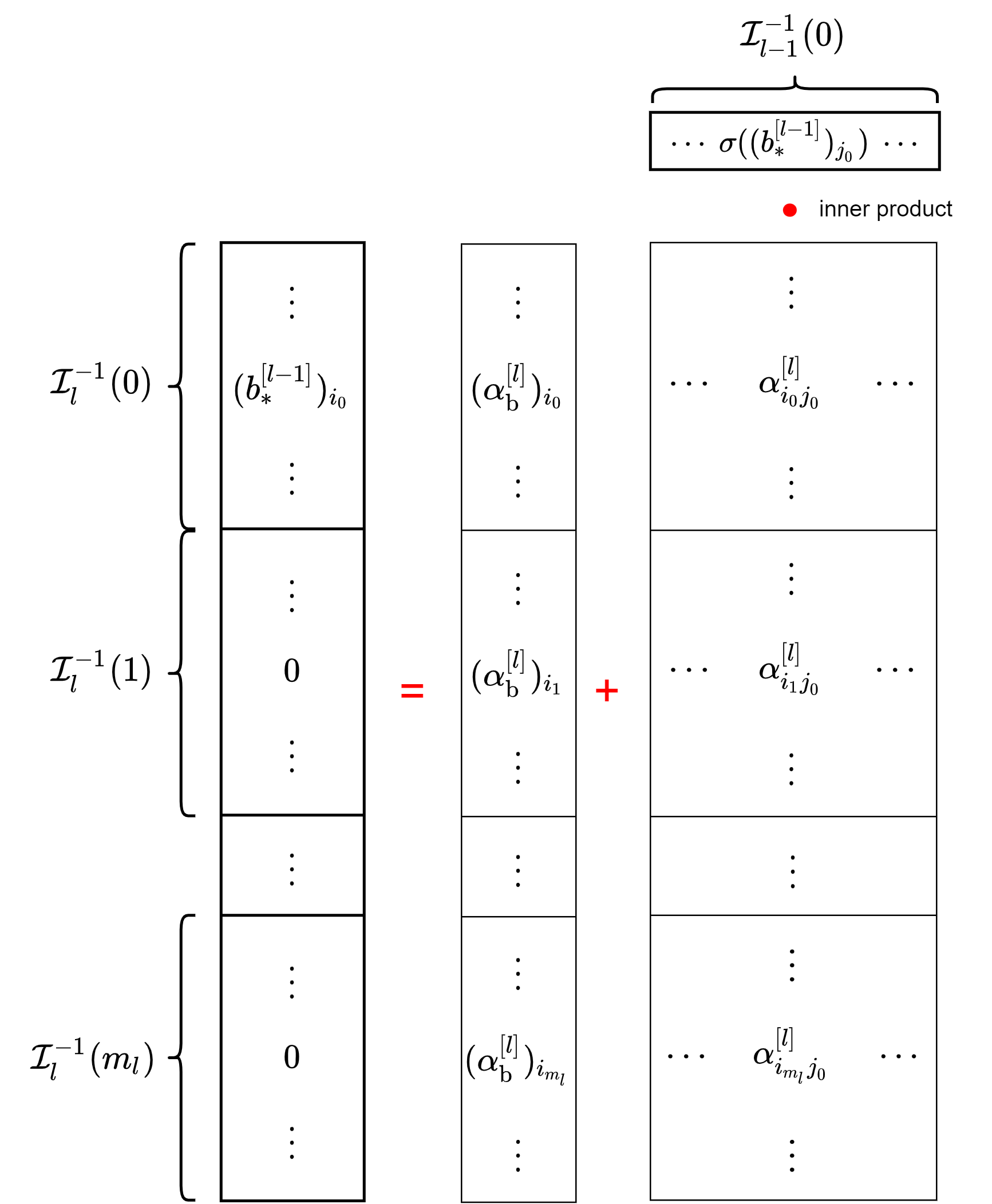}
\caption{Illustration of compatibility conditions II for $\valpha^{[l]}_{\rb}$. \label{fig:alpha_b}}
\end{figure} 
	We then illustrate the general compatible embedding as follows. First, we illustrate an example of the general compatible embedding without null neurons through the forward conditions in Fig. \ref{fig:forwardgenesplitting} and the backward conditions in Fig. \ref{fig:backwardgenesplitting}. Second, we illustrate the conditions related to null neurons in Fig. \ref{fig:nullgeneral}.
	
	\begin{figure}[h]
		\centering
		\includegraphics[width=\textwidth]{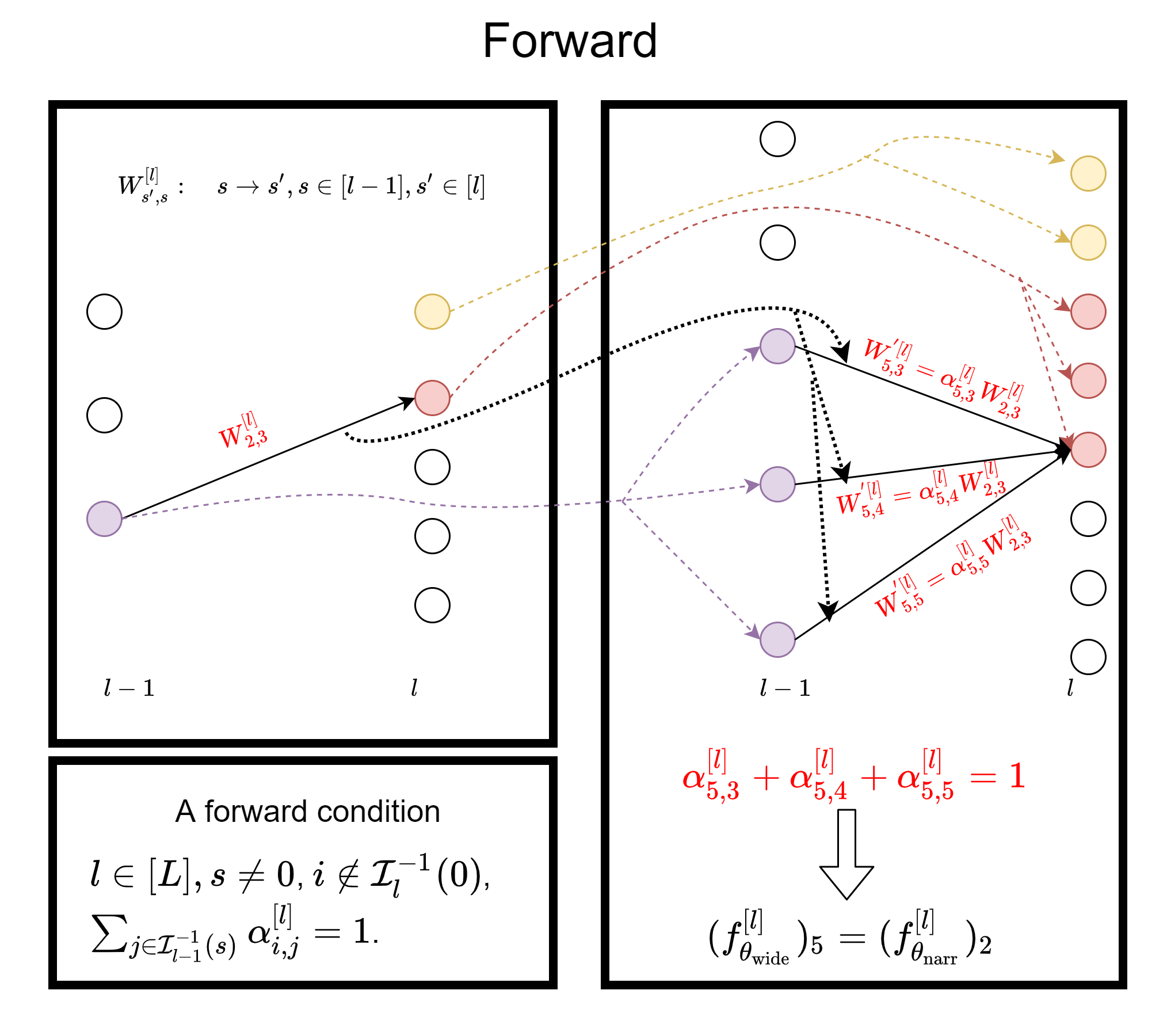}
		\caption{Illustration of the forward conditions for split neurons in compatibility conditions I. The sum of the input weights from neurons that are split from the same previous neuron to a post neuron should be equal to the weight from the previous to the post neuron. \label{fig:forwardgenesplitting}}
	\end{figure}

	\begin{figure}[h]
		\centering
		\includegraphics[width=\textwidth]{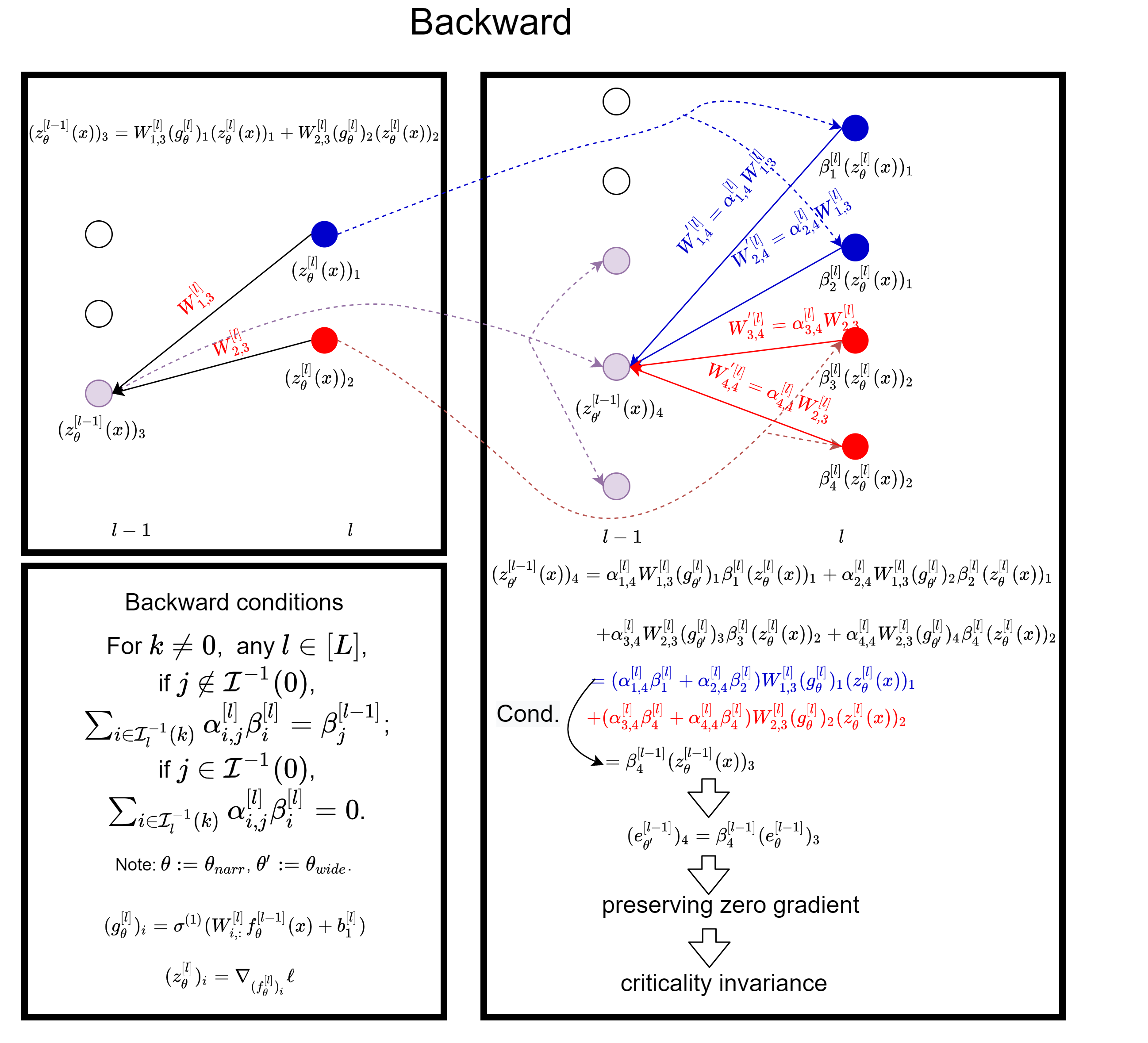}
		\caption{Illustration of the backward conditions for split neurons in compatibility conditions I. The gradient of the loss w.r.t. to the neuron output of the wide network is proportional to the gradient of the loss w.r.t. the corresponding neuron in the narrow network. \label{fig:backwardgenesplitting}}
	\end{figure} 
	
	\begin{figure}[h]
		\centering
		\includegraphics[width=\textwidth]{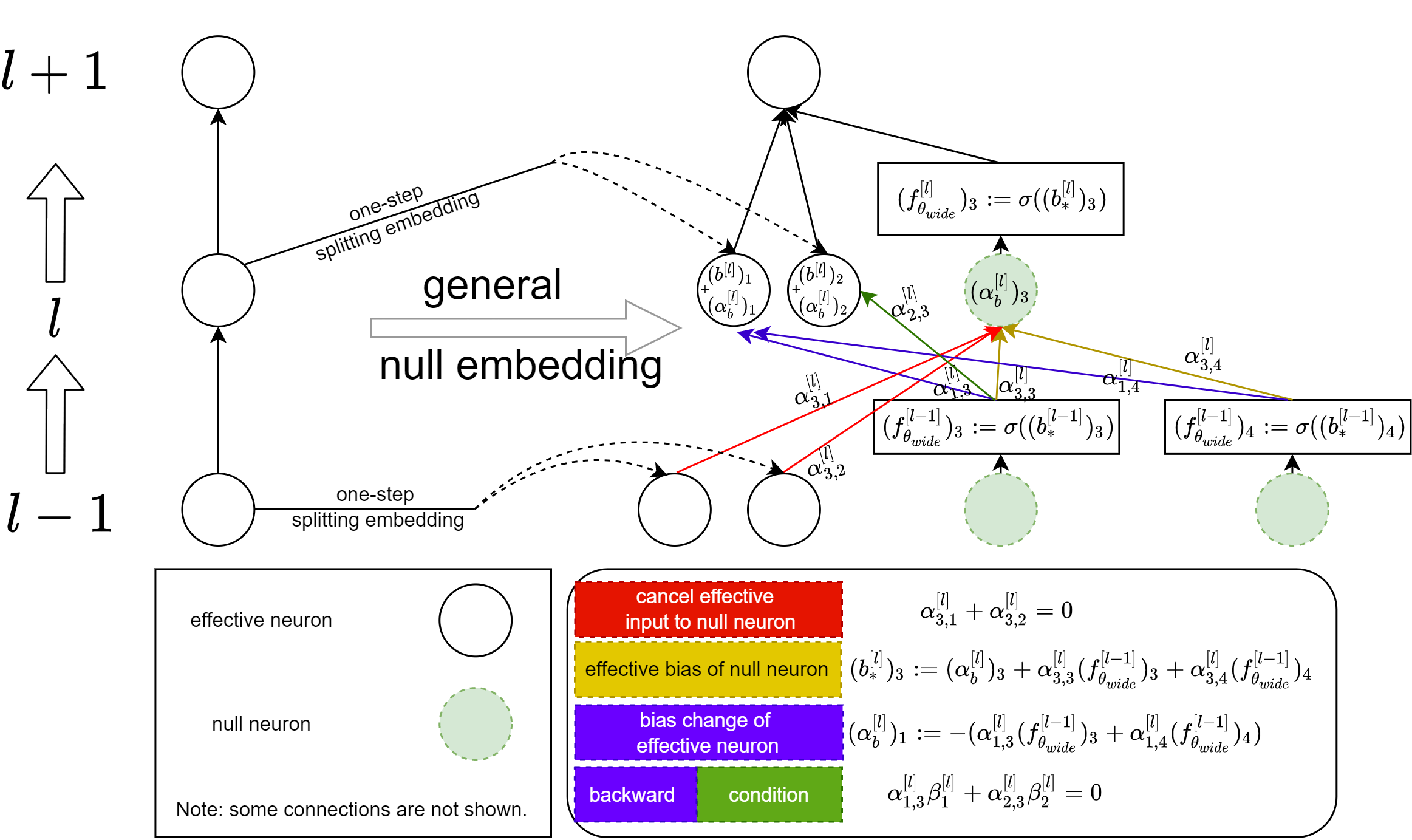}
		\caption{Illustration of compatibility conditions for null neurons. Red: The input from effective neurons split from the same neuron should cancel out. Yellow: the effective bias of a null neuron is the sum of its bias (can be any value) and its input from other null neurons. Purple: The added null neurons have constant input to the effective neuron, which would be added a bias correction to cancel the constant input. Purple+Green: the backward condition of a null neuron leads to the zero gradient. \label{fig:nullgeneral}}
	\end{figure} 
	
	\begin{theorem}\label{thm:gce}
		General compatible embedding is  a critical embedding.
	\end{theorem}
	Remark that we later name it as general compatible critical embedding in this work.
	%%%%%%%%%%%%%%%%%%%%%%%%%%%%%%%%%%%%%%%%%%%%%%%%%%%%%%%%%%

	\subsection{Special cases of general compatible critical embedding}
	In the following, we present some special cases of general compatible critical embedding. Specifically, the three-fold global splitting embedding is the key to the proof of a necessary condition of a ``truly-bad'' critical point in Section \ref{section...PropertyofCP}.
	
	\begin{definition}[Splitting embedding\label{def_splitting_embedding}]
		For a general compatible critical embedding $\fT_{\fI}^{\valpha}$, if there is no null neuron in the embedded NN, i.e., $\fI^{-1}_l(0)=\emptyset$ for any $l\in[L-1]$, then we call $\fT_{\fI}^{\valpha}$ a splitting embedding. 
	\end{definition}
	
	\begin{definition}[Null embedding]
		For a general compatible critical embedding $\fT_{\fI}^{\valpha}$, if only null neurons are added, i.e., $\#\fI^{-1}_l(0)=m'_l-m_l$ for any $l\in[L-1]$, hence there is no neuron splitting, i.e., $\#\fI^{-1}_l(s)=1$ for any $l\in[L-1]$, $s\in [m_l]$,  then we call $\fT_{\fI}^{\valpha}$ a null embedding. 
	\end{definition}
	
	\begin{rmk}
		One-step null embedding and its multi-step composition are special cases of null embedding. Similarly, one-step splitting embedding and its multi-step composition are special cases of splitting embedding. Multi-step embedding composed by a mixture of one-step null and splitting embedding is a special case of general compatible critical embedding.
	\end{rmk}
	
	\begin{example}[A three-fold global splitting embedding]\label{ex}
		For any $l\in[L-1]$, we define the operator $\fT_{\mathrm{global}}$ applying on $\vtheta$ as follows
		\begin{align*}
			\fT_{\mathrm{global}}(\vtheta)|_l
			&= \left({\begin{bmatrix}
					\mW^{[l]} & \mzero & \mzero\\
					\mzero & \mW^{[l]} & \mzero\\
					\mzero & \mzero    & \mW^{[l]}\\
			\end{bmatrix} },
			\left[ {\begin{array}{cc}
					\vb^{[l]} \\
					\vb^{[l]} \\
					\vb^{[l]} \\
			\end{array} } \right]\right), \\
			% \fV_{G}(\vtheta)|_l
			%     &=\left(\mzero_{2m_{l}\times 2m_{l-1}},\mzero_{2m_{l}\times 1}\right),\\
			\fT_{\mathrm{global}}(\vtheta)|_L
			&=\left(\left[\mW^{[L]},\mW^{[L]},-\mW^{[L]}\right],\vb^{[L]}\right),\\
			% \fV_{G}(\vtheta)|_L
			%     &=\left(\left[-\mW^{[L]},\mW^{[L]}\right],\mzero_{m_{L}\times 1}\right),\\
		\end{align*}
	\end{example}
	
	\begin{remark}
		For this global splitting embedding, $m'_l = 3m_l$ and  $\fI^{-1}_{l}(s)=\{s,s+m_l,s+2m_l\}$ for any $l\in[L-1]$ and $s\in[m_l]$. For any $l\in[L-1]$,
		$$\valpha^{[l]}=
		\begin{bmatrix}
			\mathbf{1}_{m_l\times m_{l-1}} & \mzero_{m_l\times m_{l-1}} & \mzero_{m_l\times m_{l-1}}\\
			\mzero_{m_l\times m_{l-1}} & \mathbf{1}_{m_l\times m_{l-1}} & \mzero_{m_l\times m_{l-1}}\\
			\mzero_{m_l\times m_{l-1}} & \mzero_{m_l\times m_{l-1}}    & \mathbf{1}_{m_l\times m_{l-1}}\\
		\end{bmatrix},$$
		$$\valpha_{\rb}^{[l]}=\vzero_{3m_l\times 1},$$
		and
		$$\valpha^{[L]}=
		\begin{bmatrix}
			\mathbf{1}_{m_l\times m_{l-1}} & \mathbf{1}_{m_l\times m_{l-1}} & -\mathbf{1}_{m_l\times m_{l-1}}\\
		\end{bmatrix},$$
		$$\valpha_{\rb}^{[l]}=\vzero_{m_L\times 1}.$$
		Moreover, for any $l\in[L-1]$, $\vbeta^{[l]}_j=1$ when $j\in[2m_l]$ and $\vbeta^{[l]}_j=-1$ when $j\in[2m_l+1:3m_l]$. It is a special case of splitting embedding.
	\end{remark}

	\section{Analysis of critical points/submanifolds by critical embeddings}\label{section...PropertyofCP}
	
	\subsection{Degeneracy of critical points/submanifolds}
	A key observation to the critical embeddings proposed above is that it is not unique. Actually, given any NN and a wider NN, there is a class of critical embeddings. Therefore, any critical point can be embedded to a set of critical points in a wider NN. By a $K$-step composition embedding, because embedding $\fT_{l_i,s_i}^{\alpha_i}$ at each step $i$ has one degree of freedom parameterized by $\alpha_i$ given $l_i$ and $s_i$, one critical point in general can be embedded to a $K$-dimensional critical affine subspace, which provides a lower bound to the degeneracy of embedded critical points. Precisely, in  \cite{zhang2021embedding}, we prove the following theorem using composition of one-step splitting embeddings.
	
	\begin{theorem} [Theorem 2 in \cite{zhang2021embedding}\label{theorem2}]
		Given an  $\mathrm{NN}(\{m_l\}_{l=0}^{L})$ and a $K$-neuron wider $\mathrm{NN}(\{m'_l\}_{l=0}^{L})$, then for any  critical point $\vtheta^{\rc}_{\rnarr}=\left(\mW^{[1]},\vb^{[1]},\cdots,\mW^{[L]},\vb^{[L]}\right)$ satisfying $\mW^{[l]}\neq \mzero$ for each  $l\in[L]$, $\vtheta^{\rc}_{\rnarr}$  can be critically embedded to a $K$-dimensional critical affine subspace $\fM_{\rwide}=\{\vtheta_{\rwide}+\sum_{i=1}^{K}\alpha_i \vmu_i|\alpha_i\in \sR \}$ of  loss landscape of $\mathrm{NN}(\{m'_l\}_{l=0}^{L})$. Here $\vtheta_{\rwide}:=(\prod_{l=1}^{K} \fT_{l_l,s_l})(\vtheta_{\rnarr})$ and $\vmu_i:=\fT_{l_K,s_K}\cdots \fV_{l_i,s_i} \cdots\fT_{l_1,s_1}(\vtheta_{\rnarr})$, where $s_l\neq 0$ for all $l\in[K]$.~(The definitions of  $\fV_{l_l,s_l}$ and  $\fT_{l_l,s_l}$ can be found in Definition \ref{def4}).
	\end{theorem}
	
	In addition, in this work, as we propose above a wider class of general compatible critical embeddings, we attempt to obtain a better estimate of the dimension of critical submanifolds corresponding to any given critical function $\vf_{\vtheta^{\rc}_{\rnarr}}$ of a narrow NN. A gross estimate yields the following heuristic argument.

	\begin{sketch}
		Given any total index mapping $\fI$, the degree of freedom of all possible $\valpha$ for general compatible critical embedding $\fT^{\valpha}_\fI$ to a $K$-neuron wider NN is $K+\sum_{l\in[L]} K_l K_{l-1}$, where $K_l:=m'_l-m_l$, and $K:=\sum_{l\in[L-1]}K_l$.
	\end{sketch}
	
	\begin{just}
		% We justify this argument as follows.
		By the definition of general compatible critical embedding, $\{\malpha^{[l]}\}_{l=1}^L$ and auxiliary  variables $\vbeta=\left\{\vbeta^{[l]}_j\in\sR|~l\in[L],j\in[m'_l]\backslash\fI_l^{-1}(0)\right\}$ with $\vbeta^{[L]}_k=\vone$ satisfy forward and backward compatibility conditions. 
		
		Step 1: We first observe that $\sum_{i\in \fI^{-1}_l(s)}\vbeta_i^{[l]}=1$ for any $s\in[m_l]$. Thus, $\vbeta$ has $K-m_{\mathrm{null}}$ degrees of freedom, where $m_{\mathrm{null}}$ is the number of null neurons.
		
		Step 2: Given any such $\vbeta$, we now consider the degrees of freedom for $\{\malpha^{[l]}\}_{l=1}^L$. 
		Note that forward and backward conditions are not independent of one another because  $$\sum_{j\in \fI^{-1}_{l-1}(s)}\sum_{i\in \fI^{-1}_{l}(k)} \vbeta^{[l]}_i\malpha^{[l]}_{ij}=\sum_{i\in \fI^{-1}_{l}(k)} \vbeta^{[l]}_i=1=\sum_{j\in \fI^{-1}_{l-1}(s)}\vbeta^{[l-1]}_j$$
		automatically holds for the given $\vbeta$. Therefore, there are $m'_l m_{l-1}+m_l m'_{l-1}-m_l m_{l-1}$ independent linear equations for $m'_l m'_{l-1}$ parameters in $\valpha^{[l]}$ for each layer $l$, resulting in $\sum_{l\in[L]}K_l K_{l-1}$ degrees of freedom in total. 
		
		Step 3: Given any effective biases $\mB_*$ as auxiliary variables for all null neurons and any $\malpha^{[l]}$ satisfying forward and backward conditions, $(\valpha_{\rb}^{[l]})_i$ is uniquely determined for all neurons in the wide NN. Therefore, there are $m_{\mathrm{null}}$ degrees of freedom in $\mB_*$.
		
		In the end, the degrees of freedom in $\valpha$ is the summation of degrees of freedom in auxiliary variables $\vbeta$, degrees of freedom in $\{\malpha^{[l]}\}_{l=1}^L$ given $\vbeta$ and additional degrees of freedom in $\mB_*$ for all null neurons, which add up to $K+\sum_{l\in[L]} K_l K_{l-1}$ degrees of freedom.
	\end{just}
	
	Intuitively, degrees of freedom in critical embedding $\fT$ transform into dimensions of critical submanifolds, resulting in a higher estimate of degrees of degeneracy $K+\sum_{l\in[L]} K_l K_{l-1}$ in comparision to $K$ obtained by multi-step composition embedding in \cite{zhang2021embedding}. We note that the nonlinear coupling between $\malpha$ and $\vbeta$ in general results in non-affine curved critical submanifolds for three-layer or deeper NNs containing the corresponding $K$-dimensional critical affine subspaces identified in Theorem \ref{theorem2}. A more rigorous estimate requires careful handling of the nonlinear coupling. We leave it to later works.

	\subsection{Irreversible transition to strict-saddle point}
	In this subsection, we look further into the transition between different types of critical points, e.g., local minima, saddle points, through critical embeddings. We are specifically interested in the transition of a critical point to a strict-saddle point due to its good optimization guarantee detailed in \cite{lee2019first}. We prove the following irreversibility property for any critical embedding, which guarantees that a strict-saddle point always embedds to strict-saddle points. 
	
	\begin{theorem}\label{t2}
		Given an $\mathrm{NN}(\{m_l\}_{l=0}^{L})$ and 
		any of its parameters $\vtheta\in \sR^M$, for any critical embedding $\fT:\sR^M\to \sR^{M'}$ to any wider $\mathrm{NN}(\{m'_l\}_{l=0}^{L})$, the number of positive, zero, negative eigenvalues of $\mH_S(\fT(\vtheta))$ is no less than the counterparts of $\mH_S(\vtheta)$.   
	\end{theorem}
	
	\subsection{``Truly-bad'' critical point}
	The above irreversibility property of any critical embedding provokes the following thought about a conventional bad local minimum: for any bad local minimum of a given NN, if it can become a strict-saddle point in wider NNs, then it should not be a problem as we can simply use a wider NN in practice. However, there is still a ``truly-bad'' situation in which a critical point may never become a strict-saddle point through any critical embedding. In the following, through proving a stringent necessary condition for such a ``truly-bad'' critical point defined below, we justify its rarity, hence providing a reasonable conjecture for the easy optimization of wide NNs widely observed in practice. 
	
	We denote hereafter that $\mA\succeq \mB$ if and only if $\mA-\mB$ is a semi-positive definite matrix, and $\mA\succ \mB$ if and only if $\mA-\mB$ is a strictly positive definite matrix. 
	\begin{definition}[\textbf{"Truly-bad" critical point}]
		Given any data $S$, loss $\ell(\cdot,\cdot)$ and activation $\sigma(\cdot)$, for any NN,   if there exists a  critical point $\vtheta^{\rc}\in\vTheta^{\rc}$  satisfying that:\\ (i)~$\vtheta^{\rc}$ is not a strict-saddle point~\citep[Definition 1]{lee2019first};\\
		(ii)~For any critical embedding $\fT$,  $\fT(\vtheta^{\rc})$ is also not a strict-saddle point,\\ then we term this critical point a {\textbf{"truly-bad" critical point}}.
	\end{definition}
	We would like to introduce some additional notations in order to state Lemma \ref{lem4} and Lemma \ref{lem5}. 
We denote  
	\[
	\mH^{(1),[L-1]}_S(\vtheta):=\sum_{i,j=1}^{m_L}\Exp_{S}\partial_{ij}\ell (\vf_{\vtheta},\vf^*)\nabla_{\vtheta^{[L-1]}}(\vf_{\vtheta})_i\left(\nabla_{\vtheta^{[L-1]}}(\vf_{\vtheta})_j\right)^\T,
	\]
	and
	\[
	\mH^{(2),[L-1]}_S(\vtheta):=\sum_{i,j=1}^{m_L}\Exp_{S}\partial_i\ell (\vf_{\vtheta},\vf^*)\mW_{i,j}^{[L]}\nabla_{\vtheta^{[L-1]}}\nabla_{\vtheta^{[L-1]}}\left(\vf_{\vtheta}^{[L-1]}\right)_j.
	\]
	We denote further that 
	$\mH^{[L-1]}_S(\vtheta):=\mH^{(1),[L-1]}_S(\vtheta)+\mH^{(2),[L-1]}_S(\vtheta)$.

	We state Lemma \ref{lem4} and Lemma \ref{lem5} as follows.
	%%%%%%%%%%%%%%%%%%%%%%%%%%%%%%%%%%%%%%%%%%%%%
	\begin{lemma}\label{lem4}
		Given any data $S$, loss $\ell(\cdot,\cdot)$ and activation $\sigma(\cdot)$, for any NN, if a critical point $\vtheta^{\rc}\in\vTheta^{\rc}$ satisfies:\\ (i)~$\mH_S(\vtheta^{\rc})\succeq 0$;\\  
		(ii)~$\mH^{(2),[L-1]}_S(\vtheta^{\rc})\neq \mathbf{0}$,\\
		then there exists a general compatible critical embedding $\fT$, such that $\fT(\vtheta^{\rc})$ is a strict-saddle point.
	\end{lemma}

	%
	
	%  \\
	%     \nabla_{\vtheta_1}\nabla_{\vtheta_1}\vf_{\vtheta_{\rwide}} &= \nabla_{\vtheta_2}\nabla_{\vtheta_2}\vf_{\vtheta_{\rwide}}
	%     =\mW_{\rnarr}^{[L]}\nabla_{\vtheta_1}\nabla_{\vtheta_1}\vf^{[L-1]}, \\
	%     \nabla_{\vtheta''^{[L-1]}}^2\vf_{\vTheta}(\vx) &= 
	%     -\mW^{[L]}\nabla\nabla_{\vtheta^{[L-1]}}\vf^{[L-1]},\\
	%     \nabla_{\vtheta^{[L-1]}}\nabla_{\vtheta'^{[L-1]}}\vf_{\vTheta}(\vx)&=
	%     \nabla_{\vtheta^{[L-1]}}\nabla_{\vtheta''^{[L-1]}}\vf_{\vTheta}(\vx) = \nabla_{\vtheta'^{[L-1]}}\nabla_{\vtheta''^{[L-1]}}\vf_{\vTheta}(\vx) = \mathbf{0},

	\begin{lemma}\label{lem5}
		Given any data $S$, loss $\ell(\cdot,\cdot)$ and activation $\sigma(\cdot)$, for any NN, if a critical point $\vtheta^{\rc}\in\vTheta^{\rc}$ satisfies:\\
		(i)~$\mH_S(\vtheta^{\rc})\succeq 0$;\\ 
		(ii)~$\mH^{(2),[L-1]}_S(\vtheta^{\rc})= \mathbf{0}$;\\
		(iii)~ $\mH^{(2)}_S(\vtheta^{\rc})\neq \mathbf{0}$.\\
		Then there exist a general compatible critical embedding $\fT$, such that $\fT(\vtheta)$ is a strict-saddle point.
	\end{lemma}
	% For convex loss function $l(\cdot,\cdot)$ such that $l''(\cdot,\cdot)\geq 0$, $\mH^{(1)}_S=\Exp_{\vx\sim S}l''(f(\vx,\vtheta),f^*(\vx))\nabla_{\vtheta}f_{\vtheta}(\vx)\nabla_{\vtheta}f_{\vtheta}(\vx)^{\T}$ is always positive semidefinite. We prove that if $\mH^{(2)}_S= \Exp_{\vx\sim S}l'(f(\vx,\vtheta),f^*(\vx))\nabla_{\vtheta}^2f_{\vtheta}(\vx)$ has a negative eigenvalue, then, there exists a critical embedding such that $\mH_S(\fT\vtheta_0)$ has a negative eigenvalue.
	
	\begin{theorem}\label{thm:trulybad}
		Given any data $S$, loss $\ell(\cdot,\cdot)$ and activation $\sigma(\cdot)$, for any NN, if a critical point $\vtheta^{\rc}$ with $\mH_S(\vtheta^{\rc})\succeq 0$ satisfies $\mH^{(2)}_S(\vtheta^{\rc}) \neq \mathbf{0}$, then there exist a general compatible critical embedding $\fT$ such that $\fT(\vtheta^{\rc})$ is a strict-saddle point, i.e., $\mH_S(\fT(\vtheta^{\rc}))$ has at least one negative eigenvalue.
	\end{theorem}
	
	\begin{proof}
		This can be directly obtained by Lemma \ref{lem4} and \ref{lem5}.
	\end{proof}
	
	\begin{theorem*}[short version of Theorem \ref{thm:trulybad}]
		$\mH^{(2)}_S(\vtheta^{\rc}) = \mzero$ is a necessary condition for a critical point $\vtheta^{\rc}$ being a "truly-bad" critical point.
	\end{theorem*}
	%%%%%%%%%%%%%%%%%%%%%%%%%%%%%%%%%%%%%%%%%%%%%%%%%%%%%%%%%%%%% 
% 	\bibliographystyle{elsarticle-num-names}

\section{Conclusion and discussion}
In this work, we prove the Embedding Principle that loss landscape of an NN \emph{contains} all critical points/functions of all the narrower NNs. We define the critical embedding, which serves as the key tool not only to the proof of Embedding Principle but also to the study of the general geometry of loss landscape. Importantly, we discover a wide class of general compatible embedding, by which we obtain rich understanding about the critical points/submanifolds, e.g., their degree of degeneracy, their easy and irreversible transition to strict-saddle points.

% This embedding principle unravels the
% wide existence of highly degenerate critical points with low complexity in the loss landscape of a wide DNN, i.e., critical points with low-complexity output function and degenerate Hessian matrix, embedded from critical points of narrow DNNs. 
% With such a loss landscape of DNN, the gradient-based training has the potential of getting attracted or even converging to a low complexity critical point as confirmed by numerical experiments, which implies a potential implicit regularization towards low-complexity function of nonlinear DNN training dynamics.

The general compatible embedding proposed in this work unravels that the embedding relation from a critical point of the loss landscape of a narrow NN to critical points of the loss landscape of any wide NN is not limited by the composition of adding neurons one by one, but can be a collective operation. As a consequence, all critical points of a wide NN embedded from a critical point of a narrower NN by all possible general compatible critical embeddings, form high-dimensional critical submanifolds which in general are not affine subspaces for three-layer or deeper NNs. It is interesting and important. However, it remains a problem for the future study about whether the general compatible embeddings in certain sense are all the critical embeddings.

Embedding Principle provides an integrated view about loss landscapes of NNs with different widths. Specifically, it informs us that the existence of bad local minimum in a NN with a fixed width may not be a big deal for optimization as long as it can become a strict-saddle point in wider NNs, i.e., not a ``truly-bad'' critical point. In this work, we prove a stringent necessary condition for a ``truly-bad'' critical point. Still, it remains a problem about whether non-trivial ``truly-bad'' critical points indeed exist, and whether they are indeed a headache for optimization. 

We emphasize that Embedding Principle provides a function space view on the critical points~/submanifolds of loss landscape, which is of great importance for studying the implicit regularization and generalization of NNs. In recent years, it has been more and more clear that optimization and generalization are heavily intertwined with one another for deep learning. Therefore, one can not expect to develop a deep learning theory with optimization theory and generalization theory established separately. Yet, for a long time, the study of loss landscape focuses mainly on the parameter space in pursuit of an optimization guarantee. On the other hand, the study of generalization focuses mainly on the function space, failing to incorporate the geometry of loss landscape in parameter space which is key to the training. Now, by the Embedding Principle, we see a hierarchical structure of critical functions of different complexities of the loss landscape, which originate from loss landscape with clear optimization implication while being amenable to the analysis of complexity with clear generalization implication. Such a critical function hierarchy reflects the degree of matching between the NN architecture and data, i.e., if critical functions of narrow NNs attain low training error, then the NN architecture well matches the target function and may obtain a well generalized solution through training like in Fig. \ref{fig:EPrinciple_example}. Though, more works need to be done to unravel the full details and implication of this critical function hierarchy, we believe it is an important piece and may be a key to the deep learning theory.

\section*{Acknowledgments}
This work is sponsored by the National Key R\&D Program of China  Grant No. 2019YFA0709503 (Z. X.), the Shanghai Sailing Program, the Natural Science Foundation of Shanghai Grant No. 20ZR1429000  (Z. X.), the National Natural Science Foundation of China Grant No. 62002221 (Z. X.), the National Natural Science Foundation of China Grant No. 12101401 (T. L.), the National Natural Science Foundation of China Grant No. 12101402 (Y. Z.), Shanghai Municipal of Science and Technology Project Grant No. 20JC1419500 (Y.Z.), Shanghai Municipal of Science and Technology Major Project No. 2021SHZDZX0102, and the HPC of School of Mathematical Sciences and the Student Innovation Center at Shanghai Jiao Tong University.

\appendix
\section{Theoretical details}
\begin{lemma*} [Lemma \ref{lem...appen...section1...null} in main text.]
For any one-step null embedding $\fT_{l,0}^{\alpha}$, given any  $\mathrm{NN}(\{m_l\}_{l=0}^{L})$ and its parameters
$\vtheta_{\rnarr}\in\mathrm{Tuple}_{\{m_0,\cdots,m_L\}}$ with $\mathrm{Tuple}_{\{m_0,\cdots,m_L\}}\in\fD_{l,0}$, we have $\vtheta_{\rwide}:=\fT_{l,0}^{\alpha}(\vtheta_{\rnarr})$ satisfies the following conditions: given any data $S$, loss $\ell(\cdot,\cdot)$ and activation $\sigma(\cdot)$, for any $l\in[L-1]$,\\
(i) feature vectors in $\vF_{\vtheta_{\rwide}}$: $\vf^{[l']}_{\vtheta_{\rwide}}=\vf^{[l']}_{\vtheta_{\rnarr}}$, for $l'\in[L]$ and $l'\neq l$, $\vf^{[l]}_{\vtheta_{\rwide}}=\left[(\vf_{\vtheta_{\rnarr}}^{[l]})^\T,\sigma(\alpha)\right]^\T$;\\
(ii) feature gradients in $\vG_{\vtheta_{\rwide}}$: $\vg^{[l']}_{\vtheta_{\rwide}}=\vg^{[l']}_{\vtheta_{\rnarr}}$, for $l'\in[L]$ and $l'\neq l$, $\vg^{[l]}_{\vtheta_{\rwide}}=
\left[(\vg^{[l]}_{\vtheta_{\rnarr}})^\T,\sigma^{(1)}(\alpha)\right]^\T$;\\
(iii) error vectors in $\vZ_{\vtheta_{\rwide}}$: 
$\vz^{[l']}_{\vtheta_{\rwide}}=\vz^{[l']}_{\vtheta_{\rnarr}}$, for $l'\in[L]$ and $l'\neq l$, $\vz^{[l]}_{\vtheta_{\rwide}}=
\left[ \vz_{\vtheta_{\rnarr}}^{[l]},0\right]^\T$;\\
(iv) $\fT_{l,0}^{\alpha}$ is injective for all $\alpha$;\\
(v) $\fT_{l,0}^{\alpha}$ is an affine embedding for all $\alpha$.
\end{lemma*}
%%%%%%%%%%%%%%%%%%%%%%%%%%%%%%%%%%%%%%%%%%%%%%%%%%%%%%%%%%%%%%%%%%%%%%%%%%%%%%%%
\begin{proof}
(i) By the construction of $\vtheta_{\rwide}$, it is clear that $\vf^{[l']}_{\vtheta_{\rwide}}=\vf^{[l']}_{\vtheta_{\rnarr}}$ for all $l'\in [l-1]$. Then
\begin{align}
	\vf^{[l]}_{\vtheta_{\rwide}}
	&= \sigma  \left(\left[ {\begin{array}{cc}
			\mW^{[l]} \\
			\vzero_{1\times m_{l-1}} \\
	\end{array} } \right] \vf^{[l-1]}_{\vtheta_{\rnarr}}+
	\left[ {\begin{array}{cc}
			\vb^{[l]} \\
			\alpha \\
	\end{array} } \right]
	\right)
	= \left[ {\begin{array}{cc}
			\vf^{[l]}_{\vtheta_{\rnarr}} \\
			\sigma(\alpha) \\
	\end{array} } \right].
\end{align}
Note that since
\begin{equation*}
	\alpha\left[\mzero_{m_{l+1}\times {(m_l+1)}}\right]\left[ {\begin{array}{cc}
			\vf^{[l]}_{\vtheta_{\rnarr}} \\
			\sigma(\alpha) \\
	\end{array} } \right]=\mzero_{m_{l+1}\times 1}
\end{equation*}
Thus
\begin{align}
	\vf^{[l+1]}_{\vtheta_{\rwide}} 
	&=\sigma  \left(
	\left[\mW^{[{l+1}]},\mzero_{m_{l+1}\times 1}\right]
	\left[ {\begin{array}{cc}
			\vf^{[l]}_{\vtheta_{\rnarr}} \\
			\sigma(\alpha) \\
	\end{array} } \right]+\mzero_{m_{l+1}\times1}+
	\vb^{[l+1]}+\alpha~\mzero_{m_{l+1}\times1} 
	\right) = \vf_{\vtheta_{\rnarr}}^{[l+1]}.
\end{align}
Next, by the construction of $\vtheta_{\rwide}$, it is clear that $\vf^{[l']}_{\vtheta_{\rwide}}=\vf^{[l']}_{\vtheta_{\rnarr}}$ for any $l'\in[l+1:L]$. 

(ii) The results for feature gradients $\vg^{[l']}_{\vtheta_{\rwide}}=\vg^{[l']}_{\vtheta_{\rnarr}}$ for $l'\in[L]$ can be calculated in a similar way except by replacing $\sigma(\cdot)$ with  $\sigma^{(1)}(\cdot)$.

(iii)  By the backpropagation and the above facts in (i), we have \[\vz^{[L]}_{\vtheta_{\rwide}}=\nabla\ell(\vf^{[L]}_{\vtheta_{\rwide}},\vy)=\nabla\ell(\vf^{[L]}_{\vtheta_{\rnarr}},\vy)=\vz^{[L]}_{\vtheta_{\rnarr}}.\]

Recall the recurrence relation for $l'\in[l+1:L-1]$, then we recursively obtain the following equality for $l'$ from $L-1$ down to $l+1$:
\begin{equation}
	\vz^{[l']}_{\vtheta_{\rwide}}
	=(\mW^{[l'+1]})^\T \left(\vz^{[l'+1]}_{\vtheta_{\rwide}}\circ \vg^{[l'+1]}_{\vtheta_{\rwide}}\right)
	=(\mW^{[l'+1]})^\T \left(\vz^{[l'+1]}_{\vtheta_{\rnarr}}\circ \vg^{[l'+1]}_{\vtheta_{\rnarr}}\right)
	=\vz_{\vtheta_{\rnarr}}^{[l']}.
\end{equation}
Next,
\begin{align}
	\vz^{[l]}_{\vtheta_{\rwide}}
	&=\left(\left[\mW^{[{l+1}]},\mzero_{m_{l+1}\times 1}\right]
	+\alpha \left[\mzero_{m_{l+1}\times {(m_l+1)}}\right]\right)^\T
	\left(\vz^{[l+1]}_{\vtheta_{\rwide}}\circ \vg^{[l+1]}_{\vtheta_{\rwide}}\right) \nonumber\\
	&=\left[ {\begin{array}{cc}
			\vz^{[l]}_{\vtheta_{\rnarr}} \\
			0 \\
	\end{array} } \right]+
	\left[ \mzero_{ {(m_l+1)\times1}} \right]= \left[ \vz_{\vtheta_{\rnarr}}^{[l]},0\right]^\T.
\end{align}
Finally, 
\begin{align}
	\vz^{[l-1]}_{\vtheta_{\rwide}}
	&=\left[{\begin{array}{cc}
			{\mW^{[l]}}^\T,
			\vzero_{ m_{l-1} \times 1} \\
	\end{array} } \right] 
	\left(\left[ {\begin{array}{cc}
			\vz^{[l]}_{\vtheta_{\rnarr}} \\
			0 \\
	\end{array} } \right]
	\circ \left[ {\begin{array}{cc}
			\vg^{[l]}_{\vtheta_{\rnarr}} \\
			\sigma^{(1)}(\alpha)\\
	\end{array} } \right]\right)\nonumber\\
	&=(\mW^{[l]})^\T\left(\vz^{[l]}_{\vtheta_{\rnarr}}\circ\vg^{[l]}_{\vtheta_{\rnarr}}\right)+\mzero_{m_{l-1}\times 1}=\vz_{\vtheta_{\rnarr}}^{[l-1]}.
\end{align}
This with the recurrence relation once again leads to $\vz^{[l']}_{\vtheta_{\rwide}}=\vz_{\vtheta_{\rnarr}}^{[l']}$ for all $l'\in[l-1]$.

(iv) If for $\vtheta_1, \vtheta_2\in\mathrm{Tuple}_{\{m_0,\cdots,m_L\}}$ and $\vtheta_1\neq\vtheta_2$, since $\fT_{l,0}^{\alpha}|_k$ for $k\neq l, l+1$ is the identity map, then if there exists some $k_0\neq l, l+1$, such that $\mW_1^{[k_0]}\neq \mW_2^{[k_0]}$ or $\vb_1^{[k_0]}\neq \vb_2^{[k_0]}$,  then obviously $\fT_{l,0}^{\alpha}(\vtheta_1)\neq\fT_{l,0}^{\alpha}(\vtheta_2)$.   If   $k_0=l$ or $k_0=l+1$, by similar reasoning, $\fT_{l,0}^{\alpha}(\vtheta_1)\neq\fT_{l,0}^{\alpha}(\vtheta_2)$. 

(v) For $\vtheta_{ 0}=(\mW_0^{[1]},\vb_0^{[1]},\cdots,\mW_0^{[L]},\vb_0^{[L]})\in\mathrm{Tuple}_{\{m_0,\cdots,m_L\}}$, we have 
\[
\Tilde{\fT}_{l,0}^{\alpha}(\vtheta_0) =\left(\mW_0^{[1]},\vb_0^{[1]},\cdots,\left[ {\begin{array}{cc}
		\mW_0^{[l]} \\
		\vzero_{1\times m_{l-1}} \\
\end{array} } \right], \left[ {\begin{array}{cc}
		\vb_0^{[l]} \\
		0 \\
\end{array} } \right], \left[\mW_0^{[l+1]},\mzero_{m_{l+1}\times 1}\right],\vb_0^{[l+1]},\cdots,  \mW_0^{[L]},\vb_0^{[L]}\right),
\]
obviously $\Tilde{\fT}_{l,0}^{\alpha}$ is a linear operator, thus $\fT_{l,0}^{\alpha}$ is an affine operator.
\end{proof}
%%%%%%%%%%%%%%%%%%%%%%%%%%%%%%%%%%%%%%%%

\begin{lemma*}[Lemma \ref{lem...appen...section1...split} in main text]  
For any one-step splitting embedding $\fT_{l,s}^{\alpha}$, given any  $\mathrm{NN}(\{m_l\}_{l=0}^{L})$ and its parameters
$\vtheta_{\rnarr}\in\mathrm{Tuple}_{\{m_0,\cdots,m_L\}}$ with $\mathrm{Tuple}_{\{m_0,\cdots,m_L\}}\in\fD_{l,s}$, we have $\vtheta_{\rwide}:=\fT_{l,s}^{\alpha}(\vtheta_{\rnarr})$ satisfies the following conditions: given any data $S$, loss $\ell(\cdot,\cdot)$ and activation $\sigma(\cdot)$, for any $l\in[L-1]$,\\
(i) feature vectors in $\vF_{\vtheta_{\rwide}}$: $\vf^{[l']}_{\vtheta_{\rwide}}=\vf^{[l']}_{\vtheta_{\rnarr}}$, for $l'\in[L]$ and $l'\neq l$, $\vf^{[l]}_{\vtheta_{\rwide}}=\left[(\vf_{\vtheta_{\rnarr}}^{[l]})^\T,(\vf_{\vtheta_{\rnarr}}^{[l]})_s\right]^\T$;\\
(ii) feature gradients in $\vG_{\vtheta_{\rwide}}$: $\vg^{[l']}_{\vtheta_{\rwide}}=\vg^{[l']}_{\vtheta_{\rnarr}}$, for $l'\in[L]$ and $l'\neq l$, $\vg^{[l]}_{\vtheta_{\rwide}}=
\left[(\vg^{[l]}_{\vtheta_{\rnarr}})^\T,(\vg_{\vtheta_{\rnarr}}^{[l]})_s\right]^\T$;\\
(iii) error vectors in $\vZ_{\vtheta_{\rwide}}$: 
$\vz^{[l']}_{\vtheta_{\rwide}}=\vz^{[l']}_{\vtheta_{\rnarr}}$, \\
for $l'\in[L]$ and $l'\neq l$, $\vz^{[l]}_{\vtheta_{\rwide}}=
\left[ \left(\vz_{\vtheta_{\rnarr}}^{[l]}\right)^\T_{[1:s-1]},(1-\alpha)(\vz_{\vtheta_{\rnarr}}^{[l]})_s,\left(\vz_{\vtheta_{\rnarr}}^{[l]}\right)^\T_{[s+1:m_l]}, \alpha(\vz_{\vtheta_{\rnarr}}^{[l]})_s\right]^\T$;\\
(iv) $\fT_{l,s}^{\alpha}$ is injective for all $\alpha$.\\
(v) $\fT_{l,s}^{\alpha}$ is an affine embedding for all $\alpha$.
\end{lemma*}
%%%%%%%%%%%%%%%%%%%%%%%%%%%%%%%%%%%%%%%%%%%%%%%%%%%%%

\begin{proof}
(i) By the construction of $\vtheta_{\rwide}$, it is clear that $\vf^{[l']}_{\vtheta_{\rwide}}=\vf^{[l']}_{\vtheta_{\rnarr}}$ for all $l'\in [l-1]$. Then
\begin{align}
	\vf^{[l]}_{\vtheta_{\rwide}}
	&= \sigma  \left(\left[ {\begin{array}{cc}
			\mW^{[l]} \\
			\mW^{[l]}_{s,[1:m_{l-1}]} \\
	\end{array} } \right] \vf^{[l-1]}_{\vtheta_{\rnarr}}+
	\left[ {\begin{array}{cc}
			\vb^{[l]} \\
			\vb^{[l]}_s \\
	\end{array} } \right]
	\right)
	= \left[ {\begin{array}{cc}
			\vf^{[l]}_{\vtheta_{\rnarr}} \\
			(\vf_{\vtheta_{\rnarr}}^{[l]})_s\\
	\end{array} } \right].
\end{align}
Note that
% \begin{equation*}
	%     \alpha\left[\mzero_{m_{l+1}\times (s-1)},-\mW^{[l+1]}_{[1:m_{l+1}],s},\mzero_{m_{l+1}\times (m_{l}-s)},\mW^{[l+1]}_{[1:m_{l+1}],s}\right]\left[ {\begin{array}{cc}
			%     \vf^{[l]}_{\vtheta} \\
			%     (\vf^{[l]}_{\vtheta})_s \\
			%     \end{array} } \right]=\mzero_{m_{l+1}\times1}.
	% \end{equation*}
\begin{equation*}
	\alpha\left[\mzero_{m_{l+1}\times (s-1)},-\mW^{[l+1]}_{[1:m_{l+1}],s},\mzero_{m_{l+1}\times (m_{l}-s)},\mW^{[l+1]}_{[1:m_{l+1}],s}\right]\left[ {\begin{array}{cc}
			\vf^{[l]}_{\vtheta_{\rnarr}} \\
			(\vf_{\vtheta_{\rnarr}}^{[l]})_s\\
	\end{array} } \right]=\mzero_{m_{l+1}\times1}.
\end{equation*}
Thus
\begin{align}
	\vf^{[l+1]}_{\vtheta_{\rwide}} 
	&=\sigma \left(
	\left[\mW^{[{l+1}]},\mzero_{m_{l+1}\times1}\right]
	\left[ {\begin{array}{cc}
			\vf^{[l]}_{\vtheta_{\rnarr}} \\
			(\vf_{\vtheta_{\rnarr}}^{[l]})_s\\
	\end{array} } \right]+\mzero_{m_{l+1}\times1}+
	\vb^{[l+1]} +\alpha~\mzero_{m_{l+1}\times1}
	\right) = \vf_{\vtheta_{\rnarr}}^{[l+1]}.
\end{align}
Next, by the construction of $\vtheta_{\rwide}$, it is clear that $\vf^{[l']}_{\vtheta_{\rwide}}=\vf^{[l']}_{\vtheta_{\rnarr}}$ for any $l'\in[l+1:L]$. 

(ii) The results for feature gradients $\vg^{[l']}_{\vtheta_{\rwide}}=\vg^{[l']}_{\vtheta_{\rnarr}}$ for $l'\in[L]$ can be calculated in a similar way except by replacing $\sigma(\cdot)$ with  $\sigma^{(1)}(\cdot)$.

(iii)  By the backpropagation and the above facts in (i), we have \[\vz^{[L]}_{\vtheta_{\rwide}}=\nabla\ell(\vf^{[L]}_{\vtheta_{\rwide}},\vy)=\nabla\ell(\vf^{[L]}_{\vtheta_{\rnarr}},\vy)=\vz^{[L]}_{\vtheta_{\rnarr}}.\]

Recall the recurrence relation for $l'\in[l+1:L-1]$, then we recursively obtain the following equality for $l'$ from $L-1$ down to $l+1$:
\begin{equation}
	\vz^{[l']}_{\vtheta_{\rwide}}
	=(\mW^{[l'+1]})^\T \left(\vz^{[l'+1]}_{\vtheta_{\rwide}}\circ \vg^{[l'+1]}_{\vtheta_{\rwide}}\right)
	=(\mW^{[l'+1]})^\T \left(\vz^{[l'+1]}_{\vtheta_{\rnarr}}\circ \vg^{[l'+1]}_{\vtheta_{\rnarr}}\right)
	=\vz_{\vtheta_{\rnarr}}^{[l']}.
\end{equation}
Next,
\begin{align}
	\vz^{[l]}_{\vtheta_{\rwide}}
	&=\left(\left[\mW^{[{l+1}]},\mzero_{m_{l+1}\times1}\right]
	+\alpha \left[\mzero_{m_{l+1}\times (s-1)},-\mW^{[l+1]}_{[1:m_{l+1}],s},\mzero_{m_{l+1}\times (m_{l}-s)},\mW^{[l+1]}_{[1:m_{l+1}],s}\right]\right)^\T
	\left(\vz^{[l+1]}_{\vtheta_{\rwide}}\circ \vg ^{[l+1]}_{\vtheta_{\rwide}}\right) \nonumber\\
	&=\left[ {\begin{array}{cc}
			\vz^{[l]}_{\vtheta_{\rnarr}} \\
			0 \\
	\end{array} } \right]+
	\left[ {\begin{array}{cc}
			\mzero_{(s-1)\times1}\\
			-\alpha(\vz^{[l]}_{\vtheta_{\rnarr}})_s \\
			\mzero_{(m_l-s)\times1} \\
			\alpha(\vz^{[l]}_{\vtheta_{\rnarr}})_s
	\end{array} } \right]\nonumber\\
	&=\left[ (\vz_{\vtheta_{\rnarr}}^{[l]})^\T_{[1:s-1]},(1-\alpha)(\vz_{\vtheta_{\rnarr}}^{[l]})_s, (\vz_{\vtheta_{\rnarr}}^{[l]})^\T_{[s+1:m_l]},\alpha (\vz_{\vtheta_{\rnarr}}^{[l]})_s \right]^\T.
\end{align}
Finally, 
\begin{align}
	\vz^{[l-1]}_{\vtheta_{\rwide}}
	&=\left[(\mW^{[l]})^\T,(\mW^{[l]})_{s,[1:m_{l-1}]}^\T\right] 
	\left( \left(\left[ {\begin{array}{cc}
			\vz^{[l]}_{\vtheta_{\rnarr}} \\
			0 \\
	\end{array} } \right]+
	\left[ {\begin{array}{cc}
			\mzero_{(s-1)\times1}\\
			-\alpha(\vz^{[l]}_{\vtheta_{\rnarr}})_s \\
			\mzero_{(m_l-s)\times1} \\
			\alpha(\vz^{[l]}_{\vtheta_{\rnarr}})_s
	\end{array} } \right]\right)
	\circ \left[ {\begin{array}{cc}
			\vg^{[l]}_{\vtheta_{\rnarr}} \\
			(\vg^{[l]}_{\vtheta_{\rnarr}})_s \\
	\end{array} } \right]\right)\nonumber\\
	&=(\mW^{[l]})^\T\left(\vz^{[l]}_{\vtheta_{\rnarr}}\circ\vg^{[l]}_{\vtheta_{\rnarr}}\right)+\mzero_{m_{l-1}\times 1}=\vz_{\vtheta_{\rnarr}}^{[l-1]}.
\end{align}
This with the recurrence relation once again leads to $\vz^{[l']}_{\vtheta_{\rwide}}=\vz_{\vtheta_{\rnarr}}^{[l']}$ for all $l'\in[l-1]$.

(iv) If for $\vtheta_1,\vtheta_2\in\mathrm{Tuple}_{\{m_0,\cdots,m_L\}}$ and $\vtheta_1\neq\vtheta_2$, since $\fT_{l,s}^{\alpha}|_k$ for $k\neq l, l+1$ is the identity map, then if there exists some $k_0\neq l, l+1$, such that $\mW_1^{[k_0]}\neq \mW_2^{[k_0]}$ or $\vb_1^{[k_0]}\neq \vb_2^{[k_0]}$,  then obviously $\fT_{l,s}^{\alpha}(\vtheta_1)\neq\fT_{l,s}^{\alpha}(\vtheta_2)$.   If   $k_0=l$ or $k_0=l+1$, by similar reasoning, $\fT_{l,s}^{\alpha}(\vtheta_1)\neq\fT_{l,s}^{\alpha}(\vtheta_2)$. 

(v) For $\vtheta_{ 0}=(\mW_0^{[1]},\vb_0^{[1]},\cdots,\mW_0^{[L]},\vb_0^{[L]})\in\mathrm{Tuple}_{\{m_0,\cdots,m_L\}}$, we have 
\begin{align*}
	& \Tilde{\fT}_{l,s}^{\alpha}(\vtheta_0) =\Big(\mW_0^{[1]},\vb_0^{[1]},\cdots,\left[ {\begin{array}{cc}
			\mW_0^{[l]} \\
			({\mW_0^{[l]}})_{s,[1:m_{l-1}]} \\
	\end{array} } \right],
	\left[ {\begin{array}{cc}
			\vb_0^{[l]} \\
			({\vb_0^{[l]}})_s \\
	\end{array} } \right],   \\ &\left[(\mW^{[l+1]}_0)_{[1:m_{l+1}],[1:s-1]},(1-\alpha)(\mW^{[l+1]}_0)_{[1:m_{l+1}],s},(\mW^{[l+1]}_0)_{[1:m_{l+1}],[s+1:m_l]},\alpha(\mW_0^{[l+1]})_{[1:m_{l+1}],s}\right],\\
	&~~\vb_0^{[l+1]},
	\cdots, \mW_0^{[L]},\vb_0^{[L]}\Big),
\end{align*}
obviously $\Tilde{\fT}_{l,s}^{\alpha}$ is a linear operator, thus $\fT_{l,s}^{\alpha}$ is an affine operator.
\end{proof}
%%%%%%%%%%%%%%%%%%%%%%%%%%%%%%%%%%%%%%%%
Directly from Lemma \ref{lem...appen...section1...null} and Lemma \ref{lem...appen...section1...split}, we obtain that both one-step null embedding and one-step splitting embedding satisfy the property of output preserving and representation preserving, and all we need is to check the property of criticality preserving.
%%%%%%%%%%%%%%%%%%%%%%%%%%%%%%%%%%%%%%%%
%%%%%%%%%%%%%%%%%%%%%%%%%%%%%%%%%%%%%%%%
\begin{prop*}[Proposition \ref{prop...null} in main  text]
For any one-step null embedding $\fT_{l,0}^{\alpha}$, given any  $\mathrm{NN}(\{m_l\}_{l=0}^{L})$ and its parameters
$\vtheta_{\rnarr}\in\mathrm{Tuple}_{\{m_0,\cdots,m_L\}}$ with $\mathrm{Tuple}_{\{m_0,\cdots,m_L\}}\in\fD_{l,0}$, we have $\vtheta_{\rwide}:=\fT_{l,0}^{\alpha}(\vtheta_{\rnarr})$ satisfies the following conditions: given any data $S$, loss $\ell(\cdot,\cdot)$ and activation $\sigma(\cdot)$,
if $\nabla_{\vtheta}\RS(\vtheta_{\rnarr})=\mzero$, then $\nabla_{\vtheta}\RS(\vtheta_{\rwide})=\mzero$.
\end{prop*}
%%%%%%%%%%%%%%%%%%%%%%%%%%%%%%%%%%%%%%%%%%%%%%%%%%%%%%%%
\begin{proof}
Gradient of loss with respect to network parameters of each layer can be computed from $\vF$, $\vG$, and $\vZ$ as follows
\begin{align*}
	\nabla_{\mW^{[l']}}R_S(\vtheta)
	&= \nabla_{\mW^{[l']}}\Exp_S \ell(\vf_{\vtheta}(\vx),\vy)=\Exp_S\left(\left(\vz_{\vtheta}^{[l']}\circ \vg^{[l']}_{\vtheta}\right)(\vf_{\vtheta}^{[l'-1]})^\T\right),\\
	\nabla_{\vb^{[l']}}R_S(\vtheta)
	&= \nabla_{\vb^{[l]}}\Exp_S \ell(\vf_{\vtheta}(\vx),\vy)=\Exp_S\left(\vz_{\vtheta}^{[l']}\circ \vg^{[l']}_{\vtheta}\right).
\end{align*}
Then, by Lemma \ref{lem...appen...section1...null}, we have for $l'\neq l, l+1$, \[\nabla_{\mW^{[l']}}R_S(\vtheta_{\rwide})=
\nabla_{\mW^{[l']}}R_S(\vtheta_{\rnarr})=\mzero,\] 
and 
\[\nabla_{\vb^{[l']}}R_S(\vtheta_{\rwide})=
\nabla_{\vb^{[l']}}R_S(\vtheta_{\rnarr})=\mzero.\]
Also, for any $j\in[m_{l+1}]$,~$k\in[m_{l}]$, since $(\vz^{[l+1]}_{\vtheta_{\rwide}})_j=(\vz_{\vtheta_{\rnarr}}^{[l+1]})_j$, 
$(\vg^{[l+1]}_{\vtheta_{\rwide}})_j=(\vg^{[l+1]}_{\vtheta_{\rnarr}})_j$,  
$(\vf^{[l]}_{\vtheta_{\rwide}})_k=(\vf^{[l]}_{\vtheta_{\rnarr}})_k$, and $\mW_{j,(m_l+1)}^{[l+1]}\equiv 0$,
we obtain that
\begin{align*}
	\nabla_{\mW^{[l+1]}_{j,k}}R_S(\vtheta_{\rwide})
	&= \nabla_{\mW^{[l+1]}_{j,k}}R_S(\vtheta_{\rnarr})=0,\\
	\nabla_{\mW^{[l+1]}_{j,(m_l+1)}}R_S(\vtheta_{\rwide})
	&=0,\\
	\nabla_{\vb^{[l+1]}_{j}}R_S(\vtheta_{\rwide})
	&= \nabla_{\vb^{[l+1]}_{j}}R_S(\vtheta_{\rnarr})=0.
\end{align*}
Similarly, for any $j\in[m_{l}]$,~$k\in[m_{l-1}]$, since $(\vz^{[l]}_{\vtheta_{\rwide}})_j=(\vz_{\vtheta_{\rnarr}}^{[l]})_j$, 
$(\vg^{[l]}_{\vtheta_{\rwide}})_j=(\vg^{[l]}_{\vtheta_{\rnarr}})_j$,  
$(\vf^{[l-1]}_{\vtheta_{\rwide}})_k=(\vf^{[l-1]}_{\vtheta_{\rnarr}})_k$, and $\mW_{(m_l+1),k}^{[l]}\equiv 0$, we have
\begin{align*}
	\nabla_{\mW^{[l]}_{j,k}}R_S(\vtheta_{\rwide})
	&= \nabla_{\mW^{[l]}_{j,k}}R_S(\vtheta_{\rnarr})=0,\\
	\nabla_{\mW^{[l]}_{(m_l+1),k}}R_S(\vtheta_{\rwide})
	&=0,\\
	\nabla_{\vb^{[l]}_{j}}R_S(\vtheta_{\rwide})
	&= \nabla_{\vb^{[l]}_{j}}R_S(\vtheta_{\rnarr})=0.
\end{align*}  
Moreover, by Lemma \ref{lem...appen...section1...null}, the output function $\vf^{[L]}_{\vtheta_{\rwide}}=\vf_{\vtheta_{\rnarr}}^{[L]}$ is independent of the hyperparameter $\alpha$, and $R_S(\vtheta_{\rwide})=R_S(\vtheta_{\rnarr})$, then since $\vb^{[l]}_{(m_l+1)}=\alpha$, we have
\[
\nabla_{\vb^{[l]}_{(m_l+1)}}R_S(\vtheta_{\rwide})
= \frac{\partial}{\partial \alpha}R_S(\vtheta_{\rnarr})=0.
\] Collecting all the above relations, we obtain that $\nabla_{\vtheta}\RS(\vtheta_{\rwide})=\mzero$.
\end{proof}
%%%%%%%%%%%%%%%%%%%%%%%%%%%%%%%%%%%%%%%%%%%%%%%%%%%%%%%%
%%%%%%%%%%%%%%%%%%%%%%%%%%%%%%%%%%%%%%%%
\begin{prop*}[Proposition \ref{prop...split} in main text]
For any one-step splitting embedding $\fT_{l,s}^{\alpha}$, given any  $\mathrm{NN}(\{m_l\}_{l=0}^{L})$ and its parameters
$\vtheta_{\rnarr}\in\mathrm{Tuple}_{\{m_0,\cdots,m_L\}}$ with $\mathrm{Tuple}_{\{m_0,\cdots,m_L\}}\in\fD_{l,s}$, we have $\vtheta_{\rwide}:=\fT_{l,s}^{\alpha}(\vtheta_{\rnarr})$ satisfies the following conditions: given any data $S$, loss $\ell(\cdot,\cdot)$ and activation $\sigma(\cdot)$, if $\nabla_{\vtheta}\RS(\vtheta_{\rnarr})=\mzero$, then $\nabla_{\vtheta}\RS(\vtheta_{\rwide})=\mzero$.
\end{prop*}
%%%%%%%%%%%%%%%%%%%%%%%%%%%%%%%%%%%%%%%%%%%%%%%%%%%%%%%%
\begin{proof}
Gradient of loss with respect to network parameters of each layer can be computed from $\vF$, $\vG$, and $\vZ$ as follows
\begin{align*}
	\nabla_{\mW^{[l']}}R_S(\vtheta)
	&= \nabla_{\mW^{[l']}}\Exp_S \ell(\vf_{\vtheta}(\vx),\vy)=\Exp_S\left(\left(\vz_{\vtheta}^{[l']}\circ \vg^{[l']}_{\vtheta}\right)(\vf_{\vtheta}^{[l'-1]})^\T\right),\\
	\nabla_{\vb^{[l']}}R_S(\vtheta)
	&= \nabla_{\vb^{[l]}}\Exp_S \ell(\vf_{\vtheta}(\vx),\vy)=\Exp_S\left(\vz_{\vtheta}^{[l']}\circ \vg^{[l']}_{\vtheta}\right).
\end{align*}
Then, by Lemma \ref{lem...appen...section1...split}, we have for $l'\neq l, l+1$, \[\nabla_{\mW^{[l']}}R_S(\vtheta_{\rwide})=
\nabla_{\mW^{[l']}}R_S(\vtheta_{\rnarr})=\mzero,\] 
and 
\[\nabla_{\vb^{[l']}}R_S(\vtheta_{\rwide})=
\nabla_{\vb^{[l']}}R_S(\vtheta_{\rnarr})=\mzero.\]
Also, for any $j\in[m_{l+1}]$,~$k\in[m_{l}]$, since $(\vz^{[l+1]}_{\vtheta_{\rwide}})_j=(\vz_{\vtheta_{\rnarr}}^{[l+1]})_j$, 
$(\vg^{[l+1]}_{\vtheta_{\rwide}})_j=(\vg^{[l+1]}_{\vtheta_{\rnarr}})_j$,  
$(\vf^{[l]}_{\vtheta_{\rwide}})_k=(\vf^{[l]}_{\vtheta_{\rnarr}})_k$, and and $(\vf^{[l]}_{\vtheta_{\rwide}})_{m_l+1}=(\vf^{[l]}_{\vtheta_{\rnarr}})_s$, we obtain
\begin{align*}
	\nabla_{\mW^{[l+1]}_{j,k}}R_S(\vtheta_{\rwide})
	&= \nabla_{\mW^{[l+1]}_{j,k}}R_S(\vtheta_{\rnarr})=0,\\
	\nabla_{\mW^{[l+1]}_{j,(m_l+1)}}R_S(\vtheta_{\rwide})
	&=0,\\
	\nabla_{\vb^{[l+1]}_{j}}R_S(\vtheta_{\rwide})
	&= \nabla_{\vb^{[l+1]}_{j}}R_S(\vtheta_{\rnarr})=0.
\end{align*}
Similarly, for any $j\in[m_{l}]\backslash \{s\}$,   \[(\vz^{[l]}_{\vtheta_{\rwide}})_j=(\vz_{\vtheta_{\rnarr}}^{[l]})_j,~
(\vz^{[l]}_{\vtheta_{\rwide}})_{s}=(1-\alpha)(\vz_{\vtheta_{\rnarr}}^{[l]})_s,~ (\vz^{[l]}_{\vtheta_{\rwide}})_{(m_l+1)}=\alpha(\vz_{\vtheta_{\rnarr}}^{[l]})_s,\] and  for any $i\in[m_{l}]$,
\[(\vg^{[l]}_{\vtheta_{\rwide}})_i=(\vg^{[l]}_{\vtheta_{\rnarr}})_i,~ (\vg^{[l]}_{\vtheta_{\rwide}})_{(m_l+1)}=(\vg^{[l]}_{\vtheta_{\rnarr}})_s,\] and
for~$k\in[m_{l-1}]$,
\[(\vf^{[l-1]}_{\vtheta_{\rwide}})_k=(\vf^{[l-1]}_{\vtheta_{\rnarr}})_k,\]hence for any $j\in[m_{l}]\backslash \{s\}$,~$k\in[m_{l-1}]$:
\begin{align*}
	\nabla_{\mW^{[l]}_{j,k}}R_S(\vtheta_{\rwide})
	&= \nabla_{\mW^{[l]}_{j,k}}R_S(\vtheta_{\rnarr})=0,\\
	\nabla_{\vb^{[l]}_{j}}R_S(\vtheta_{\rwide})
	&= \nabla_{\vb^{[l]}_{j}}R_S(\vtheta_{\rnarr})=0,\\
	\nabla_{\mW^{[l]}_{s,k}}R_S(\vtheta_{\rwide})
	&= (1-\alpha)\nabla_{\mW^{[l]}_{s,k}}R_S(\vtheta_{\rnarr})=0,\\
	\nabla_{\mW^{[l]}_{(m_l+1),k}}R_S(\vtheta_{\rwide})
	&= \alpha\nabla_{\mW^{[l]}_{s,k}}R_S(\vtheta_{\rnarr})=0,\\
	\nabla_{\vb^{[l]}_{s}}R_S(\vtheta_{\rwide})
	&= (1-\alpha)\nabla_{\vb^{[l]}_{s}}R_S(\vtheta_{\rnarr})=0,\\
	\nabla_{\vb^{[l]}_{(m_l+1)}}R_S(\vtheta_{\rwide})
	&= \alpha\nabla_{\vb^{[l]}_{s}}R_S(\vtheta_{\rnarr})=0.
\end{align*}
Collecting all the above relations, we obtain that $\nabla_{\vtheta}\RS(\vtheta_{\rwide})=\mzero$.
\end{proof}
%%%%%%%%%%%%%%%%%%%%%%%%%%%%%%%%%%%%%%%%%%%%%%%%%%%%%%%%
%%%%%%%%%%%%%%%%%%%%%%%%%%%%%%%%%%%%%%%%%%%%%%%%%%%%%%%%
Combining altogether Lemma \ref{lem...appen...section1...null}, Lemma \ref{lem...appen...section1...split}, Proposition \ref{prop...null} and Proposition \ref{prop...split}, we finish our proof for Theorem \ref{thm...Null+SplitisCritical}.
%%%%%%%%%%%%%%%%%%%%%%%%%%%%%%%%%%%%%%%%%%%%%%%%%%%%%%%%

%%%%%%%%%%%%%%%%%%%%%%%%%%%%%%%%%%%%%%%%%%%%%%%%%
\begin{theorem*}[Theorem \ref{prop...KStep} in main text]
A $K$-step composition embedding is a critical embedding. 
\end{theorem*}
%%%%%%%%%%%%%%%%%%%%%%%%%%%%%%%%%%%%%%%%%%
\begin{proof}
We shall prove it using induction.

For $K=1$, Proposition \ref{prop...KStep} holds since both one-step null embedding and one step splitting embedding are critical embeddings.

Assume that  Proposition \ref{prop...KStep} holds for $K=l-1$, we want to show that it also holds for $K=l$.

From the induction hypothesis, we only need to show that if given two critical embeddings $\fT_1$, $\fT_2$, then $\fT_2\fT_1$ is also a critical embedding.

(i) $\fT_2\fT_1$ is injective:\\
For $\vtheta_1$ and $\vtheta_2$ belonging to a same tuple class but  $\vtheta_1\neq \vtheta_2$, since $\fT_1$ is injective, then $\fT_1(\vtheta_1)\neq \fT_1(\vtheta_2)$. Since $\fT_2$ is injective, then $\fT_2\fT_1(\vtheta_1)\neq \fT_2\fT_1(\vtheta_2)$.

(ii) $\fT_2\fT_1$ is an affine embedding:\\
We have that $\Tilde{\fT}_1$  and $\Tilde{\fT}_2$ are  linear operators, then   \[
\widetilde{\fT_2\fT_1}(\vtheta):=\fT_2\fT_1(\vtheta)-\fT_2\fT_1(\vzero),
\]
we need to show that for any $\vtheta_1, \vtheta_2$ and $c\in\sR$,
\[  \widetilde{\fT_2\fT_1}(\vtheta_1+\vtheta_2)= \widetilde{\fT_2\fT_1}(\vtheta_1)+\widetilde{\fT_2\fT_1}(\vtheta_2),\]
and 
\[  \widetilde{\fT_2\fT_1}(c\vtheta_1 )= c\widetilde{\fT_2\fT_1}(\vtheta_1).\]

Since 
\begin{align*}
	\widetilde{\fT_2\fT_1}(\vtheta_1+\vtheta_2)&=\fT_2\fT_1(\vtheta_1+\vtheta_2)-\fT_2\fT_1(\vzero)\\
	&=\fT_2\left(\fT_1(\vtheta_1)+\fT_1(\vtheta_2)-\fT_1(\vzero)\right)-\fT_2\fT_1(\vzero)\\
	&=\fT_2\left(\fT_1(\vtheta_1)\right)+\fT_2\left(\fT_1(\vtheta_2)\right)-\fT_2\left(\fT_1(\vzero)\right)-\fT_2\fT_1(\vzero)\\
	&=\widetilde{\fT_2\fT_1}(\vtheta_1)+\widetilde{\fT_2\fT_1}(\vtheta_2),
\end{align*}
and 
\begin{align*}
	\widetilde{\fT_2\fT_1}(c\vtheta_1)&=\fT_2\fT_1(c\vtheta_1)-\fT_2\fT_1(\vzero)\\
	&=\fT_2\left(c\fT_1(\vtheta_1)+(1-c)\fT_1(\vzero)\right)-\fT_2\fT_1(\vzero)\\
	&=c \fT_2\left(\fT_1(\vtheta_1)\right)+(1-c)\fT_2\left(\fT_1(\vzero)\right)-\fT_2\fT_1(\vzero)\\
	&=c\widetilde{\fT_2\fT_1}(\vtheta_1).
\end{align*}

(iii) $\fT_2\fT_1$ satisfies the property of output preserving:\\
Since $\fT_1$ satisfies the property of output preserving, then for any $\vtheta$, $\vf_{\fT_1(\vtheta)}=\vf_{\vtheta}$. Similarly for  $\fT_2$, we have $\vf_{\fT_2\fT_1(\vtheta)}=\vf_{\fT_1(\vtheta)}$, hence $\vf_{\fT_2\fT_1(\vtheta)}=\vf_{\vtheta}$. 

(iv) $\fT_2\fT_1$ satisfies the property of representation preserving:\\
Similar reasoning in (iii).

(v) $\fT_2\fT_1$ satisfies the property of criticality preserving:\\
Since $\vtheta$ is a critical point of $\RS({\vtheta})$, so is $\fT_1(\vtheta)$, and $\fT_2\fT_1(\vtheta)$ as well, we finish the proof.
\end{proof}

\begin{lemma*}[Lemma \ref{lem...Output.to.critPreser} in main text]
For any affine embedding $\fT:\mathrm{Tuple}_{\{m_0,\cdots,m_L\}}\to\mathrm{Tuple}_{\{m'_0,\cdots,m'_L\}}$ satisfying the output preserving property, 
% and we have two NNs, $\mathrm{NN}(\{m_l\}_{l=0}^{L})$ and  $\mathrm{NN}(\{m'_l\}_{l=0}^{L})$,   %then for any given    $\fT$    maps  its network parameters $\vtheta_{\rnarr}\in\mathrm{Tuple}_{\{m_0,\cdots,m_L\}}$ to that of a wider NN $\vtheta_{\rwide}=\fT(\vtheta_{\rnarr})\in\mathrm{Tuple}_{\{m'_0,\cdots,m'_L\}}$. \\
if there exist a total index mapping $\fI=\{\fI_l\}_{l=0}^{L}$ from $\mathrm{NN}(\{m'_l\}_{l=0}^{L})$ to $\mathrm{NN}(\{m_l\}_{l=0}^{L})$
and  auxiliary  variables $\vbeta=\left\{\vbeta^{[l]}_j\in\sR|~l\in[L],j\in[m'_l]\backslash\fI_l^{-1}(0)\right\}$, 
such that for any given neuron   belonging to $\mathrm{NN}(\{m'_l\}_{l=0}^{L})$, located in layer $l$ with index $j$, the following two statements hold:  \\
(i) If $\fI_l(j)\neq 0$, $(\vf_{\vtheta_{\rwide}}^{[l]})_j=(\vf_{\vtheta_{\rnarr}}^{[l]})_{\fI_l(j)}$ and $(\ve_{\vtheta_{\rwide}}^{[l]})_j=\vbeta_j^{[l]}(\ve_{\vtheta_{\rnarr}}^{[l]})_{\fI_l(j)}$,\\
(ii) If $\fI_l(j)= 0$,  $(\vf_{\vtheta_{\rwide}}^{[l]})_j=\mathrm{Const}$ and $(\ve_{\vtheta_{\rwide}}^{[l]})_j=0$,\\
then $\fT$ is a critical embedding.
\end{lemma*}

\begin{proof}
For any critical point $\vtheta_{\rnarr}^{\rc}\in\vTheta_{\rnarr}^{\rc}$, we set $\vtheta_{\rwide}:=\fT(\vtheta_{\rnarr}^{\rc})$. Then, since we have for any $l'\in[L]$
\begin{align*}
	\nabla_{\mW^{[l']}}R_S(\vtheta)
	&= \nabla_{\mW^{[l']}}\Exp_S \ell(\vf_{\vtheta}(\vx),\vy)=\Exp_S\left(\left(\vz_{\vtheta}^{[l']}\circ \vg^{[l']}_{\vtheta}\right)(\vf_{\vtheta}^{[l'-1]})^\T\right),\\
	\nabla_{\vb^{[l']}}R_S(\vtheta)
	&= \nabla_{\vb^{[l]}}\Exp_S \ell(\vf_{\vtheta}(\vx),\vy)=\Exp_S\left(\vz_{\vtheta}^{[l']}\circ \vg^{[l']}_{\vtheta}\right).
\end{align*}
Hence for $\fI_l(i), \fI_{l-1}(j)\neq 0$
\begin{align*}
	\nabla_{(\mW^{[l]}_{\rwide})_{ij}}\RS(\vtheta_{\rwide})&=\Exp_S(\ve_{\vtheta_{\rwide}}^{[l]})_i(\vf_{\vtheta_{\rwide}}^{[l-1]})_j\\
	&=\Exp_S\vbeta^{[l]}_{i}(\ve_{\vtheta_{\rnarr}^{\rc}}^{[l]})_{\fI_{l}(i)}(\vf_{\vtheta_{\rnarr}^{\rc}}^{[l-1]})_{\fI_{l-1}(j)}\\
	&=\vbeta^{[l]}_{i}\nabla_{(\mW^{[l]}_{\rnarr})_{\fI_{l}(i),\fI_{l-1}(j)}}\RS(\vtheta^{\rc}_{\rnarr})=0,\\
	\nabla_{(\vb^{[l]}_{\rwide})_{i}}\RS(\vtheta_{\rwide})&=\Exp_S(\ve_{\vtheta_{\rwide}}^{[l]})_i\\
	&=\Exp_S\vbeta^{[l]}_{i}(\ve_{\vtheta_{\rnarr}^{\rc}}^{[l]})_{\fI_{l}(i)}\\
	&=\vbeta^{[l]}_{i}\nabla_{(\vb^{[l]}_{\rnarr})_{\fI_{l}(i)}}\RS(\vtheta^{\rc}_{\rnarr})=0.
\end{align*}
By condition (ii), these gradients are obviously $0$ for $\fI_l(i) = 0$ or $\fI_{l-1}(j) = 0$.
Therefore, $\vtheta_{\rwide}$ is also a critical point and $\fT$ is criticality preserving.

Since $\fI$ is a total index mapping, for any feature vector of $\mathrm{NN}(\{m_l\}_{l=0}^{L})$, the component of  which is also the output function of at least a neuron in the wide NN by condition (i). Moreover, any neuron output function of a neuron in the wide NN is either constant or output function of a neuron in the narrow NN by condition (i) and (ii). Therefore, $\fT$ is representation preserving. Then $\fT$ is a critical embedding.
\end{proof}

	\begin{theorem*}[Theorem \ref{thm:gce} in main text]
		General compatible embedding is  a critical embedding.
	\end{theorem*}
	Remark that we later name it as general compatible critical embedding in this work.
	%%%%%%%%%%%%%%%%%%%%%%%%%%%%%%%%%%%%%%%%%%%%%%%%%%%%%%%%%%
	\begin{proof}
		We need to prove four properties one by one.
		\begin{enumerate}
			\item $\fT^{\valpha}_{\fI}$ is output preserving and representation preserving;
			\item $\fT^{\valpha}_{\fI}$ is an injective operator;
			\item $\fT^{\valpha}_{\fI}$ is an affine embedding;
			\item $\fT^{\valpha}_{\fI}$ is criticality preserving.
		\end{enumerate}
		
		(i) We prove  output preserving and representation preserving by doing induction on layers.
		
		For the first layer, i.e., $l=1$, we have for $i\in[m'_1]$,
		\begin{align*}
			(\vf^{[1]}_{\vtheta_{\rwide}})_i
			&=\sigma\left(\sum_{j\in[m'_0]}(\mW^{[1]}_{\rwide})_{ij}\vx_j+(\vb^{[1]}_{\rwide})_i\right)\\
			&= \sigma\left(%\sum_{s\in[m_0]\cup\{0\}}
			\sum_{j\in[m_0]}\malpha^{[1]}_{ij}(\mW^{[1]}_{\rnarr})_{\fI_1(i),j}\vx_j+(\valpha^{[1]}_{\rb})_i+(\vb^{[1]}_{\rnarr})_{\fI_1(i)}\right),
		\end{align*}
		then for $i\notin \fI^{-1}_l(0)$, for each $j\in[m_0]$,
		since $ \valpha^{[1]}_{ij}=\sum_{s\in \fI^{-1}_{0}(j)}\valpha_{is}=1$, $(\valpha^{[1]}_{\rb})_i=0$, hence 
		\[(\vf^{[1]}_{\vtheta_{\rwide}})_i=(\vf^{[1]}_{\vtheta_{\rnarr}})_{\fI_1(i)}.\]
		Otherwise, for $i\in \fI^{-1}_1(0)$, we have $\malpha^{[1]}_{ij}=0$ and $(\valpha_{\rb}^{[1]})_i=(\vb^{[1]}_*)_i$, then 
		\begin{align*}
			(\vf^{[1]}_{\vtheta_{\rwide}})_i
			=\sigma\left((\vb^{[1]}_*)_i\right),
		\end{align*}
		which is a constant function playing the same role as the bias term in the next layer.
		
		Suppose for layer $l-1$, we have for $i\notin \fI^{-1}_{l-1}(0)$ \[(\vf^{[l-1]}_{\vtheta_{\rwide}})_i
		=(\vf^{[l-1]}_{\vtheta_{\rnarr}})_{\fI_{l-1}(i)},\]  and  for $i\in \fI^{-1}_{l-1}(0)$, 
		\[
		(\vf^{[l-1]}_{\vtheta_{\rwide}})_i
		=\sigma\left((\vb^{[l-1]}_*)_i\right).
		\]
		Then we want to show that this is also the case for layer $l$.
		
		We obtain that
		\begin{align*}
			(\vf^{[l]}_{\vtheta_{\rwide}})_i
			&=\sigma\left(\sum_{j\in[m'_{l-1}]}(\mW^{[l]}_{\rwide})_{ij}(\vf^{[l-1]}_{\vtheta_{\rwide}})_j+(\vb^{[l]}_{\rwide})_i\right)\\
			&= \sigma\left(\sum_{s\in[m_{l-1}]\cup\{0\}}\sum_{j\in \fI^{-1}_{l-1}(s)}\malpha^{[l]}_{ij}(\mW^{[l]}_{\rnarr})_{\fI_l(i),s}(\vf^{[l-1]}_{\vtheta_{\rwide}})_j+(\valpha^{[l]}_{\rb})_i+(\vb^{[l]}_{\rnarr})_{\fI_l(i)}\right)\\
			&= \sigma\Bigg(\sum_{s\in[m_{l-1}]}(\mW^{[l]}_{\rnarr})_{\fI_l(i),s}(\vf^{[l-1]}_{\vtheta_{\rnarr}})_{s}\sum_{j\in \fI^{-1}_{l-1}(s)}\malpha^{[l]}_{ij}+\sum_{j\in\fI^{-1}_{l-1}(0)}\malpha^{[l]}_{ij}\sigma\left((\vb^{[l-1]}_*)_j\right)\\
			&~~~~
			+(\valpha^{[l]}_{\rb})_i+(\vb^{[l]}_{\rnarr})_{\fI_l(i)}\Bigg).
		\end{align*}
		For $i\notin \fI^{-1}_l(0)$, we have 
		\begin{align*}
			(\vf^{[l]}_{\vtheta_{\rwide}})_i
			= \sigma\left(\sum_{s\in[m_{l-1}]}(\mW^{[l]}_{\rnarr})_{\fI_l(i),s}(\vf^{[l-1]}_{\vtheta_{\rnarr}})_{s}+(\vb^{[l]}_{\rnarr})_{\fI_l(i)}\right)=(\vf^{[l]}_{\vtheta_{\rnarr}})_{\fI_l(i)}.
		\end{align*}
		Otherwise, for $i\in \fI^{-1}_l(0)$, we have 
		\begin{align*}
			(\vf^{[l]}_{\vtheta_{\rwide}})_i
			=\sigma\left((\vb^{[l]}_*)_i\right).
		\end{align*}
		Then, for any layer $l$, we have for $i\notin \fI^{-1}_l(0)$, \[
		(\vf^{[l]}_{\vtheta_{\rwide}})_i=(\vf^{[l]}_{\vtheta_{\rwide}})_{\fI_l(i)},
		\]
		and for $i\in \fI^{-1}_l(0)$, \[
		(\vf^{[l]}_{\vtheta_{\rwide}})_i
		=\sigma\left((\vb^{[l]}_*)_i\right).\] 
		Hence, we have proved already that $\fT^{\valpha}_{\fI}$ is output preserving and representation preserving.
		
		(ii) We prove that $\fT^{\valpha}_{\fI}$ is injective.
		
		If for $\vtheta_{\rnarr,1}, \vtheta_{\rnarr,2}\in\mathrm{Tuple}_{\{m_0,\cdots,m_L\}}$ and $\vtheta_{\rnarr,1}\neq\vtheta_{\rnarr,2}$,  then  there exists some $l\in[L]$, such that $\mW_{\rnarr,1}^{[l]}\neq \mW_{\rnarr,2}^{[l]}$ or $\vb_{\rnarr,1}^{[l]}\neq \vb_{\rnarr,2}^{[l]}$.
		
		Then, if $\fT^{\valpha}_{\fI}$ is not injective, there exists $\vtheta_{\rnarr,1}\neq\vtheta_{\rnarr,2}$, such that $\vtheta_{\rwide,1}=\vtheta_{\rwide,2}$, where $\vtheta_{\rwide,1}:=\fT^{\valpha}_{\fI}(\vtheta_{\rnarr,1})$ and $\vtheta_{\rwide,2}:=\fT^{\valpha}_{\fI}(\vtheta_{\rnarr,2})$, and we want to show that this will never happen.
		
		Since  there exists  $l\in[L]$, such that $\mW_{\rnarr,1}^{[l]}\neq \mW_{\rnarr,2}^{[l]}$ or $\vb_{\rnarr,1}^{[l]}\neq \vb_{\rnarr,2}^{[l]}$. For the case  $\mW_{\rnarr,1}^{[l]}\neq \mW_{\rnarr,2}^{[l]}$, we obtain that there exists $i\in[m_{l+1}], j\in[m_l]$, such that $\left(\mW_{\rnarr,1}^{[l]}\right)_{i,j}\neq \left(\mW_{\rnarr,2}^{[l]}\right)_{i,j}$. We observe that 
		$\mW_{\rwide,1}^{[l]}=\valpha^{[l]}\circ\mW_{\rinter,1}^{[l]}$,  and $\mW_{\rwide,2}^{[l]}=\valpha^{[l]}\circ\mW_{\rinter,2}^{[l]}$. Then, since $\fI$ is a total index mapping, then for $i, j\neq 0$, $\fI^{-1}_l(i), \fI^{-1}_{l-1}(j)\neq \emptyset$, hence for any $k\in\fI^{-1}_l(i)$ and $l\in \fI^{-1}_{l-1}(j)$,
		\[
		\left(\mW_{\rwide,1}^{[l]}\right)_{k,l}=\valpha^{[l]}_{k,l}\left(\mW_{\rnarr,1}^{[l]}\right)_{i,j},
		\]
		and
		\[
		\left(\mW_{\rwide,2}^{[l]}\right)_{k,l}=\valpha^{[l]}_{k,l}\left(\mW_{\rnarr,2}^{[l]}\right)_{i,j}.
		\]
		If $\mW_{\rwide,1}^{[l]}= \mW_{\rwide,2}^{[l]}$, then 
		\[
		\sum_{s\in \fI^{-1}_{l-1}(j)}\left(\mW_{\rwide,1}^{[l]}\right)_{k,s}=\sum_{s\in \fI^{-1}_{l-1}(j)}\left(\mW_{\rwide,2}^{[l]}\right)_{k,s},
		\]
		hence 
		\[
		\sum_{s\in \fI^{-1}_{l-1}(j)}\valpha^{[l]}_{k,s}\left(\mW_{\rnarr,1}^{[l]}\right)_{i,j}=\sum_{s\in \fI^{-1}_{l-1}(j)}\valpha^{[l]}_{k,s}\left(\mW_{\rnarr,2}^{[l]}\right)_{i,j},
		\]
		and since $\sum_{s\in \fI^{-1}_{l-1}(j)}\valpha^{[l]}_{k,s}=1$, we obtain that 
		\[
		\left(\mW_{\rnarr,1}^{[l]}\right)_{i,j}=\left(\mW_{\rnarr,2}^{[l]}\right)_{i,j},
		\]
		which contradicts $\vtheta_{\rnarr,1}\neq\vtheta_{\rnarr,2}$.
		
		For the case  where
		$\vb_{\rnarr,1}^{[l]}\neq \vb_{\rnarr,2}^{[l]}$, since $(\valpha_{\rb}^{[l]})_i=0$ for any $l\in[L]$ with $i\notin\fI_l^{-1}(0)$. Then for any $k\in\fI^{-1}_l(j)$, $j\neq0$, 
		\[
		\left(\vb_{\rwide,1}^{[l]}\right)_{k}=\valpha^{[l]}_{\rb}+\left(\vb_{\rnarr,1}^{[l]}\right)_{j}=\left(\vb_{\rnarr,1}^{[l]}\right)_{j},
		\]
		and 
		\[
		\left(\vb_{\rwide,2}^{[l]}\right)_{k}=\valpha^{[l]}_{\rb}+\left(\vb_{\rnarr,2}^{[l]}\right)_{j}=\left(\vb_{\rnarr,2}^{[l]}\right)_{j},
		\]
		hence $\vb_{\rwide,1}^{[l]}\neq \vb_{\rwide,2}^{[l]}$, and we finish the injection proof.
		
		(iii) We prove that $\fT^{\valpha}_{\fI}$ is affine.
		
		It is obvious that $\fT^{\valpha}_{\fI}$ is affine since for any  $\vtheta \in\mathrm{Tuple}_{\{m_0,\cdots,m_L\}}$,
		$\tilde{\fT}^{\valpha}_{\fI}(\vtheta):={\fT}^{\valpha}_{\fI}(\vtheta)-{\fT}^{\valpha}_{\fI}(\vzero)$
		puts the weights and biases of null neuron as zero, and multiplies the weights and biases of effective by some constant, thus $\tilde{\fT}^{\valpha}_{\fI}$ is a linear operator.
		
		(iv) We prove that $\fT^{\valpha}_{\fI}$ is criticality preserving.
		
		We only need to check whether or not the conditions in Lemma \ref{lem...Output.to.critPreser} are satisfied, and we show this using induction.
		
		For layer $L$,  we have
		%$(\ve^{[L]}_{\vtheta_{\rwide}})_i=(\ve^{[L]}_{\vtheta_{\rwide}})_i=\vbeta^{[L]}_i(\ve^{[L]}_{\vtheta_{\rnarr}})_i$.
		\[
		\ve^{[L]}_{\vtheta_{\rwide}}=\vz_{\vtheta_{\rwide}}^{[L]}\circ\vg_{\vtheta_{\rwide}}^{[L]},
		\]
		since  $\fT^{\valpha}_{\fI}$ is output preserving for any activation, hence
		\[
		\ve^{[L]}_{\vtheta_{\rwide}}=\vz_{\vtheta_{\rwide}}^{[L]}\circ\vg_{\vtheta_{\rwide}}^{[L]}=\vz_{\vtheta_{\rnarr}}^{[L]}\circ\vg_{\vtheta_{\rnarr}}^{[L]}= \ve^{[L]}_{\vtheta_{\rnarr}},
		\]
		hence $\ve^{[L]}_{\vtheta_{\rwide}}=\vgamma^{[L]}_i\ve^{[L]}_{\vtheta_{\rnarr}}$, with $\vgamma^{[L]}_i=1$.
		
		%and the above relation coincides with Condition 1 condition:$\vbeta^{[L]}_k=1$ for $k\in[m'_L]$.
		For any $l\in[L]$, then suppose for layer $l$,  we have  for $i\notin \fI^{-1}_{l}(0)$, there exists $\vgamma^{[l]}_i\in\sR$, such that 
		\[
		\left(\ve^{[L]}_{\vtheta_{\rwide}}\right)_i=\vgamma^{[l]}_i\left(\ve^{[L]}_{\vtheta_{\rnarr}}\right)_{\fI_l(i)},
		\]
		and for $i\in \fI^{-1}_{l}(0)$,
		\[
		\left(\ve^{[L]}_{\vtheta_{\rwide}}\right)_i=0.
		\]
		Then  for layer $l-1$, we have
		\begin{align*}
			(\ve_{\vtheta_{\rwide}}^{[l-1]})_j
			&= (\vg_{\vtheta_{\rwide}}^{[l-1]})_j\sum_{i\in[m'_l]}(\mW^{[l]}_{\rwide})_{ij}(\ve_{\vtheta_{\rwide}}^{[l]})_i\\
			&= (\vg_{\vtheta_{\rwide}}^{[l-1]})_j\sum_{k\in[m_l]\cup\{0\}}\sum_{i\in \fI^{-1}_l(k)}\malpha^{[l]}_{ij}(\mW^{[l]}_{\rwide})_{ij}(\ve_{\vtheta_{\rwide}}^{[l]})_i\\
			&= (\vg_{\vtheta_{\rwide}}^{[l-1]})_j\sum_{k\in[m_l]}(\mW^{[l]}_{\rnarr})_{k,\fI_{l-1}(j)}(\ve_{\vtheta_{\rnarr}}^{[l]})_k\sum_{i\in \fI^{-1}_l(k)}\malpha^{[l]}_{ij}\vgamma^{[l]}_i,
		\end{align*}
		wlog, we denote $\vgamma^{[l-1]}_j:=\sum_{i\in \fI^{-1}_l(k)}\malpha^{[l]}_{ij}\vgamma^{[l]}_i$, then
		\begin{align*} 
			(\ve_{\vtheta_{\rwide}}^{[l-1]})_j&= (\vg_{\vtheta_{\rwide}}^{[l-1]})_j\vgamma^{[l-1]}_j\sum_{k\in[m_l]}(\mW^{[l]}_{\rnarr})_{k,\fI_{l-1}(j)}(\ve_{\vtheta_{\rnarr}}^{[l]})_k.%\\
			%    &= (\vg_{\vtheta_{\rwide}}^{[l-1]})_j\vbeta^{[l-1]}_j(\vz^{[l-1]}_{\vtheta_{\rnarr}})_{\fI_{l-1}(j)}.\\
		\end{align*}
		For $j\notin \fI^{-1}_{l-1}(0)$, we have 
		\begin{align*}
			(\ve_{\vtheta_{\rwide}}^{[l-1]})_j=\vgamma^{[l-1]}_j(\vg_{\vtheta_{\rnarr}}^{[l-1]})_{\fI_{l-1}(j)}\sum_{k\in[m_l]}(\mW^{[l]}_{\rnarr})_{k,\fI_{l-1}(j)}(\ve_{\vtheta_{\rnarr}}^{[l]})_k=\vgamma^{[l-1]}_j(\ve_{\vtheta_{\rnarr}}^{[l-1]})_j.
		\end{align*}
		Otherwise, for $j\in \fI^{-1}_{l-1}(0)$, we simply set $\vgamma^{[l-1]}_j=0$, then
		\begin{align*}
			(\ve_{\vtheta_{\rwide}}^{[l-1]})_j=\vbeta^{[l-1]}_j\sigma\left((\vb^{[l-1]}_*)_j\right)(\vz^{[l-1]}_{\vtheta_{\rnarr}})_{\fI_{l-1}(j)}=0.
		\end{align*}
		From the above proof, we find out that  
		$\vgamma:=\left\{\vgamma^{[l]}_j\in\sR|~l\in[0:L],~j\in[m'_l]\right\}$ satisfy the property of the  group  of  auxiliary  variables $\vbeta:=\left\{\vbeta^{[l]}_j\in\sR|~l\in[0:L],~j\in[m'_l]\backslash\fI_l^{-1}(0)\right\}$ in Condition 1, hence we finish our proof.
	\end{proof}

	\begin{theorem*}[Theorem \ref{t2} in main text]
		Given an $\mathrm{NN}(\{m_l\}_{l=0}^{L})$ and 
		any of its parameters $\vtheta\in \sR^M$, for any critical embedding $\fT:\sR^M\to \sR^{M'}$ to any wider $\mathrm{NN}(\{m'_l\}_{l=0}^{L})$, the number of positive, zero, negative eigenvalues of $\mH_S(\fT(\vtheta))$ is no less than the counterparts of $\mH_S(\vtheta)$.   
	\end{theorem*}
	\begin{proof}
		Because $\fT$ is a critical embedding, therefore, it is an affine injective operator associated with $\mA\in\sR^{M'\times M},\vc\in\sR^{M'}$, such that $\fT(\vtheta)=\mA\vtheta+\vc$.
		By the  output preserving  property of $\fT$, we have \[
		\RS(\vtheta)\equiv\RS(\mA\vtheta+\vc).
		\]
		Hence, 
		\[
		\nabla_{\vtheta}\nabla_{\vtheta}\RS(\vtheta)\equiv\nabla_{\vtheta}\nabla_{\vtheta}\RS(\mA\vtheta+\vc).\]
		
		Then
		$$
		\mA^{\T}\mH_S(\mA\vtheta+\vc)\mA \equiv \mH_S(\vtheta).
		$$
		Given any $\vtheta_0$, if $\mH_S(\vtheta_0)$ has $k$ negative eigenvalues $\{\lambda_{j}^{\mathrm{neg}}\}_{j=1}^k$ with associated orthonormal eigenvectors $\{\ve_{j}^{\mathrm{neg}}\}_{j=1}^k$, then $\{\mA\ve_{j}^{\mathrm{neg}}\}_{j=1}^k$ satisfies, for any $\ve_j^{\mathrm{neg}}$,
		\begin{align}
			(\mA\ve_{j}^{\mathrm{neg}})^\T\mH_S(\mA\vtheta_0+\vc)\mA\ve_{j}^{\mathrm{neg}}
			={\left(\ve^{\mathrm{neg}}_j\right)}^\T\mH_S(\vtheta)\ve_j^{\mathrm{neg}} = \lambda_{j}^{\mathrm{neg}}<0.
		\end{align}
		By full rankness of  $\mA$, we have 
		\[
		\mathrm{dim}\left(\mathrm{span}\left(\left\{\mA\ve_{j}^{\mathrm{neg}}\right\}_{j=1}^k\right)\right)=k.
		\] Thus, $\mH_S(\mA\vtheta_0+\vc)$ has at least $k$ negative eigenvalues. 
		Similarly, we can prove this result for the number of zero and positive eigenvalues.
		
		Hence, in particular, for any critical embedding $\fT$,
		the number of negative eigenvalues of $\mH_S(\vtheta)$ is no more than the counterpart of $\mH_S(\fT(\vtheta))$.
	\end{proof}
	
	We would like to introduce some additional notations in order to state Lemma \ref{lem4} and Lemma \ref{lem5}.   In order to calculate the Hessian $\mH_S(\vtheta)=\nabla_{\vtheta}\nabla_{\vtheta}\RS(\vtheta)$, we need to compute $\vv_S(\vtheta)$:
	\begin{align*}
		\vv_S(\vtheta)&:=\Exp_{S}\nabla\ell (\vf(\vx,\vtheta),\vf^*(\vx))^\T\nabla_{\vtheta}\vf_{\vtheta}(\vx)=\sum_{i=1}^{m_L} \Exp_{S}\partial_i\ell (\vf_{\vtheta},\vf^*)\nabla_{\vtheta}(\vf_{\vtheta})_i, 
	\end{align*}
	where $\partial_i\ell (\vf_{\vtheta},\vf^*)$ is the $i$-th element of  $\nabla\ell (\vf(\vx,\vtheta),\vf^*(\vx))$, and $(\vf_{\vtheta})_i$ is the $i$-th element of vector $\vf_{\vtheta}$,  then for the Hessian $\mH_S(\vtheta)$, we have
	\begin{align*}
		\mH_S(\vtheta)&=\nabla_{\vtheta}\nabla_{\vtheta}\RS(\vtheta)=\sum_{i=1}^{m_L}\Exp_{S}\nabla_{\vtheta}\left(\partial_i\ell (\vf_{\vtheta},\vf^*)\right)\nabla_{\vtheta}(\vf_{\vtheta})_i+\sum_{i=1}^{m_L}\Exp_{S}\partial_i\ell (\vf_{\vtheta},\vf^*)\nabla_{\vtheta}\nabla_{\vtheta}\left((\vf_{\vtheta})_i\right)\\
		&=\sum_{i,j=1}^{m_L}\Exp_{S}\partial_{ij}\ell (\vf_{\vtheta},\vf^*)\nabla_{\vtheta}(\vf_{\vtheta})_i\left(\nabla_{\vtheta}(\vf_{\vtheta})_j\right)^\T
		+
		\sum_{i=1}^{m_L}\Exp_{S}\partial_i\ell (\vf_{\vtheta},\vf^*)\nabla_{\vtheta}\nabla_{\vtheta}\left((\vf_{\vtheta})_i\right),
	\end{align*}
	where $\partial_{ij}\ell (\vf_{\vtheta},\vf^*)$ is the $(i,j)$-th element of  $\nabla\nabla\ell (\vf(\vx,\vtheta),\vf^*(\vx))$.
	
	Hence $\mH^{(1)}_S(\vtheta)$ and $\mH^{(2)}_S(\vtheta)$ respectively becomes
	\begin{equation}
		\mH^{(1)}_S(\vtheta):=\sum_{i,j=1}^{m_L}\Exp_{S}\partial_{ij}\ell (\vf_{\vtheta},\vf^*)\nabla_{\vtheta}(\vf_{\vtheta})_i\left(\nabla_{\vtheta}(\vf_{\vtheta})_j\right)^\T,
	\end{equation}
	and
	\begin{equation}
		\mH^{(2)}_S(\vtheta):=\sum_{i=1}^{m_L}\Exp_{S}\partial_i\ell (\vf_{\vtheta},\vf^*)\nabla_{\vtheta}\nabla_{\vtheta}\left((\vf_{\vtheta})_i\right).
	\end{equation}
	We observe that for $i\in[m_L]$, $j\in[m_{L-1}]$,
	\[
	(\vf_{\vtheta})_i=\sum_{j=1}^{m_{L-1}}(\mW^{[L]})_{i,j}\left(\vf_{\vtheta}^{[L-1]}\right)_j+(\vb^{[L]})_{i},
	\]
	hence  we obtain that, for any $i\in[m_L]$, 
	\begin{equation}\label{eq..derivativew.r.t[L]}
		\begin{aligned}
			\frac{\partial (\vf_{\vtheta})^i}{\partial \mW^{[
					L]}}&=\left[ {\begin{array}{cc}
					\mzero_{(i-1)\times m_{L-1}} \\
					\left(\vf_{\vtheta}^{[L-1]}\right)_j\\
					\mzero_{(m_L-i)\times m_{L-1}} \\
			\end{array} } \right],~~j\in[m_{L-1}],\\
			\frac{\partial (\vf_{\vtheta})^i}{\partial \vb^{[
					L]}}&=\left[ {\begin{array}{cc}
					\mzero_{(i-1)\times 1} \\
					1\\
					\mzero_{(m_L-i)\times 1} \\
			\end{array} } \right].
		\end{aligned}
	\end{equation}
	Finally, given  an   $\mathrm{NN}(\{m_l\}_{l=0}^{L})$ and its parameter
	\[\vtheta=(\mW^{[1]},\vb^{[1]},\cdots,\mW^{[L]},\vb^{[L]})\in\mathrm{Tuple}_{\{m_0,\cdots,m_L\}},\] we remind the readers once again that  the collection of parameters $\vtheta$ is a $2L$-tuple.
	%, and we also identify $\vtheta$ from its vectorization $\mathrm{vec}(\vtheta)\in \sR^M$ with $M=\sum_{l=0}^{L-1}(m_l+1) m_{l+1}$.
	We denote that the upper bracket $[L-1]$ by limiting ourselves to the first $2L-2$ element of the tuple, i.e., 
	\begin{equation}
		\vtheta^{[L-1]}:=(\mW^{[1]},\vb^{[1]},\cdots,\mW^{[L-2]},\vb^{[L-2]},\mW^{[L-1]},\vb^{[L-1]})\in\mathrm{Tuple}_{\{m_0,\cdots,m_{L-2},m_{L-1}\}},
	\end{equation}
	and similarly, we    identify $ \vtheta^{[L-1]}$ with its vectorization $\mathrm{vec}(\vtheta^{[L-1]})\in \sR^{M^{[L-1]}}$ with $M^{[L-1]}:=\sum_{l=0}^{L-2}(m_l+1) m_{l+1}$.

	Moreover, we observe that for $i\in[m_L]$, $j\in[m_{L-1}]$, 
	\[
	\sum_{i,j=1}^{m_L}\Exp_{S}\partial_i\ell (\vf_{\vtheta},\vf^*)\mW_{i,j}^{[L]}\nabla_{\vtheta^{[L-1]}}\nabla_{\vtheta^{[L-1]}}\left(\vf_{\vtheta}^{[L-1]}\right)_j=\sum_{i=1}^{m_L}\Exp_{S}\partial_i\ell (\vf_{\vtheta},\vf^*)\nabla_{\vtheta^{[L-1]}}\nabla_{\vtheta^{[L-1]}}\left((\vf_{\vtheta})_i\right),
	\]
	hence we denote hereafter the expression by 
	\[
	\mH^{(2),[L-1]}_S(\vtheta):=\sum_{i,j=1}^{m_L}\Exp_{S}\partial_i\ell (\vf_{\vtheta},\vf^*)\mW_{i,j}^{[L]}\nabla_{\vtheta^{[L-1]}}\nabla_{\vtheta^{[L-1]}}\left(\vf_{\vtheta}^{[L-1]}\right)_j.
	\]
	We denote further that 
	\[
	\mH^{(1),[L-1]}_S(\vtheta):=\sum_{i,j=1}^{m_L}\Exp_{S}\partial_{ij}\ell (\vf_{\vtheta},\vf^*)\nabla_{\vtheta^{[L-1]}}(\vf_{\vtheta})_i\left(\nabla_{\vtheta^{[L-1]}}(\vf_{\vtheta})_j\right)^\T,
	\]
	and $\mH^{[L-1]}_S(\vtheta):=\mH^{(1),[L-1]}_S(\vtheta)+\mH^{(2),[L-1]}_S(\vtheta)$.

	We state Lemma \ref{lem4} and Lemma \ref{lem5} as follows.
	%%%%%%%%%%%%%%%%%%%%%%%%%%%%%%%%%%%%%%%%%%%%%
	\begin{lemma*}[Lemma \ref{lem4} in main text]
		Given any data $S$, loss $\ell(\cdot,\cdot)$ and activation $\sigma(\cdot)$, for any NN, if a critical point $\vtheta^{\rc}\in\vTheta^{\rc}$ satisfies:\\ (i)~$\mH_S(\vtheta^{\rc})\succeq 0$;\\  
		(ii)~$\mH^{(2),[L-1]}_S(\vtheta^{\rc})\neq \mathbf{0}$,\\
		then there exists a general compatible critical embedding $\fT$, such that $\fT(\vtheta^{\rc})$ is a strict-saddle point.
	\end{lemma*}
	
	\begin{proof}
		% We define the following embedding $\fT$ doubling the neuron in each hidden layer such that $\vTheta=\fT\vtheta = [\vtheta=\vtheta,\vtheta'=\vtheta,\vtheta''=\left(-\mW^{[L]},-\vb^{[L]},\vtheta^{[L-1]}\right),\vtheta_\mathrm{c} = \mathbf{0}]$.
		We consider the three-fold global splitting embedding $\fT_{\mathrm{global}}$ defined in Example \ref{ex}. 
		%and let $\vTheta=\fT_G\vtheta = [\vtheta=\vtheta,\vtheta'=\vtheta,\vtheta''=\left(-\mW^{[L]},-\vb^{[L]},\vtheta^{[L-1]}\right),\vtheta_\mathrm{c} = \mathbf{0}]$. Then
		Obviously, $\fT_{\mathrm{global}}$ is a general compatible critical embedding. After choosing a critical point $\vtheta_{\rnarr}^{\rc}\in\vTheta^{\rc}$, and we have that $\vtheta_{\rwide}^{\rc}:=\fT_{\mathrm{global}}(\vtheta_{\rnarr}^{\rc})$. 
		
		$\vtheta_{\rnarr}^{\rc}$ and $\vtheta_{\rwide}^{\rc}$ are tuples, and we misuse these notations and identify them with  their vectorizations, i.e., $\vtheta_{\rnarr}^{\rc}\in \sR^M$ and $\vtheta_{\rwide}^{\rc}\in\sR^{M'}$ for some $M$ and $M'$. More specifically, we set
		\begin{align*}
			\vtheta_{\rnarr}^{\rc}&=\left(\mathrm{vec}(\mW^{[1]})^\T,  \cdots, \mathrm{vec}(\mW^{[L-2]})^\T, \mathrm{vec}(\mW^{[L-1]})^\T, {\vb^{[1]}}^\T,\cdots, {\vb^{[L-1]}}^\T, \mathrm{vec}(\mW^{[L]})^\T, {\vb^{[L]}}^\T \right)^\T,\\ 
			\vtheta_{\rwide}^{\rc}&=\left(\vtheta_1, \vtheta_2, \vtheta_3,  \mzero_{1\times {(M'-3M+2m_L)}}, \mathrm{vec}(\mW^{[L]})^\T, \mathrm{vec}(\mW^{[L]})^\T, \mathrm{vec}(-\mW^{[L]})^\T, {\vb^{[L]}}^\T  \right)^\T,
		\end{align*}
		with 
		\begin{align*}
			\vtheta_1:=\vtheta_2:=\vtheta_3:=\vtheta^{[L-1]}=\left(\mathrm{vec}(\mW^{[1]})^\T,  \cdots, \mathrm{vec}(\mW^{[L-2]})^\T, \mathrm{vec}(\mW^{[L-1]})^\T, {\vb^{[1]}}^\T,\cdots, {\vb^{[L-1]}}^\T\right),
		\end{align*}
		then we observe that 
		\[\vtheta_{\rnarr}^{\rc}=\left(\vtheta_1, \mathrm{vec}(\mW^{[L]})^\T,{\vb^{[L]}}^\T \right)^\T,\]
		and  we are able to do some computations:
		
		\begin{equation}
			\begin{aligned}
				\mH^{[L-1]}_S(\vtheta^{\rc}_{\rwide})&=\sum_{i,j=1}^{m_L}\Exp_{S}\partial_{ij}\ell (\vf_{\vtheta_{\rwide}},\vf^*)\nabla_{\vtheta_{\rwide}^{[L-1]}}(\vf_{\vtheta_{\rwide}})_i\left(\nabla_{\vtheta_{\rwide}^{[L-1]}}(\vf_{\vtheta_{\rwide}})_j\right)^\T\\
				&~~+\sum_{i,j=1}^{m_L}\Exp_{S}\partial_i\ell (\vf_{\vtheta_{\rwide}},\vf^*)\mW_{i,j}^{[L]}\nabla_{\vtheta_{\rwide}^{[L-1]}}\nabla_{\vtheta_{\rwide}^{[L-1]}}\left(\vf_{\vtheta_{\rwide}}^{[L-1]}\right)_j,
			\end{aligned}
		\end{equation}
		where  $\mW_{i,j}^{[L]}$ refers to the matrix components of $\mW_{\rwide}^{[L]}$, i.e., $\mW_{i,j}^{[L]}:=\left(\mW_{\rwide}^{[L]}\right)_{i,j}$.
		
		Then for the first part $
		\sum_{i,j=1}^{m_L}\Exp_{S}\partial_{ij}\ell (\vf_{\vtheta_{\rwide}},\vf^*)\nabla_{\vtheta_{\rwide}^{[L-1]}}(\vf_{\vtheta_{\rwide}})_i\left(\nabla_{\vtheta_{\rwide}^{[L-1]}}(\vf_{\vtheta_{\rwide}})_j\right)^\T
		$,    we have
		
		\begin{align*}
			&\sum_{i,j=1}^{m_L}\Exp_{S}\partial_{ij}\ell (\vf_{\vtheta_{\rwide}},\vf^*)\nabla_{\vtheta_{\rwide}^{[L-1]}}(\vf_{\vtheta_{\rwide}})_i\left(\nabla_{\vtheta_{\rwide}^{[L-1]}}(\vf_{\vtheta_{\rwide}})_j\right)^\T =\begin{pmatrix}
				\mA_{1,1}&\mA_{1,2}&\mA_{1,3}\\
				\mA_{2,1}&\mA_{2,2}&\mA_{2,3}\\
				\mA_{3,1}&\mA_{3,2}&\mA_{3,3}
			\end{pmatrix},
		\end{align*}
		with 
		\[
		\mA_{p,q}:=\Exp_{S}\sum_{i,j=1}^{m_L}\partial_{ij}\ell (\vf_{\vtheta_{\rwide}},\vf^*) \nabla_{\vtheta_p}(\vf_{\vtheta_{\rwide}})_i\left(\nabla_{\vtheta_q}(\vf_{\vtheta_{\rwide}})_j\right)^\T, 
		\]
		for $p\in[3]$ and $q\in[3]$.
		%  &=\Exp_S\begin{pmatrix} \partial_{ij}\ell \nabla_{\vtheta_1}(\vf_{\vtheta_{\rwide}})^i\left(\nabla_{\vtheta_1}(\vf_{\vtheta_{\rwide}})^j\right)^\T &\partial_{ij}\ell \nabla_{\vtheta_1}(\vf_{\vtheta_{\rwide}})^i\left(\nabla_{\vtheta_2}(\vf_{\vtheta_{\rwide}})^j\right)^\T & \partial_{ij}\ell \nabla_{\vtheta_1}(\vf_{\vtheta_{\rwide}})^i\left(\nabla_{\vtheta_3}(\vf_{\vtheta_{\rwide}})^j\right)^\T\\
			%      \partial_{ij}\ell \nabla_{\vtheta_2}(\vf_{\vtheta_{\rwide}})^i\left(\nabla_{\vtheta_1}(\vf_{\vtheta_{\rwide}})^j\right)^\T &\partial_{ij}\ell \nabla_{\vtheta_2}(\vf_{\vtheta_{\rwide}})^i\left(\nabla_{\vtheta_2}(\vf_{\vtheta_{\rwide}})^j\right)^\T & \partial_{ij}\ell \nabla_{\vtheta_2}(\vf_{\vtheta_{\rwide}})^i\left(\nabla_{\vtheta_3}(\vf_{\vtheta_{\rwide}})^j\right)^\T\\
			%       \partial_{ij}\ell \nabla_{\vtheta_3}(\vf_{\vtheta_{\rwide}})^i\left(\nabla_{\vtheta_1}(\vf_{\vtheta_{\rwide}})^j\right)^\T &\partial_{ij}\ell \nabla_{\vtheta_3}(\vf_{\vtheta_{\rwide}})^i\left(\nabla_{\vtheta_2}(\vf_{\vtheta_{\rwide}})^j\right)^\T & \partial_{ij}\ell \nabla_{\vtheta_3}(\vf_{\vtheta_{\rwide}})^i\left(\nabla_{\vtheta_3}(\vf_{\vtheta_{\rwide}})^j\right)^\T
			%      \end{pmatrix},
		
		For the second part $
		\sum_{i,j=1}^{m_L}\Exp_{S}\partial_i\ell (\vf_{\vtheta_{\rwide}},\vf^*)\mW_{i,j}^{[L]}\nabla_{\vtheta_{\rwide}^{[L-1]}}\nabla_{\vtheta_{\rwide}^{[L-1]}}\left(\vf_{\vtheta_{\rwide}}^{[L-1]}\right)_j
		$,    we have
		\begin{align*}
			& \sum_{i,j=1}^{m_L}\Exp_{S}\partial_i\ell (\vf_{\vtheta_{\rwide}},\vf^*)\mW_{i,j}^{[L]}\nabla_{\vtheta_{\rwide}^{[L-1]}}\nabla_{\vtheta_{\rwide}^{[L-1]}}\left(\vf_{\vtheta_{\rwide}}^{[L-1]}\right)_j =\begin{pmatrix}
				\mB_{1,1}&\mB_{1,2}&\mB_{1,3}\\
				\mB_{2,1}&\mB_{2,2}&\mB_{2,3}\\
				\mB_{3,1}&\mB_{3,2}&\mB_{3,3}
			\end{pmatrix},
		\end{align*}        
		% \Exp_S\begin{pmatrix}  \partial_i\ell\mW_{i,j}^{[L]}\nabla_{\vtheta_{1}}\nabla_{\vtheta_{1}}\left(\vf_{\vtheta_{\rwide}}^{[L-1]}\right)^j,& \partial_i\ell\mW_{i,j}^{[L]}\nabla_{\vtheta_{1}}\nabla_{\vtheta_{2}}\left(\vf_{\vtheta_{\rwide}}^{[L-1]}\right)^j &\partial_i\ell\mW_{i,j}^{[L]}\nabla_{\vtheta_{1}}\nabla_{\vtheta_{3}}\left(\vf_{\vtheta_{\rwide}}^{[L-1]}\right)^j\\
			%   \partial_i\ell\mW_{i,j}^{[L]}\nabla_{\vtheta_{2}}\nabla_{\vtheta_{1}}\left(\vf_{\vtheta_{\rwide}}^{[L-1]}\right)^j,& \partial_i\ell\mW_{i,j}^{[L]}\nabla_{\vtheta_{2}}\nabla_{\vtheta_{2}}\left(\vf_{\vtheta_{\rwide}}^{[L-1]}\right)^j &\partial_i\ell\mW_{i,j}^{[L]}\nabla_{\vtheta_{2}}\nabla_{\vtheta_{3}}\left(\vf_{\vtheta_{\rwide}}^{[L-1]}\right)^j\\
			%      \partial_i\ell\mW_{i,j}^{[L]}\nabla_{\vtheta_{3}}\nabla_{\vtheta_{1}}\left(\vf_{\vtheta_{\rwide}}^{[L-1]}\right)^j,& \partial_i\ell\mW_{i,j}^{[L]}\nabla_{\vtheta_{3}}\nabla_{\vtheta_{2}}\left(\vf_{\vtheta_{\rwide}}^{[L-1]}\right)^j &\partial_i\ell\mW_{i,j}^{[L]}\nabla_{\vtheta_{3}}\nabla_{\vtheta_{3}}\left(\vf_{\vtheta_{\rwide}}^{[L-1]}\right)^j
			%      \end{pmatrix}.
		with 
		\[
		\mB_{p,q}:=\Exp_{S}\sum_{i,j=1}^{m_L}\partial_i\ell (\vf_{\vtheta_{\rwide}},\vf^*)\mW_{i,j}^{[L]}\nabla_{\vtheta_{p}}\nabla_{\vtheta_{q}}\left(\vf_{\vtheta_{\rwide}}^{[L-1]}\right)_j, 
		\]
		for $p\in[3]$ and $q\in[3]$.
		
		Moreover,   
		\begin{equation}
			\begin{aligned}
				\nabla_{\vtheta_1}\left(\vf_{\vtheta_{\rwide}}(\vx)\right)_i&=\sum_{j=1}^{m_{L-1}} \left(\mW_{\rwide}^{[L]}\right)_{i,j}\nabla_{\vtheta_1}\left(\vf_{\vtheta_{\rwide}}^{[L-1]}\right)_j, i\in[m_L],\\
				\nabla_{\vtheta_2}\left(\vf_{\vtheta_{\rwide}}(\vx)\right)_i&=\sum_{j=m_{L-1}+1}^{2m_{L-1}} \left(\mW_{\rwide}^{[L]}\right)_{i,j}\nabla_{\vtheta_2}\left(\vf_{\vtheta_{\rwide}}^{[L-1]}\right)_j, i\in[m_L], 
				\\ \nabla_{\vtheta_3}\left(\vf_{\vtheta_{\rwide}}(\vx)\right)_i&=\sum_{j=2m_{L-1}+1}^{3m_{L-1}} \left(\mW_{\rwide}^{[L]}\right)_{i,j}\nabla_{\vtheta_3}\left(\vf_{\vtheta_{\rwide}}^{[L-1]}\right)_j, i\in[m_L],
			\end{aligned}
		\end{equation}
		hence  by construction of $\fT_{\mathrm{global}}$, we obtain that  
		\[
		\nabla_{\vtheta_1}\left(\vf_{\vtheta_{\rwide}}(\vx)\right)_i=\nabla_{\vtheta_2}\left(\vf_{\vtheta_{\rwide}}(\vx)\right)_i=-\nabla_{\vtheta_3}\left(\vf_{\vtheta_{\rwide}}(\vx)\right)_i.
		\]
		
		And for the Hessian, we have
		\begin{equation}
			\begin{aligned}
				\nabla_{\vtheta_1}\nabla_{\vtheta_1}\left(\vf_{\vtheta_{\rwide}}(\vx)\right)_i&=\sum_{j=1}^{m_{L-1}}\left(\mW_{\rwide}^{[L]}\right)_{i,j}\nabla_{\vtheta_1}\nabla_{\vtheta_1}\left(\vf_{\vtheta_{\rwide}}^{[L-1]}\right)_j, i\in[m_L], \\
				\nabla_{\vtheta_2}\nabla_{\vtheta_2}\left(\vf_{\vtheta_{\rwide}}(\vx)\right)_i&=\sum_{j=m_{L-1}+1}^{2m_{L-1}}\left(\mW_{\rwide}^{[L]}\right)_{i,j}\nabla_{\vtheta_2}\nabla_{\vtheta_2}\left(\vf_{\vtheta_{\rwide}}^{[L-1]}\right)_j, i\in[m_L],\\
				\nabla_{\vtheta_3}\nabla_{\vtheta_3}\left(\vf_{\vtheta_{\rwide}}(\vx)\right)_i&=\sum_{j=2m_{L-1}+1}^{3m_{L-1}}\left(\mW_{\rwide}^{[L]}\right)_{i,j}\nabla_{\vtheta_3}\nabla_{\vtheta_3}\left(\vf_{\vtheta_{\rwide}}^{[L-1]}\right)_j, i\in[m_L], \\
				\nabla_{\vtheta_1}\nabla_{\vtheta_2}\left(\vf_{\vtheta_{\rwide}}(\vx)\right)_i&= \nabla_{\vtheta_1}\nabla_{\vtheta_3}\left(\vf_{\vtheta_{\rwide}}(\vx)\right)_i= \nabla_{\vtheta_2}\nabla_{\vtheta_3}\left(\vf_{\vtheta_{\rwide}}(\vx)\right)_i,i\in[m_L],
			\end{aligned}
		\end{equation}
		Thus for $\mH^{(2),[L-1]}_S(\vtheta_{\rnarr}^{\rc})\neq\mathbf{0}$, then there exist a nonzero eigenvalue $\lambda\neq 0$ associated with its unit eigenvector $\vv$. 
		
		For the cases where $\lambda>0$, then $\vv^{\T}\mH^{(2),[L-1]}_S(\vtheta_{\rnarr}^{\rc})\vv = \lambda$. 
		We observe that  the matrix below is  a  principle submatrix of the Hessian at  $\vtheta_{\rwide}^{\rc}$, i.e., by choosing the columns and rows corresponding to $\left(\vtheta_1, \vtheta_2, \vtheta_3\right)$, we obtain that
		\begin{align*}
			\widetilde{\mH}(\vtheta_{\rwide}^{\rc}):&=\begin{bmatrix}
				\mH^{(1),[L-1]}_S(\vtheta_{\rnarr}^{\rc})  &\mH^{(1),[L-1]}_S(\vtheta_{\rnarr}^{\rc})  & -\mH^{(1),[L-1]}_S(\vtheta_{\rnarr}^{\rc})  \\
				\mH^{(1),[L-1]}_S(\vtheta_{\rnarr}^{\rc})  &\mH^{(1),[L-1]}_S(\vtheta_{\rnarr}^{\rc})  & -\mH^{(1),[L-1]}_S(\vtheta_{\rnarr}^{\rc}) \\
				-\mH^{(1),[L-1]}_S(\vtheta_{\rnarr}^{\rc})  &-\mH^{(1),[L-1]}_S(\vtheta_{\rnarr}^{\rc})  & \mH^{(1),[L-1]}_S(\vtheta_{\rnarr}^{\rc})  \\
			\end{bmatrix}\\
			&~~+
			\begin{bmatrix}
				\mH^{(2),[L-1]}_S(\vtheta_{\rnarr}^{\rc}) &\mathbf{0} & \mathbf{0} \\
				\mathbf{0} &\mH^{(2),[L-1]}_S(\vtheta_{\rnarr}^{\rc}) & \mathbf{0} \\
				\mathbf{0} & \mathbf{0} & -\mH^{(2),[L-1]}_S(\vtheta_{\rnarr}^{\rc}) \\
			\end{bmatrix},
		\end{align*}
		Then for $\vu = [\frac{1}{2}\vv^{\T},\frac{1}{2}\vv^{\T},\vv^{\T}]^{\T}$, 
		$$
		\vu^{\T}\widetilde{\mH}(\vtheta_{\rwide}^{\rc}) \vu = -\frac{1}{2}\lambda<0,
		$$
		indicating $\vtheta_{\rwide}^{\rc}$ is a strict-saddle point.

		Otherwise, if $\lambda<0$, since $\vv^{\T}\widetilde{\mH}(\vtheta_{\rwide}^{\rc}) \vv = \lambda$,
		then for $\vu = [\vv^{\T},-\vv^{\T},\mathbf{0}]^{\T}$, 
		$$
		\vu^{\T}\widetilde{\mH}(\vtheta_{\rwide}^{\rc}) \vu = 2\lambda<0,
		$$
		indicating $\vtheta_{\rwide}^{\rc}$ is also a strict-saddle point.
	\end{proof}
	
	%
	
	%  \\
	%     \nabla_{\vtheta_1}\nabla_{\vtheta_1}\vf_{\vtheta_{\rwide}} &= \nabla_{\vtheta_2}\nabla_{\vtheta_2}\vf_{\vtheta_{\rwide}}
	%     =\mW_{\rnarr}^{[L]}\nabla_{\vtheta_1}\nabla_{\vtheta_1}\vf^{[L-1]}, \\
	%     \nabla_{\vtheta''^{[L-1]}}^2\vf_{\vTheta}(\vx) &= 
	%     -\mW^{[L]}\nabla\nabla_{\vtheta^{[L-1]}}\vf^{[L-1]},\\
	%     \nabla_{\vtheta^{[L-1]}}\nabla_{\vtheta'^{[L-1]}}\vf_{\vTheta}(\vx)&=
	%     \nabla_{\vtheta^{[L-1]}}\nabla_{\vtheta''^{[L-1]}}\vf_{\vTheta}(\vx) = \nabla_{\vtheta'^{[L-1]}}\nabla_{\vtheta''^{[L-1]}}\vf_{\vTheta}(\vx) = \mathbf{0},

	\begin{lemma*}[Lemma \ref{lem5} in main text]
		Given any data $S$, loss $\ell(\cdot,\cdot)$ and activation $\sigma(\cdot)$, for any NN, if a critical point $\vtheta^{\rc}\in\vTheta^{\rc}$ satisfies:\\
		(i)~$\mH_S(\vtheta^{\rc})\succeq 0$;\\ 
		(ii)~$\mH^{(2),[L-1]}_S(\vtheta^{\rc})= \mathbf{0}$;\\
		(iii)~ $\mH^{(2)}_S(\vtheta^{\rc})\neq \mathbf{0}$.\\
		Then there exist a general compatible critical embedding $\fT$, such that $\fT(\vtheta)$ is a strict-saddle point.
	\end{lemma*}
	\begin{proof}
		Since we have
		\begin{align*}
			&\mH^{(2)}_S(\vtheta_{\rnarr}^{\rc}) =\begin{pmatrix}
				\mC_{1,1}&\mC_{1,2}&\mC_{1,3}\\
				\mC_{2,1}&\mC_{2,2}&\mC_{2,3}\\
				\mC_{3,1}&\mC_{3,2}&\mC_{3,3}
		\end{pmatrix}\end{align*}
		with 
		\begin{align*}
			\mC_{1,1}&=    \sum_{i=1}^{m_L}\partial_i\ell (\vf_{\vtheta_{\rwide}},\vf^*)\nabla_{\mW^{[L]}}\nabla_{\mW^{[L]}}\left((\vf_{\vtheta_{\rnarr}})_i\right),\\
			\mC_{1,2}&=\sum_{i=1}^{m_L}\partial_i\ell (\vf_{\vtheta_{\rwide}},\vf^*)\nabla_{\mW^{[L]}}\nabla_{\vb^{[L]}}\left((\vf_{\vtheta_{\rnarr}})_i\right),\\
			\mC_{1,3}&=\sum_{i=1}^{m_L}\partial_i\ell (\vf_{\vtheta_{\rwide}},\vf^*)\nabla_{\mW^{[L]}}\nabla_{\vtheta^{[L-1]}}\left((\vf_{\vtheta_{\rnarr}})_i\right),\\
			\mC_{2,1}&=\sum_{i=1}^{m_L}\partial_i\ell (\vf_{\vtheta_{\rwide}},\vf^*)\nabla_{\vb^{[L]}}\nabla_{\mW^{[L]}}\left((\vf_{\vtheta_{\rnarr}})_i\right),\\
			\mC_{2,2}&=\sum_{i=1}^{m_L}\partial_i\ell (\vf_{\vtheta_{\rwide}},\vf^*)\nabla_{\vb^{[L]}}\nabla_{\vb^{[L]}}\left((\vf_{\vtheta_{\rnarr}})_i\right),\\
			\mC_{2,3}&=\sum_{i=1}^{m_L}\partial_i\ell (\vf_{\vtheta_{\rwide}},\vf^*)\nabla_{\vb^{[L]}}\nabla_{\vtheta^{[L-1]}}\left((\vf_{\vtheta_{\rnarr}})_i\right),\\
			\mC_{3,1}&=\sum_{i=1}^{m_L}\partial_i\ell (\vf_{\vtheta_{\rwide}},\vf^*)\nabla_{\vtheta^{[L-1]}}\nabla_{\mW^{[L]}}\left((\vf_{\vtheta_{\rnarr}})_i\right),\\
			\mC_{3,2}&=\sum_{i=1}^{m_L}\partial_i\ell (\vf_{\vtheta_{\rwide}},\vf^*)\nabla_{\vtheta^{[L-1]}}\nabla_{\vb^{[L]}}\left((\vf_{\vtheta_{\rnarr}})_i\right),\\
			\mC_{3,3}&=\sum_{i=1}^{m_L}\partial_i\ell (\vf_{\vtheta_{\rwide}},\vf^*)\nabla_{\vtheta^{[L-1]}}\nabla_{\vtheta^{[L-1]}}\left((\vf_{\vtheta_{\rnarr}})_i\right),
		\end{align*}
		% &= \Exp_{S}
		%     \begin{bmatrix}
			%      \partial_i\ell (\vf_{\vtheta_{\rwide}},\vf^*) &\partial_i\ell\nabla_{\mW^{[L]}}\nabla_{\vtheta^{[L-1]}}\left((\vf_{\vtheta_{\rnarr}})^i\right)  \\
			%     \partial_i\ell\nabla_{\vb^{[L]}}\nabla_{\mW^{[L]}}\left((\vf_{\vtheta_{\rnarr}})^i\right) &        \partial_i\ell\nabla_{\vb^{[L]}}\nabla_{\vb^{[L]}}\left((\vf_{\vtheta_{\rnarr}})^i\right) &\partial_i\ell\nabla_{\vb^{[L]}}\nabla_{\vtheta^{[L-1]}}\left((\vf_{\vtheta_{\rnarr}})^i\right)  \\
			%     \partial_i\ell\nabla_{\vtheta^{[L-1]}}\nabla_{\mW^{[L]}}\left((\vf_{\vtheta_{\rnarr}})^i\right) &\partial_i\ell\nabla_{\vtheta^{[L-1]}}\nabla_{\vb^{[L]}}\left((\vf_{\vtheta_{\rnarr}})^i\right) & \mH^{(2),[L-1]}_S(\vtheta_{\rnarr}^{\rc}) \\
			%     \end{bmatrix},
		from  relation \eqref{eq..derivativew.r.t[L]}, we obtain that 
		\begin{align*}
			&\mH^{(2)}_S(\vtheta_{\rnarr}^{\rc}) \\
			&= \Exp_{S}
			\begin{bmatrix}
				\mathbf{0} &\mathbf{0} &\sum_{i=1}^{m_L}\partial_i\ell (\vf_{\vtheta_{\rwide}},\vf^*)\nabla_{\mW^{[L]}}\nabla_{\vtheta^{[L-1]}}\left((\vf_{\vtheta_{\rnarr}})_i\right)  \\
				\mathbf{0} &        \mzero &\mathbf{0}  \\
				\sum_{i=1}^{m_L}\partial_i\ell (\vf_{\vtheta_{\rwide}},\vf^*)\nabla_{\vtheta^{[L-1]}}\nabla_{\mW^{[L]}}\left((\vf_{\vtheta_{\rnarr}})_i\right) &\mathbf{0} & \mH^{(2),[L-1]}_S(\vtheta_{\rnarr}^{\rc}) \\
			\end{bmatrix},
		\end{align*}
		we observe that the matrix 
		\begin{align*}
			&\widetilde{\mH}^{(2)}_S(\vtheta_{\rnarr}^{\rc}) \\
			&= \Exp_{S}
			\begin{bmatrix}
				\mathbf{0} & \sum_{i=1}^{m_L}\partial_i\ell (\vf_{\vtheta_{\rwide}},\vf^*)\nabla_{\mW^{[L]}}\nabla_{\vtheta^{[L-1]}}\left((\vf_{\vtheta_{\rnarr}})_i\right)  \\
				\sum_{i=1}^{m_L}\partial_i\ell (\vf_{\vtheta_{\rwide}},\vf^*)\nabla_{\vtheta^{[L-1]}}\nabla_{\mW^{[L]}}\left((\vf_{\vtheta_{\rnarr}})_i\right) & \mH^{(2),[L-1]}_S(\vtheta_{\rnarr}^{\rc}) \\
			\end{bmatrix},
		\end{align*}
		satisfies that  $\mathrm{tr}(\widetilde{\mH}^{(2)}_S(\vtheta_{\rnarr}^{\rc}))= 0$. 
		
		Since ${\mH}^{(2)}_S(\vtheta_{\rnarr}^{\rc})\neq \mathbf{0}$, then $\widetilde{\mH}^{(2)}_S(\vtheta_{\rnarr}^{\rc})$ has at least one negative eigenvalue, denoted  by $\lambda^{\mathrm{neg}}$ associated with its unit eigenvector $\vv$.
		Then, we consider the three-fold global splitting embedding $\fT_{\mathrm{global}}$ defined in Example \ref{ex},  for $\vtheta_{\rwide}^{\rc}=\fT_{\mathrm{global}}(\vtheta_{\rnarr}^{\rc})$, and $\vu = [\vv^{\T},-\vv^{\T}]^{\T}$, we observe that  the matrix below is  a  principle submatrix of the Hessian at  $\vtheta_{\rwide}^{\rc}$, i.e.,  by choosing the columns and rows corresponding to $\left(\vtheta_1, \vtheta_2,\mathrm{vec}(\mW^{[L]})^\T, \mathrm{vec}(\mW^{[L]})^\T  \right)$, we obtain that
		\begin{align*}
			\widetilde{\mH}(\vtheta_{\rwide}^{\rc}) := 
			\begin{bmatrix}
				\bar{\mH}^{(1)}_S(\vtheta_{\rnarr}^{\rc})  &\bar{\mH}^{(1)}_S(\vtheta_{\rnarr}^{\rc})  \\
				\bar{\mH}^{(1)}_S(\vtheta_{\rnarr}^{\rc})  &\bar{\mH}^{(1)}_S(\vtheta_{\rnarr}^{\rc})    \\
			\end{bmatrix}
			+
			\begin{bmatrix}
				\widetilde{\mH}^{(2)}_S(\vtheta_{\rnarr}^{\rc})  &\mathbf{0}  \\
				\mathbf{0} &\widetilde{\mH}^{(2)}_S(\vtheta_{\rnarr}^{\rc})  \\
			\end{bmatrix},
		\end{align*}
		then for $\vu = [\vv^T,-\vv^T]^T$, we obtain that 
		\[
		\vV^{\T}\widetilde{\mH}(\vtheta_{\rwide}^{\rc})\vV = 2\lambda^{\mathrm{neg}}<0,
		\]
		indicating $\vtheta_{\rwide}^{\rc}$ is a strict-saddle point.
	\end{proof}
	
\bibliography{references}
\end{document}